\documentclass[11pt,letterpaper]{article}

\usepackage{epsfig}
\usepackage[margin=1in]{geometry}
\usepackage{amssymb, amsmath, amsthm, graphicx, subcaption, mathtools, tikz-cd}
\usepackage[section]{algorithm}
\usepackage{enumitem}
\usepackage{nicematrix, tikz}
\usepackage[dvipsnames]{xcolor}
\usepackage{hyperref}
\usepackage{algorithmic}
\usepackage{placeins}
\usepackage{mathdots}
\usepackage{blindtext}
\usepackage{tcolorbox}

\usepackage[authoryear,round]{natbib}

\hypersetup{
    colorlinks=true,
    linkcolor=violet,
    filecolor=magenta,      
    urlcolor=red,
    citecolor=blue,
    linktoc=all
}

\definecolor{myblue}{RGB}{0,102,204}        
\definecolor{mygreen}{RGB}{0,153,102}       
\definecolor{myorange}{RGB}{230,120,20}     
\definecolor{mypurple}{RGB}{120,80,200}     
\definecolor{mygray}{RGB}{90,90,90}         

\newcommand{\E}{\mathbb{E}}

\newcommand{\Prob}{\mathbb{P}}

\newcommand{\rpen}{\lambda_{rid}}

\newcommand{\wopt}{w^\star_{\epsilon, \lambda}}
\newcommand{\Exs}{\mathbb{E}}

\newcommand{\gtilde}{\widetilde{g}}
\newcommand{\ghat}{\widehat{g}}
\newcommand{\grad}{\nabla}
\newcommand{\betaols}{\widehat{\beta}_{\mathrm{OLS}}}

\newcommand{\real}{\mathbb{R}}

\newcommand{\gstar}{\bar{g}^\star}
\newcommand{\wstar}{w^\star} 
\newcommand{\weps}{w_{\epsilon}}
\newcommand{\deleps}{\Delta_{\epsilon}}
\newcommand{\regT}{\mathrm{Regret}(T)}
\newcommand{\numactions}{|\mathcal{A}|} 
\newcommand{\ST}{\ensuremath{\mathrm{S}_T}}
\newcommand{\inprob}{\; \stackrel{p}{\rightarrow} \; }

\newcommand{\BT}{{\Sigma^\star_T}}
\newcommand{\Id}{\mathbf{I}}
\newcommand{\Ncal}{\mathcal{N}}
\newcommand{\cfeat}{c}

\newcommand{\algo}{\mathcal{A}}
\newcommand{\distt}{\operatorname{d}_{TV}}
\newcommand{\z}{z}
\newcommand{\dist}{\operatorname{d}_{\operatorname{K}}}
\newcommand{\err}{\Psi(\gamma_T)}
\newcommand{\Regret}{\mathrm{Regret}}

\newcommand{\PT}{\tilde{\Sigma}_T}
\newcommand{\fil}{\mathcal{F}}
\newcommand{\SigPinv}{ \left(\PT \right)^{-1}}
\newcommand{\mtp}{S_T^{-1} \PT}
\newcommand{\stao}{\mathcal{S}_1(T)}
\newcommand{\stato}{\mathcal{S}_2(T)}
\newcommand{\MTP}{\mathcal{M}_T(p)}
\newcommand{\Eone}{\mathcal{E}_1(T)}
\newcommand{\Etwo}{\mathcal{E}_2(T)}

\newcommand{\lvec}{\textbf{u}}

\newcommand{\loss}{\ell}

\theoremstyle{plain}
\newtheorem{theorem}{Theorem}
\newtheorem{corollary}{Corollary}
\newtheorem{lemma}{Lemma}
\newtheorem{definition}{Definition}

\newtheorem*{lemma*}{Lemma}

\theoremstyle{definition}  
\newtheorem{assumption}{Assumption}

\colorlet{shadecolor}{gray!10}
\colorlet{lightyellow}{yellow!10}

\title{\LARGE \textbf{Avoiding the Price of Adaptivity: \\  Inference in Linear Contextual Bandits via Stability}}
\author{
\large 
Samya Praharaj$^{\ddagger}$, Koulik Khamaru$^{\ddagger}$\\
\normalsize   \\[0.3cm]
\normalsize $^{\ddagger}$ Department of Statistics, Rutgers University
}
\date{\today}

\begin{document}
\maketitle

\begin{abstract}
Statistical inference in contextual bandits is complicated due to the adaptive, non-i.i.d.\ nature of the
data. A growing body of work has shown that classical least-squares inference may fail under
adaptive sampling, and that constructing valid confidence intervals for linear functionals of the
model parameter typically requires paying an unavoidable inflation of order $\sqrt{d \log T}$.
This phenomenon---often referred to as the \emph{price of adaptivity}---highlights the inherent
difficulty of reliable inference under general contextual bandit policies.

A key structural property that circumvents this limitation is the \emph{stability} condition of
Lai and Wei~\citep{lai1982least}, which requires the empirical feature covariance to concentrate
around a deterministic limit. When stability holds, the ordinary least-squares estimator satisfies
a central limit theorem, and classical Wald-type confidence intervals---designed for i.i.d.\ data---
become asymptotically valid even under adaptation, \emph{without} incurring the $\sqrt{d \log T}$
price of adaptivity. 

In this paper, we propose and analyze a regularized EXP4 algorithm for linear contextual bandits. Our
first main result shows that this procedure satisfies the Lai--Wei stability condition and therefore
admits valid Wald-type confidence intervals for linear functionals. We
additionally provide quantitative rates of convergence in the associated central limit theorem. Our
second result establishes that the same algorithm achieves regret guarantees that are minimax
optimal up to logarithmic factors, demonstrating that stability and statistical efficiency can
coexist within a single contextual bandit method. As an application of our theory, we show how it can be used to construct confidence intervals for the conditional average treatment effect (CATE) under adaptively collected data.  Finally, we complement our theory with
simulations illustrating the empirical normality of the resulting estimators and the sharpness of the
corresponding confidence intervals.

\end{abstract}

\newcommand{\betastar}{\beta^\star}
\newcommand{\betahat}{\widehat{\beta}}
\newcommand{\betaold}{\widehat{\beta}_{\mathrm{OLS}}}

\section{Introduction}
Statistical inference under adaptive data collection has become a central challenge in modern learning systems. Unlike classical settings where data are sampled independently of past observations, adaptive algorithms—such as contextual bandits and reinforcement learning policies—select actions based on previously observed losses. This feedback loop creates intricate dependencies that distort the distribution of both covariates and errors, often rendering standard inferential tools invalid. Even when estimators remain consistent, their asymptotic distributions can deviate substantially from classical theory, complicating uncertainty quantification and hypothesis testing~\citep{dickey1979distribution,lai1982least,zhang2020inference,deshpande2023online,khamaru2021near}. 

A recurring theme in recent work is that valid inference is possible using when the data-collection rule exhibits sufficient regularity or \emph{stability}.
At a high level, a bandit algorithm is \emph{stable} if the long-run behavior of the algorithm settles into a predictable pattern, despite ongoing adaptation. When such structure exists, it becomes possible to characterize limiting distributions of estimators, derive confidence sets, and recover analogs of the classical central limit theorem~\citep{lai1982least}. Recent works have demonstrated that this property is satisfied for the Upper Confidence Bound-type of algorithms~\citep{kalvit2021closer,fan2022typical,khamaru2024inference,han2024ucb}, and a variant of the Thompson Sampling~\citep{halder2025stable,fan2022typical}. Recent work of~\cite{fan2025statistical} show that the LinUCB~\citep{li2010contextual,abbasi2011improved} algorithm algorithm also satisfy this stability condition for linear contextual bandit problem.

\newcommand{\dims}{d}

In this paper, we investigate stability properties of bandit algorithms for a linear contextual bandit problem. Formally, at each round \(t\), the learner observes a 
context \(x_t \in \mathcal{X}\) and selects an action \(a_t \in \mathcal{A}\) based on past data and context $x_t$. Concretely, let 
$\mathcal{F}_{t - 1} := \sigma \left( x_1, a_1, \loss_1, \ldots, x_{t - 1}, a_{t - 1}, \loss_{t - 1} \right)$ denote the sigma-field generated by the observations up to time $t - 1$, then $a_t$ depends on $\mathcal{F}_{t - 1}$ and $x_t$. Upon selecting an action $a_t$ we incur a loss according to a linear model:
\begin{align}
\label{eqn:bandit-reward}
    \loss_t \;=\; \langle \phi(a_t, x_t),\, \betastar \rangle + \varepsilon_t,
\end{align}
where $\phi : \mathcal{A} \times \mathcal{X} \to \mathbb{R}^d$ is a known feature map and 
$\beta^\star \in \mathbb{R}^d$ is an unknown parameter to be learned. 
We assume that the noise sequence $\{\varepsilon_t\}_{t \ge 1}$ satisfies $\mathbb{E}[\varepsilon_t \mid \mathcal{F}_{t-1}, x_t, a_t] = 0$.
Our goal is to construct confidence intervals for linear functional of the form $a^\top \beta^\star$, for any fixed vector $a \in \real^\dims$. 

\newcommand{\VT}{\mathbf{V}_T}
\newcommand{\RT}{\mathrm{R}_T}
\newcommand{\CIwald}{\mathcal{I}^{\mathrm{Wald}}_T(a)}
\newcommand{\CIconc}{\mathcal{I}^{\mathrm{APS}}_T(a)}
\newcommand{\betaridge}{\widehat{\beta}_{\lambda, T}}

\subsection{Price of Adaptivity:}
\label{sec:price-of-adaptivity}
Before we dive into more details, it is useful to compare the available method for constructing confidence intervals for $a^\top \betastar$. One popular approach of constructing confidence intervals for $a^\top \betastar$ is to use martingale concentration inequalities~\citep{de2004self,de2009self}. Formally, let \(\betaridge\) denote the ridge-estimator with regularizer~$\lambda$ based on data up to time \(T\). Let
\[
    \VT \;=\; \lambda I_d + \sum_{t=1}^T \phi(a_t,x_t)\,\phi(a_t,x_t)^\top
\]
be the regularized design matrix. Given a target confidence $\alpha \in (0,1)$, a widely used approach, originating from the work of~\cite{abbasi2011improved}, is to construct a confidence interval
\begin{align}
    \mathcal{I}^{\mathrm{APS}}_T(a)
    \;:=\;
    \Bigl[
        a^\top \hat\betaridge
        \;\pm\;
        \RT \,\sqrt{a^\top \VT^{-1} a}
    \Bigr].
    \label{eq:APS-interval}
\end{align}
Assuming $\|\betastar\|_2 \leq S$, $\|\phi(x,a)\|_2 \leq L$ and noise $\epsilon_t$ is $1$ sub-Gaussian, the factor $\RT$ takes the following form
\begin{align}
\label{eqn:RT-factor}
    \RT
    \;=\;
    \sqrt{d \log\!\Bigl(\tfrac{T L}{\lambda}\Bigr) + \log(1/\alpha)}
    \;+\;
    \sqrt{\lambda}\, S.
\end{align}

\newcommand{\iid}{\stackrel{iid}{\sim}}

It is useful to compare the confidence interval~\eqref{eq:APS-interval} with a Wald-type interval, which is asymptotically exact when the data $(x_t, a_t) \iid \mathcal{P}$: 
\begin{align} \label{eq:wald-interval}
    \mathcal{I}^{\mathrm{Wald}}_T(a)
    \;:=\;
    \Bigl[
        a^\top \betaols
        \;\pm\;
        z_{1-\alpha/2}\,
        \widehat\sigma\, \sqrt{a^\top S_T^{-1} a}
    \Bigr],
\end{align}
where $\betaols$ is the least square estimator,  \(z_{1-\alpha/2}\) is the standard normal quantile and \(\hat\sigma\) is a consistent estimator of noise variance \(\sigma\). Comparing the length of the two confidence intervals from~\eqref{eq:APS-interval} and~\eqref{eq:wald-interval}, we observe~\footnote{One usually takes $\lambda$ to be small and the effect of the term $\lambda\sqrt{S}$ in~\eqref{eqn:RT-factor} is negligible.}  
\begin{align}
    \frac{\mathrm{width}\bigl(\mathcal{I}^{\mathrm{APS}}_T(a)\bigr)}
         {\mathrm{width}\bigl(\mathcal{I}^{\mathrm{Wald}}_T(a)\bigr)}
    \;\approx\;
    \sqrt{d \log T},
    \label{eqn:CI_length-compare}
\end{align}
Put simply, for contextual bandit problems the confidence intervals are $\sqrt{d\log T}$ times wider.

It is natural to ask whether one might construct any other confidence interval based on some other estimator. A recent line of work by~\cite{lattimore2023lower,khamaru2021near,vakili2021open} show that this enlagrement in the confidence interval length by a factor of $\sqrt{d\log T}$ is necessary when $d \geq 2$. State differently, this enlargement in confidence interval for contextual bandit is unavoidable in a worst case sense.

\subsection{Validity of Wald's Interval via Stability:}
The comparison above highlights the \emph{price of adaptivity}: without further structure, any
confidence interval for $a^\top\betastar$ must inflate by a factor of order $\sqrt{d\log T}$, reflecting
the \emph{worst–case} distortions introduced by adaptive sampling. A natural question is whether this
inflation is intrinsic to all adaptive procedures, or whether additional regularity in the data--collection
rule can restore the validity of classical Wald's confidence interval~\eqref{eq:wald-interval}.

Following the seminal work of~\cite{lai1982least}, this regularity can be formalized
through the notion of \emph{stability}.

\begin{definition}
\label{defn:stability}
    We call a bandit algorithm~$\algo$ stable if there exists a sequence of \emph{non-random positive definite} matrix $\{\BT\}$ such that 
    \begin{align}
        \BT^{-1} \ST \inprob \Id 
    \end{align}
\end{definition}
Theorem~3 of~\cite{lai1982least} ensures that if a bandit algorithm~$\algo$ is stable, then the least square estimator $\betaols$ is asymptotically normal: 
\begin{align}
    \label{eqn:ols-normality}
    \frac{1}{\widehat{\sigma}} \cdot \ST^{-\frac{1}{2}} \left( \betaols - \betastar \right) \xrightarrow{d} \Ncal \left( 0, \Id \right)
\end{align}
where $\widehat{\sigma}$ is a consistent estimator~\footnote{For instance, one might use the estimator from~\cite[Lemma 3]{lai1982least}.} for noise standard deviation $\sigma$. An immediate consequence of the above using Slutsky's theorem is that 
\begin{align*}
    \lim_{T \rightarrow \infty} \Prob \left( \CIwald \ni a^\top \betastar \right) = 1 - \alpha. 
\end{align*}
In other words, Wald's confidence interval for $a^\top \betastar$ is valid as long as the underlying bandit algorithm~$\algo$ is stable, and we do not pay the $\sqrt{d\log T}$ \emph{price of adaptivity}. Our goal in this paper is to propose and analyze an algorithm which is simultaneously stable and provide good regret.  

\subsection{Contributions}
\label{sec:Contrib}

In this paper, we study the problem of constructing confidence intervals for linear functionals of the model parameters in a linear contextual bandit problem. Section~\ref{sec:prob-setup} provides a
detailed description of the contextual bandit framework we consider, while
Section~\ref{sec:algo} introduces the algorithm we analyze. The procedure, stated formally in
Algorithm~\ref{alg:penexp4-finite}, is a regularized variant of the classical EXP4 algorithm designed
to enforce stability while allowing adaptive exploration.
Our first main result, stated in Theorem~\ref{thm:clt}, establishes that the regularized EXP4 algorithm is
stable in the sense of~\cite{lai1982least} (see definition~\eqref{defn:stability}), enabling the
construction of asymptotically valid Wald-type confidence intervals. In addition, we provide a non-asymptotic rates of convergence to normality. Our second result, detailed in Theorem~\ref{thm:regret} , analyzes the regret of the proposed algorithm. We show
that, despite being explicitly designed for inferential stability, the procedure simultaneously
achieves minimax--optimal regret guarantees up to logarithmic factors. Thus, stability and
statistical efficiency can coexist within a single contextual bandit algorithm. As an application, we establish asymptotically exact Wald-type confidence intervals for conditional average treatment effects in linear models with adaptively assigned treatments, using ridge regression. Section~\ref{sec:proof_of_theorems} contains proofs of our main theorems with some technical details deferred to the Appendix.  Finally, in Section~\ref{sec:exps}, we complement our theoretical findings with numerical
experiments demonstrating the empirical validity of the Wald confidence intervals produced by
Algorithm~\ref{alg:penexp4-finite}.

\subsection{Related Work}

\noindent
The challenges of performing statistical inference under adaptive data collection have been
well documented in recent years. Because adaptive policies bias the distribution of covariates and
errors, classical inferential procedures may become invalid. This breakdown has been observed
empirically (\citet{xu2013estimation,villar2015multi}) and supported by theoretical analyses
(\citet{nie2018adaptively,shin2019sample,shin2021bias}). To counteract these issues, several
methodological approaches have been proposed, including online debiasing techniques
(\citet{khamaru2021near,chen2022debiasing,kim2023double}) and procedures based on inverse
propensity scores (\citet{hadad2021confidence,deshpande2018accurate,zhang2022statistical,
nair2023randomization,leiner2025adaptive}).  
These methods address inference in adaptive environments more broadly, but do not directly
resolve the structural constraints that give rise to the $\sqrt{d \log T}$ ``price of adaptivity'' in
contextual bandits.

A distinct line of work focuses on identifying conditions under which classical, Wald-type
asymptotic inference is restored despite adaptivity. The seminal paper of
\citet{lai1982least} introduced the notion of \emph{stability}, showing that if the empirical design
covariance converges to a deterministic limit, then least-squares estimators satisfy a central limit
theorem and Wald confidence intervals regain asymptotic validity.  
This perspective has motivated a growing literature on stability-based inference for bandit
algorithms (\citet{kalvit2021closer,khamaru2024inference,fan2022typical,fan2024precise,
han2024ucb,halder2025stable}; \ \citet{fan2025statistical}). However, subsequent work has shown that many commonly used
bandit algorithms fail to satisfy the Lai--Wei stability condition, leading to substantial
under-coverage when Wald intervals are applied naively
(\citet{fan2024precise,praharaj2025instability}).  
This underscores the central question motivating our work: \emph{can one design adaptive
algorithms that are simultaneously stable and statistically efficient?}

Contextual and adversarial bandit algorithms are frequently derived from mirror descent
(\citet{lattimore2020bandit}), a framework that has been applied extensively in online
optimization (\citet{abernethy2009competing,audibert2014regret,bubeck2018sparsity,wei2018more})
and that underlies widely used adversarial bandit methods such as EXP3
(\citet{auer1995gambling}), EXP4 (\citet{auer2002nonstochastic}), Tsallis-INF
(\citet{zimmert2021tsallis,masoudian2021improved}), and OFTRL (\citet{pmlr-v134-ito21a}).
Despite their strong regret guarantees, the stability properties of these algorithms are not well
understood, and existing analyses suggest that many of them may not support valid Wald-type
inference.

Finally, an alternative approach to inference in adaptive settings relies on non-asymptotic,
anytime-valid confidence intervals constructed via concentration inequalities for self-normalized
martingales. This line of work builds on the foundational results of de la Pe\~na et al.\
(\citet{de2004self,de2009self}) and includes several refined analyses
(\citet{abbasi2011improved,howard2020time,waudby2024anytime}).  
These intervals hold uniformly over time and do not rely on asymptotic arguments, but the price
paid is typically much wider confidence intervals compared to those achievable under stability,
reflecting the worst-case nature of anytime-valid guarantees.

\subsection{Notation}

For any matrix $A$, let $\|A\|_{\mathrm{op}}$ and $\|A\|_{F}$ denote the operator norm and the Frobenius norm, respectively. For any vector $v \in \mathbb{R}^d$, the standard $\ell_p$-norm for $p \in (0,\infty)$ is denoted by $\|v\|_p$. For a real-valued random variable $u$, we define $\|u\|_p := \mathbb{E}\!\left[ |u|^p \right]^{1/p}.$ Given a fixed weight vector $w \in \mathbb{R}^d$, we define a weighted norm on $\mathbb{R}^d$, denoted by $\|\cdot\|_{w,\ast}$, as $\|v\|_{w,\ast}^2 := \sum_{i=1}^d w_i v_i^2,$ where $v \in \mathbb{R}^d.$ For two nonnegative sequences $\{a_n\}$ and $\{b_n\}$, we write $b_n \gg a_n$ if $\frac{b_n}{a_n} \to \infty \quad \text{as } n \to \infty .$ Throughout, we suppress absolute constants in inequalities by using the notation $\lesssim$ and $\gtrsim$. The Loewner partial order on symmetric matrices is denoted by $\preceq$. Finally, the Kolmogorov distance between two real-valued random variables $X$ and $Y$ is denoted by $\mathrm{d}_{\mathrm{K}}(X,Y)$ and is defined as
\begin{align*}
\mathrm{d}_{\mathrm{K}}(X,Y)
:= \sup_{t \in \mathbb{R}} \left| \mathbb{P}(X \le t) - \mathbb{P}(Y \le t) \right|.
\end{align*}

\section{Problem Setup}
\label{sec:prob-setup}
We consider a linear contextual bandit problem with a finite action set $\mathcal{A}$. 
At each round $t = 1, \ldots, T$, the learner observes a context vector $x_t \in \mathcal{X}$, drawn i.i.d. from a distribution $\mathcal{P_X}$, selects an action $a_t \in \mathcal{A}$, and receives a random loss
\begin{equation}\label{eqn:loss}
\loss_t \;=\; \langle \beta^\star,\, c(x_t, a_t) \rangle + \varepsilon_t,
\end{equation}
where $c(x_t, a_t) \in \mathbb{R}^d$ is a known feature representation of the pair $(x_t,a_t)$, and  
$\beta^\star \in \mathbb{R}^d$ is an unknown parameter vector. 
At each round, the learner chooses an action by sampling from a \emph{mixture of $K$ base experts} 
$\{\pi_1, \ldots, \pi_K\}$, where each $\pi_k( \cdot \mid x_t)$ defines an arbitrary (possibly stochastic) policy mapping the context $x_t$ to a distribution over actions.   
The learner maintains mixture weights $w_t = (w_{t,1}, \ldots, w_{t,K})$ belonging to the $\varepsilon$-simplex
\begin{align}
\label{eqn:eps-simplex}
\Delta_\varepsilon 
= 
\left\{\, w \in \mathbb{R}^K_{\ge 0} : \sum_{k=1}^K w_k = 1,\; w_k \ge \varepsilon \,\right\}.
\end{align}
We take $\epsilon$ to be a \emph{small} positive tuning parameter that decays with the number of rounds $T$. Introduction of the tuning parameter $\varepsilon$ guarantees $w_k \geq \varepsilon$ and consequently ensures that various importance ratio based estimators used in our algorithm are always well-defined. 
The effective policy at round $t$ is therefore the convex mixture
\begin{equation}
Q_t(a \mid x_t)
\;=\;
\sum_{k=1}^K w_{t,k}\,\pi_k(a \mid x_t).
\end{equation}
The learner then draws an action $a_t \sim Q_t(\cdot \mid x_t)$ and observes the corresponding stochastic loss $\loss_t$. We use $\mathcal{F}_t := \sigma(x_1, a_1, \loss_1, \ldots, x_{t}, a_t, \loss_t)$ to denote the $\sigma$-field generated by observations up to time $t$.  

\noindent Throughout, we work under the following assumptions:
\begin{assumption}
\leavevmode
\begin{enumerate}[label=(A\arabic*), leftmargin=2.5em]
  \item \label{assn:noise} $\mathbb{E}[\varepsilon_t|\mathcal{F}_{t -1}, x_t,a_t] = 0$, $|\varepsilon_t| \leq 1$,  and 
  $\mathbb{E}[\varepsilon^2_t|\mathcal{F}_{t -1}, x_t,a_t] = \sigma^2$
     for all $t \geq 1$.
  \item    \label{assn:bonded-loss} The feature vector and the unknown parameter satisfy 
  \begin{align*}
      \|\beta^\star\|_2 \leq 1 \quad \text{and} \quad  \|c(x, a)\|_2 \leq 1, 
      \qquad  \text{for all } (x, a) \in \mathcal{X} \times \mathcal{A}. 
  \end{align*}
    \item \label{assn:nonneg-loss} The observed loss  $\loss_t$ are non-negative for each $t \in [T]$.
    \item \label{assn:context} The context vectors $(x_t)_{t\geq1}$ are drawn i.i.d. from a distribution $\mathcal{P_X}$, and for every expert $k \in [K]$
    \begin{align}
    \label{eqn:min-eigval}
    \lambda_{\min} \left\lbrace \mathbb{E}_{x \sim \mathcal{P_X}} \left( \sum_{a \in \mathcal{A}} \pi_k(a \mid x) \cdot c(x, a) c(x, a)^\top \right) \right\rbrace  \geq \lambda^\star > 0. 
    \end{align}
    \item Assume that \label{assn:Q-bdd} $ \sup_{k \in [K]}\{\inf_{a,x}\pi_k(a|x)\} > 0$.
\end{enumerate}
\end{assumption}
Assumption~\ref{assn:noise} states that the noise sequence is conditionally zero mean, bounded and homoscedastic. Assumption~\ref{assn:bonded-loss} imposes a boundedness assumption on the parameter vector $\beta^\star$ and feature map $c(\cdot, \cdot)$. Assumptions~\ref{assn:noise} and~\ref{assn:bonded-loss} together imply that the observed losses 
$\{\loss_t\}$ are uniformly bounded. Consequently, in Assumption~\ref{assn:nonneg-loss} we may 
assume without loss of generality that $\loss_t \ge 0$. 
Indeed, if the losses are not necessarily nonnegative, uniform boundedness guarantees the existence of a constant $C > 0$ such that $\loss_t + C \ge 0$ almost surely for all $t$. Replacing $\loss_t$ by $\loss_t + C$ amounts to adding a constant offset to the loss model in~\eqref{eqn:loss}, which can be equivalently absorbed into a redefinition of the feature map and parameter vector. This transformation leaves the regret, stability properties, and inferential guarantees established in this paper unchanged.

Assumption~\ref{assn:context} posits that the the context vectors $\{x_t\}$ are i.i.d., and that the weighted covariance matrix associated with each expert is non-singular. Assumption~\ref{assn:Q-bdd} can always be satisfied by taking one of the expert is as a uniform expert $\pi_{\mathrm{unif}}(a \mid x) = 1/|\mathcal{A}|$ for all action $a \in \mathcal{A}, \;\;  \text{and context} \;\; x \in \mathcal{X}$. This assumption along with our choice of weights $w_{t,k} \geq \epsilon > 0$ (see equation~\eqref{eqn:eps-simplex}) ensures that $Q_t(a_t \mid x_t) > 0$ for all $t \geq 1$ and various importance weight based estimators are well-defined.

\subsection{Linearity of loss in $w$}
%
Let $\Delta_K$ be the simplex in $K$ dimension. At each round $t$, the learner maintains mixture weights $w_t \in \Delta_K$ over a fixed set of base policies $\{\pi_k\}_{k=1}^K$.
The executed policy is the convex combination
\begin{equation}
    Q_t(a \mid x_t)
    \;=\;
    \sum_{k=1}^{K} w_{t,k}\,\pi_k(a \mid x_t).
    \label{eq:mixture-policy}
\end{equation}
This mixture induces both the sampling distribution of actions and the expected loss of the algorithm.
The loss incurred at time $t$ is a function of both context $x_t$ and action $a_t$, due to which in the following discussion we rewrite $l_t$ as $l_t(a_t,x_t)$. Now,
at round $t$, any expectation under $Q_t$ is a convex combination of the expectations under the individual experts $\{\pi_k\}$:
\begin{align}
    \mathbb{E}_{a \sim Q_t}[\loss_t]
    = \sum_{a}\Big(\sum_{k} w_{t,k}\,\pi_k(a \mid x_t)\Big)r(a,x_t)
      &= \sum_{k} w_{t,k}\,\underbrace{\sum_{a}\pi_k(a \mid x_t)\Exs[ l(a,x_t) \mid x_t]}_{\text{expected loss of expert }k}. \\
    &\equiv\left\langle w_t, g^\star(x_t) \right\rangle   
\end{align}
Hence, the expected loss is \emph{linear in $w_t$}.  Since the per-round loss $g^\star(x_t)$ depends on the context, we define the \emph{global} (context-averaged) \emph{loss vector}
\begin{equation}
    \bar g^\star
    := \mathbb{E}_{x\sim\mathcal{P_X}}[\,g^\star(x)\,]
    = \Big(\,\mathbb{E}_{x}\!\big[\sum_a l(a,x)\,\pi_k(a\mid x)\big]\Big)_{k=1}^K.
    \label{eq:gbar-def}
\end{equation}

\subsection{Unbiased estimate of $g^\star(x_t)$}\label{sec:unbiased-grad}
Although only one action $a_t$ is observed, the importance-weighted estimator for expert $k$,
\begin{equation}
    \widehat{g}_{t,k}
    \;=\;
    \,\loss_t\,\frac{\pi_k(a_t \mid x_t)}{Q_t(a_t \mid x_t)},
    \label{eq:ips-estimator}
\end{equation}
is an unbiased estimator of the per-expert loss $g_k^\star(x_t) = \sum_a l(a,x_t)\pi_k(a\mid x_t)$.
Taking expectation over $a_t \sim Q_t(\cdot \mid x_t)$ yields
\begin{align}
    \mathbb{E}[\widehat{g}_{t,k} \mid x_t]
    &= \sum_a \Exs[ l(a,x_t) \mid x_t]\,Q_t(a\mid x_t)\,\frac{\pi_k(a\mid x_t)}{Q_t(a\mid x_t)}
     = g_k^\star(x_t).
\end{align}

\noindent 
\subsection{Regret} 
We measure the regret of our policy with respect to optimal mixture: 
\begin{align*}
    w^\star := \arg\min_{w \in \Delta_K} \left\langle \bar{g}^\star, w  \right\rangle 
\end{align*}
The cumulative regret after $T$ rounds  is defined as
\begin{equation}
\text{Reg}(T)
\;=\;
\sum_{t=1}^T \Exs \left\lbrace 
\left\langle  g^\star(x_t), w_t - w^\star \right\rangle  \right\rbrace.
\end{equation}
All expectations are taken with respect to the learner’s internal randomness and the stochastic loss environment as well as the random context $x_t$.

\section{A regularized Exp4 Algorithm}
\label{sec:algo}

In this section we describe our algorithm. Recall that at round $t$, the effective policy executed by the learner is the convex mixture
\begin{equation}
Q_t(a \mid x_t)
\;=\;
\sum_{k=1}^K w_{t,k}\,\pi_k(a \mid x_t),
\end{equation}
where each $\pi_k(a \mid x_t)$ is an arbitrary base policy that maps the context $x_t$ 
to a distribution over actions. Upon taking an action $a_t \sim Q_t(\cdot \mid x_t)$, the learner receives a stochastic loss $\loss_t$. To evaluate the experts, our algorithm --- stated in Algorithm~\ref{sec:algo} --- constructs the standard importance-weighted gradient estimator
\begin{equation}
\widehat{g}_{t,k}
=\; \,\loss_t\,\frac{\pi_k(a_t\mid x_t)}{Q_t(a_t\mid x_t)},
\qquad k=1,\ldots,K.
\end{equation}
The weight vector $w_t$ is then updated via a composite mirror-descent~\cite{duchi2011adaptive} step with the entropy-induced mirror map $\phi(\cdot)$ , and an entropy-induced  penalty $R(\cdot)$:
\begin{align}
\label{eqn:mirror-map-and-pen}
\phi(w)=\sum_{k = 1}^K w_k\log w_k - w_k \qquad \text{and} \qquad  R(w) =  \sum_{k = 1}^K w_k\left(\log w_k + \log(1/\epsilon) - 1 \right)   
\end{align}
Formally, our Algorithm~\ref{alg:penexp4-finite} minimizes the following loss in an online fashion via a mirror descent-style algorithm:
\begin{align}
\label{eqn:modified-convex-loss} 
\min_{w \in \Delta_\varepsilon} \left\lbrace
 \left\langle \bar{g}^\star, w \right\rangle + \lambda R(w)
 \right\rbrace. 
\end{align}
The regularization term $\lambda R(w)$ introduces a mild curvature to the otherwise linear objective $\langle \bar{g}^\star, w \rangle$, ensuring  better-behaved updates. The parameter $\lambda$ is chosen to be small so that the solution to the modified problem~\eqref{eqn:modified-convex-loss} achieves low regret with respect to the natural linear loss $\langle \bar{g}^\star, w \rangle$. In addition, the constraint set $\Delta_\varepsilon$ guarantees that the gradients of the loss~\eqref{eqn:modified-convex-loss} remain \emph{well-behaved}.

\begin{algorithm}[H]
\caption{Regularized EXP4}
\label{alg:penexp4-finite}
\begin{algorithmic}[1]
\STATE \textbf{Input:} Base policies $\{\pi_{k}\}_{k=1}^K$, stepsizes $\{\eta\}$, penalty $\lambda>0$, floor $\epsilon\in(0,1/K]$.
\STATE Initialize weights $w_{1,k}=1/K$ for all $k$, and set $c_R = \log(1/\epsilon)$
\FOR{$t=1,\ldots,T$}
    \STATE Observe context $x_t$
    \STATE Form mixture $Q_t(a\mid x_t)=\sum_{k=1}^K w_{t,k}\,\pi_k(a\mid x_t)$
    \STATE Sample $a_t\sim Q_t(\cdot\mid x_t)$ and observe loss $\loss_t$
    \STATE Compute estimates $\widehat{g}_{t,k}= \,\loss_t\,\pi_k(a_t\mid x_t)/Q_t(a_t\mid x_t)$
    \STATE Update intermediate weights
    \[
     w^{+}_{t+1,k}=w_{t,k}\,
    \exp\!\Big(-\eta \,\widehat g_{t,k}
              - \lambda\big( [\grad R(w_t)]_k + c_R \big)\Big)
    \]
    \STATE Update $w_{t+1} = \arg\min \limits_{w \in \Delta_\varepsilon} \left\lbrace  D_{\phi}(w, w_{t + 1}^{+})  \right\rbrace $
\ENDFOR
\end{algorithmic}
\end{algorithm}

\newcommand{\ratto}{\mathcal{V}_T}
\newcommand{\natto}{\mathcal{N}_T}
\newcommand{\delo}{\delta_{1,T}}
\newcommand{\deltoto}{\delta_{2,T}}
\newcommand{\expower}{\operatorname{exp}}
\newcommand{\laml}{\lambda^\star_L}
\newcommand{\lamu}{\lambda^\star_U}

\section{Main results}
\label{sec:main-results}
In this section we state our main results. 
Section~\ref{sec:price-of-adaptivity} highlights that, for general adaptive data-collection
rules, valid confidence intervals for linear functionals must inflate by a factor on the
order of $\sqrt{d \log T}$. In this section, we show that the regularized EXP4 procedure
(Algorithm~\ref{alg:penexp4-finite}) avoids this worst-case behavior. In particular, this algorithm satisfies the  stability from Definition~\ref{defn:stability}, which ensures
the validity of Wald-type inference, and it does so while retaining near--minimax-optimal regret.

We begin with the stability and asymptotic normality of the ordinary least-squares estimator, 
and then establish regret guarantees.

\subsection{Stability and a central limit theorem}

A key component of our analysis is that the empirical average of the weight vectors $\{w_t\}$ converges to a fixed vector $w_T^\star$:
\begin{align}
    \frac{1}{T}\sum_{t=1}^T w_{t,k} - w^\star_{T,k}
    \;\xrightarrow[]{P}\; 0,
    \qquad k \in [K].
    \label{eq:avg-weight-conv}
\end{align}
A consequence of this convergence is that the Gram matrix $S_T := \sum_{t = 1}^T c(x_t, a_t) c(x_t, a_t)^\top$ is asymptotically close to a deterministic limit $\Sigma^\star_T$, where 
\begin{align}
\label{eqn:sigmaT-star}
    \Sigma^\star_T = T  \sum^K_{k = 1} \wstar_{T,k} \Sigma_k \qquad \text{with} \qquad \Sigma_k :=   \mathbb{E}_{x \sim \mathcal{P_X}} \left( \sum_{a \in \mathcal{A}} \pi_k(a \mid x) \cdot c(x, a) c(x, a)^\top \right)
\end{align}
Thus the stability condition~\eqref{defn:stability} is satisfied. 
With this structural ingredient in place, we now state our quantitative central limit theorem.

\begin{subequations}
\begin{theorem}
\label{thm:clt}
Under Assumptions~\ref{assn:noise}-\ref{assn:Q-bdd}, the regularized-EXP4 algorithm (\ref{alg:penexp4-finite}) with step size $\eta = \sqrt{\frac{\log K}{ \numactions T}}$, and tuning parameters  $\epsilon = \frac{1}{KT}$, $\lambda = \frac{\gamma_T}{\sqrt{T}}$ with $\gamma_T \rightarrow \infty$ and $ T/ \log^2 T \gg \gamma_T$, is stable; we have 
\begin{align*}
    \Sigma_T^{\star -1} \cdot \mathrm{S}_T \xrightarrow{p} \mathbf{I}_d,   
\end{align*}
with $\Sigma_T^{\star}$ from~\eqref{eqn:sigmaT-star}.  Furthermore, for any fixed $a \in \real^d$ we have 
\begin{align}
   \sup_{a \in \real^d / \{0\} } \dist\!\left(
        \frac{a^{\top}(\widehat{\beta}_{\mathrm{OLS}} - \beta^\star)}
             {\sigma \sqrt{a^\top S_T^{-1} a}},\;
        Z
    \right)
    \;\leq\;
    C \left[ \Psi(\gamma_T)^{1/3}  +  \frac{1}{T^{1/3}}
     +  d \cdot \expower\left\{ - \frac{(\laml)^2}{32 +  8\laml/3} \cdot  T\right\} \right]
\end{align}
where $Z \sim \mathcal{N}(0,1)$,  $C$ is a constant independent of $T$, and 
\begin{align}\label{defn:err}
        \err := \sqrt{10{\frac{\sqrt{\numactions \log K}}{\gamma_T}} +
        \frac{8\gamma_T \log(K) \log^2\left( KT \right) }{ \numactions^2 T  }  } 
\end{align}
\end{theorem}
\end{subequations}

\noindent  
A direct consequence of Theorem~\ref{thm:clt} is that the Wald confidence interval
\eqref{eq:wald-interval} is asymptotically exact: for any fixed $a \in \mathbb{R}^d$,
\begin{align*}
    \lim_{T \to \infty}
    \Prob\!\left( 
        a^\top \beta^\star \in \mathcal{I}^{\mathrm{Wald}}_T(a)
    \right)
    \;=\;
    1 - \alpha.
\end{align*}
Thus, stability ensures that the classical (nonadaptive) form of the Wald interval remains 
valid despite the adaptively collected data. See Section~\ref{sec:Proof-of-thm:stability}~for a proof of Theorem~\ref{thm:clt}. 

We now show that the stability and central limit theorem established in
Theorem~\ref{thm:clt} continue to hold when the ordinary least-squares estimator is
replaced by a ridge estimator. Introducing a ridge penalty is natural in adaptive
experimentation settings, where the sample covariance matrix may be ill-conditioned or
singular in finite samples due to exploration constraints. Let $\rpen > 0$ denote a
regularization parameter and define the regularized Gram matrix
$R_T := S_T + \rpen \,\Id$. We consider the ridge estimator
\begin{align}
   \widehat{\beta}_{\mathrm{rid}}
:= R_T^{-1} S_T \beta^\star + R_T^{-1} \sum_{t=1}^T \z_t \varepsilon_t, 
\end{align}
where $\z_t := c(x_t,a_t)$. We show that, under the same stability conditions as in
Theorem~\ref{thm:clt}, Wald-type inference based on the ridge estimator remains
asymptotically valid:
\begin{corollary}\label{cor:rid}
    Under the set up of Theorem~\ref{thm:clt}, the ridge-estimator $\hat{\beta}_{rid}$ with $\rpen \ll \sqrt{T} $ satisfies 
\begin{align}
\notag
   \sup_{a \in \real^d / \{0\}} \dist\!\left(
        \frac{a^{\top}(\hat{\beta}_{rid} - \beta^\star)}
             {\sigma \sqrt{a^\top R_T^{-1} a}},\;
        Z
    \right)\;\leq\; C \left[
     \Psi(\gamma_T)^{1/3}  + \expower\left\{ - \frac{(\laml)^2}{32 +  8\laml/3} \cdot  T\right\} + \frac{1}{T^{1/3}}
    + \frac{\rpen}{ \ \sqrt{\Psi(\gamma_T) T}} \right] 
\end{align}
\end{corollary}
We prove this corollary in Appendix~\ref{appnd:ridge}. Corollary~\ref{cor:rid} has important implications in estimation of heterogeneous treatment effects which we discuss in the next subsection.

\subsection{Inference for Conditional Average Treatment Effects}
\label{sec:cate}

The \emph{conditional average treatment effect} (CATE) plays a central role in modern
causal inference, as it characterizes heterogeneity in treatment responses across
contexts. Formally, for a binary treatment $T \in \{0,1\}$ and covariates $X$, the CATE is
defined as
\begin{align*}
\tau(x) := \Exs\!\left[ W(1) - W(0) \mid X = x \right],
\end{align*}
where $W(a)$ denotes the potential outcome under treatment $a$.
CATEs are foundational to applications such as personalized medicine~\citep{powers2018some}, targeted
advertising~\citep{ascarza2018retention} and experimentation platforms~\citep{green2012modeling}, where decisions
are tailored to individual or contextual characteristics.

Despite its importance, valid pointwise inference for CATE for $\tau(x)$ is non-trivial even in non-adaptive settings, as CATE is a non-regular estimand in
general \citep{athey2016recursive,wager2018estimation,chernozhukov2018double,kunzel2019metalearners,nie2021quasi}. Furthermore, standard results do not generalize well to adaptive data-collection regimes—such as
contextual bandits or reinforcement learning—where treatment assignments depend on
past observations.

Our results provide a principled resolution to this inference problem for a
practically relevant class of models. Specifically, when the conditional mean outcome
is linear in known features,
\begin{align*}
\Exs[W_t \mid X_t = x, a_t = a] = c(x,a)^\top \beta^\star,
\end{align*}
the CATE admits the representation
\begin{align*}
\tau(x) = \big(c(x,1) - c(x,0)\big)^\top \beta^\star,
\end{align*}
that is, the CATE is a linear functional of the unknown parameter $\beta^\star$.
In this case, inference for CATE reduces to inference for a linear contrast of
$\beta^\star$.

When treatment assignments are generated adaptively using the regularized EXP4
algorithm (Algorithm~\ref{alg:penexp4-finite}), the stability guarantees established in
Theorem~\ref{thm:clt} and Corollary~\ref{cor:rid} ensure that ridge regression retains
a central limit theorem with the same asymptotic variance as in non-adaptive designs.
This allows for asymptotically valid Wald-type confidence intervals for CATE under adaptively collected data, without paying any additional price for adaptivity~\cite{abbasi2011improved}.
We formalize this implication below.

\begin{corollary}
\label{cor:cate}
Consider the setup of Section~\ref{sec:cate}
with noise satisfying Assumptions~\ref{assn:noise}--\ref{assn:Q-bdd}. Let treatment
assignments be generated by Algorithm~\ref{alg:penexp4-finite}, and let
$\widehat{\beta}_{\mathrm{ridge}}$ denote the ridge estimator with penalty
$\rpen \ll \sqrt{T}$ as in Corollary~\ref{cor:rid}. Then, for any fixed $x$ and contrast vector $a(x) := c(x,1) - c(x,0)$,
\begin{align*}
\dist\!\left(
        \frac{a^{\top}(\hat{\beta}_{rid} - \beta^\star)}
             {\sigma \sqrt{a^\top R_T^{-1} a}},\;
        Z
    \right)
   \;\leq\; C \left[
     \Psi(\gamma_T)^{1/3}  + \expower\left\{ - \frac{(\laml)^2}{32 +  8\laml/3} \cdot  T\right\} + \frac{1}{T^{1/3}}
    + \frac{\rpen}{ \ \sqrt{\Psi(\gamma_T) T}} \right] 
\end{align*}
where $R_T := S_T + \rpen \cdot \Id$.
\end{corollary}

\subsection{Regret guarantees}

We next turn to regret guarantees. The following theorem quantifies how the choice of $\gamma_T$ influences the regret, formalizing the tradeoff introduced by the stabilizing penalty. 

\begin{theorem}
\label{thm:regret}
Suppose Assumptions~\ref{assn:noise}-\ref{assn:Q-bdd} are in force. Then, 
for $T \ge 4$, under the stepsize 
$\eta = \sqrt{\frac{\log K}{T K}}$
and tuning parameters
$\epsilon = \frac{1}{KT}$ and $\lambda = \frac{\gamma_T}{\sqrt{T}}$,
Algorithm~\ref{alg:penexp4-finite} satisfies
\begin{align}
    \Regret(T)
    \;\le\;
    8\sqrt{T K \log K}
    \;+\;
    \gamma_T \log(KT)\sqrt{T}
    \;+\;
    \frac{4 \gamma_T^2 \log^3(KT)}{K^2 \sqrt{T}}.
    \label{eq:regret-bound}
\end{align}
\end{theorem}

\noindent 
We prove Theorem~\ref{thm:regret} in Section~\ref{sec:regret-thm-proof}.
Setting $\gamma_T = \sqrt{\log T}$ yields regret of order 
$O\!\left(\sqrt{T K \log(KT)}\right)$, matching minimax lower bounds up to
logarithmic factors. The same choice balances the convergence rate in 
Theorem~\ref{thm:clt}, leading to a dimension-free Wald-type interval with 
logarithmic convergence. Thus, despite being explicitly regularized for 
inferential stability, the procedure retains the characteristic efficiency of 
the EXP4 family.

\section{Experiments}
\label{sec:exps}

We evaluate the empirical behavior of the proposed \emph{regularized–EXP4} algorithm in a stochastic contextual bandit environment with linear loss structure. In this section we provide the simulation details for the standard OLS estimator. Experimental details for the ridge estimator are presented in Appendix~\ref{sec:ride-sim}.

\subsection{Simulation Environment}

\noindent The loss model follows a block–sparse linear structure
\begin{align*}
    \loss_t = \langle x_t , \theta_{a_t} \rangle + \varepsilon_t,
\end{align*}
where $\varepsilon_t\sim\text{Unif}(-0.1,0.1)$ and $x_t\in\mathbb{R}^{d_x}$ is a normalized Gaussian context vector with $\|x_t\|_2 \le 1$. Each arm $a\in\{1,\dots,A\}$ possesses an unknown parameter $\theta_a\in\mathbb{R}^{d_x}$, and the global coefficient vector
\begin{align*}
    \beta^\star = (\theta_1,\dots,\theta_A) \in \mathbb{R}^{A d_x}
\end{align*}
is normalized to satisfy $\|\beta^\star\|_2 \le 1$. The learner never observes $\theta_a$; instead, only bandit feedback $\loss_t$ is revealed.

We encode actions through a sparse block feature map
\begin{align*}
    c(x,a) = \big(0,\dots,x,\dots,0\big) \in \mathbb{R}^{A d_x},
\end{align*}

\paragraph{Extension of the feature map and parameter space.}
To ensure that the losses are non-negative,  we extend the original feature representation by introducing an intercept term. For each context–action pair $(x_t, a_t)$, the original feature vector
\begin{align*}
c(x_t, a_t) \in \mathbb{R}^{A d_x}
\end{align*}
is augmented as
\begin{align*}
\tilde{c}(x_t, a_t)
&=
\begin{pmatrix}
c(x_t, a_t) \\
1
\end{pmatrix}
\in
\mathbb{R}^{A d_x + 1}.
\end{align*}
\noindent Correspondingly, the unknown parameter vector is extended to
\begin{align*}
\tilde{\beta}^\star
&=
\begin{pmatrix}
\beta^\star \\
2
\end{pmatrix}
\in
\mathbb{R}^{A d_x + 1},
\end{align*}
Under this augmented representation, the loss model becomes
\begin{align*}
\loss_t
&=
\langle \tilde{c}(x_t, a_t), \tilde{\beta}^\star \rangle + \varepsilon_t,
\end{align*}
which is algebraically equivalent to the original linear model but explicitly accounts for a constant offset in the losses. In particular, as $|\varepsilon_t| \leq 0.1,$ and $|c(x_t,a_t)| \leq 1$ adding shift of constant $2$ to the observed loss ensures that the resultant losses are non-negative.

\noindent Our experiments are conducted in two settings. The first assumes experts based on softmax policies, and the second replaces them with  six layer neural network policies. We outline the specifics of each configuration in the following sections.

\subsubsection*{Simulation setting with softmax experts}

\noindent Each expert network produces action probabilities through a softmax map: 
\begin{align*}
\pi_k(a\mid x)= \frac{e^{\left\langle u_{a,k},x \right\rangle}}{\sum_{a \in \mathcal{A}} e^{\left\langle u_{a,k},x \right\rangle}} 
\end{align*}
where, $k \in [K]$ and the entries of the  weight vector $u_{a,k}$ are i.i.d.  draws from $\mathcal{N}(0,0.04)$ distribution. In our experiments we set $A = K = 5$ and $d_x = 10$.

\subsubsection*{Simulation setting with neural experts}

Unlike the previous setting, the experts now form neural policies with a six-layer ReLU architecture. The expert policy is a six-layer neural network given by
\begin{align*}
   x \;\longrightarrow\; h_1 \;\longrightarrow\; h_2 \;\longrightarrow\; h_3
\;\longrightarrow\; h_4 \;\longrightarrow\; h_5 \;\longrightarrow\; h_6
\;\longrightarrow\; \operatorname{softmax}(\text{logits}), 
\end{align*}
where the hidden layers satisfy
\begin{align*}
   h_i = \operatorname{ReLU}\!\left( W_i h_{i-1} + b_i \right),
\qquad i = 1,\ldots,6, 
\end{align*}
with \(h_0 = x\). The entries of the weight matrices $W_i$ are i.i.d. draws from a standard Gaussian random variable. 
The resulting expert policy is
\begin{align*}
    \pi(a \mid x) = \frac{\exp(\text{logits}_a)}{\sum_{a'=1}^A \exp(\text{logits}_{a'})}.
\end{align*}
In our experiments we set $A  = 3$ and $K = 5$ and $d_x = 50$.

\subsection{Algorithmic Configuration}

\noindent  Let $\z_t = c(x_t,a_t)$.
At the end of horizon $T$, we compute the OLS estimator
\begin{align*}
    \widehat\beta_T = S_T^{-1} b_T, \qquad
    S_T = \sum_{t=1}^T \z_t \z_t^\top,\quad
    b_T = \sum_{t=1}^T \z_t \loss_t.
\end{align*}

\noindent Hyperparameters for Algorithm~\ref{sec:algo} are set to be
\begin{align*}
\varepsilon = \frac{1}{KT},\qquad
\lambda_{\rm pen}= \frac{\sqrt{\log T}}{\sqrt{T}},\qquad
\eta = \sqrt{\frac{\log K}{ \numactions T}},
\end{align*}
unless stated otherwise. We draw a random unit direction $a\in\mathbb{R}^{A d_x}$ and for each confidence level $\alpha \in [0.20, 0.01]$, check whether the true parameter lies inside the interval. Concretely we check if the target parameter $a^\top\beta^\star$ lies in the interval 
\begin{align*}
\mathcal{I}^{\mathrm{APS}}_T(a) :=  \big[a^\top\widehat\beta_{\mathrm{ridge}}-\xi_T\sqrt{a^\top \VT^{-1}a},\;
a^\top\widehat\beta_{\mathrm{ridge}}+\xi_T\sqrt{a^\top \VT^{-1}a}\big],
\end{align*}
where $\xi_T$ is
\begin{align}
    \xi_T :=  \sqrt{
2\!\left(
\frac{1}{2}\log\frac{\det(V_t)}{\det(\lambda I)}
\;+\;
\log\frac{1}{\alpha}
\right)
}
\;+\;
\sqrt{\lambda}\,\|\beta^\star\|_2
\end{align}
We note that the confidence intervals $\mathcal{I}^{\mathrm{APS}}_T(a)$ defined above are sharper than the anytime valid confidence interval defined in equation~\eqref{eq:APS-interval}. Wald coverage is measured analogously using the confidence
 interval $ \mathcal{I}^{\mathrm{Wald}}_T(a)$ from~\eqref{eq:wald-interval} , with $\widehat{\sigma}$ as the sample standard deviation estimate~\cite[Lemma 3]{lai1982least}. For each $T \in \{500, 3000 \}$, we report empirical coverage of $\mathcal{I}^{\mathrm{APS}}(a)$ vs $\mathcal{I}^{\mathrm{Wald}}(a)$, and their average width.

\subsection{Simulation Plots}

\noindent To quantify inferential quality we compute, over $N_{\mathrm{runs}}=1200$ Monte Carlo trials, the empirical coverage and interval width of confidence intervals produced by our method and a Wald-type baseline. 

\subsubsection*{Softmax Experts}

\noindent We vary the horizon and confidence level: 
\begin{align*}
T\in\{500,3000\},\qquad
A=K=5,\qquad
d_x=30,
\end{align*}

    \begin{figure}[H]\label{fig-ucb}
    \centering

    \begin{subfigure}[t]{0.42\textwidth}
        \includegraphics[width=\textwidth,trim={1cm 0.5cm 1cm 0.5cm}]{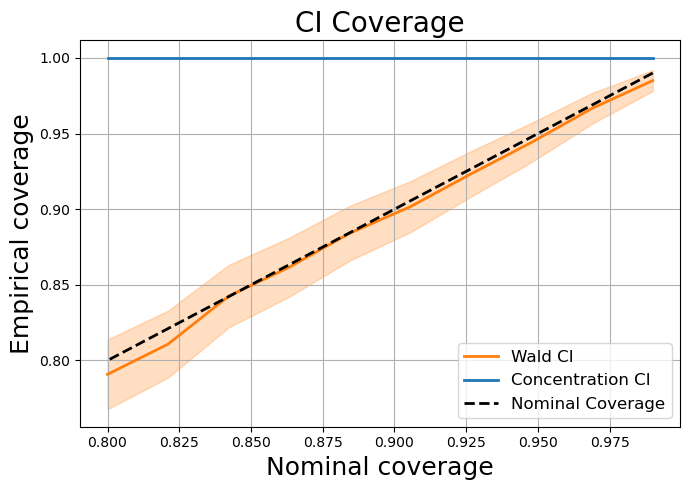}
    \end{subfigure}
    \hfill
    \begin{subfigure}[t]{0.42\textwidth}
        \includegraphics[width=\textwidth,trim={1cm 0.5cm 1cm 0.5cm}]{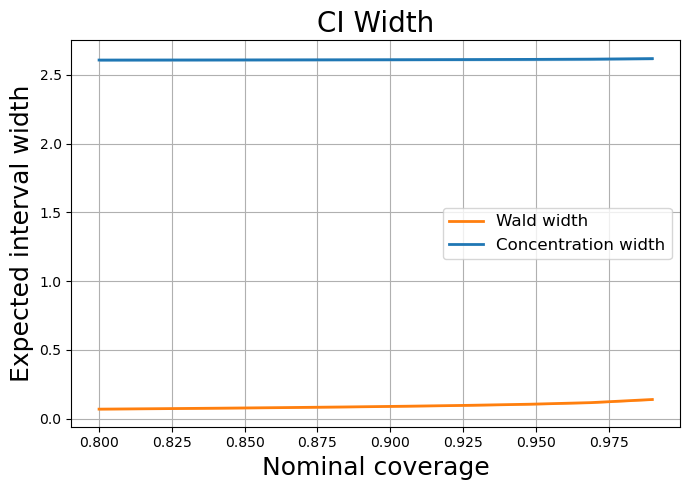}     
    \end{subfigure}
    \\[4pt]
    \caption{\textbf{Left}: Coverages of both $\mathcal{I}^{\mathrm{APS}}$ and $\mathcal{I}^{\mathrm{Wald}}$.  \textbf{Right}: Expected confidence width of  both $\mathcal{I}^{\mathrm{APS}}$ and $\mathcal{I}^{\mathrm{Wald}}$. The average CI widths of $\mathcal{I}^{\mathrm{Wald}}$ and $\mathcal{I}^{\mathrm{APS}}$ across all values of $\alpha$ are $0.08$ and $2.61$ respectively.  Simulations are based on \textbf{$T=500$} runs. }  
\end{figure}
    
    \begin{figure}[H]\label{fig-ucb}
    \centering

    \begin{subfigure}[t]{0.42\textwidth}
        \includegraphics[width=\textwidth,trim={1cm 0.5cm 1cm 0.5cm}]{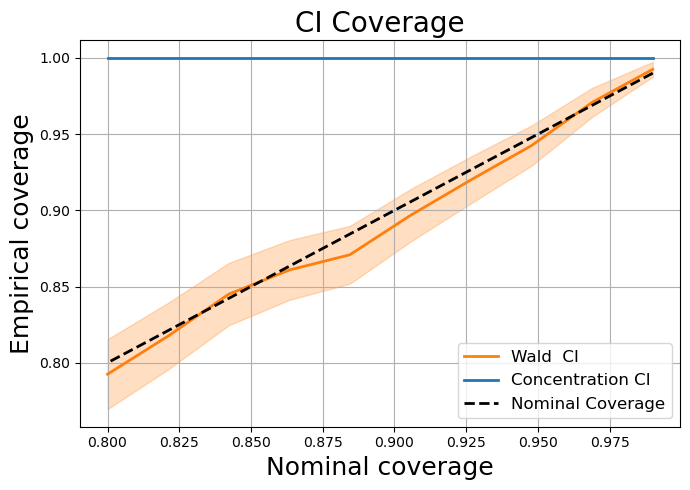}
    \end{subfigure}
    \hfill
    \begin{subfigure}[t]{0.42\textwidth}
        \includegraphics[width=\textwidth,trim={1cm 0.5cm 1cm 0.5cm}]{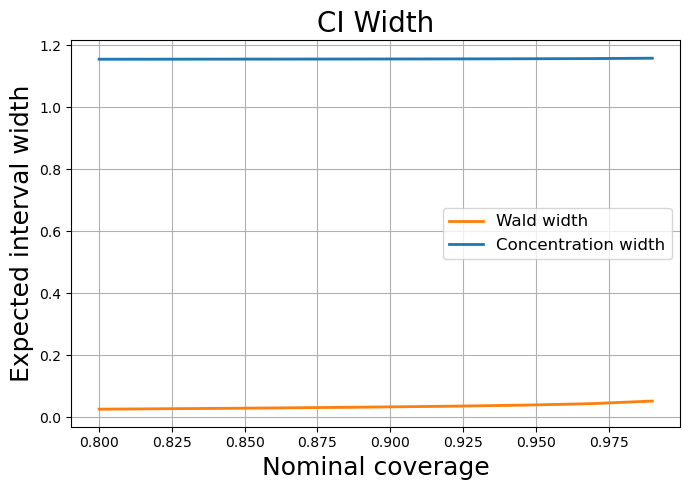}     
    \end{subfigure}
    \\[4pt]
    \caption{\textbf{Left}: Coverages of both $\mathcal{I}^{\mathrm{APS}}$ and $\mathcal{I}^{\mathrm{Wald}}$.  \textbf{Right}: Expected confidence width of  both $\mathcal{I}^{\mathrm{APS}}$ and $\mathcal{I}^{\mathrm{Wald}}$. The average CI widths of $\mathcal{I}^{\mathrm{Wald}}$ and $\mathcal{I}^{\mathrm{APS}}$ across all values of $\alpha$ are $0.03$ and $1.15$ respectively.  Simulations are based on \textbf{$T=3000$} runs.} 
\end{figure}

 \begin{figure}[H]\label{fig-ucb}
    \centering

    \begin{subfigure}[t]{0.42\textwidth}
        \includegraphics[width=\textwidth,trim={1cm 0.5cm 1cm 0.5cm}]{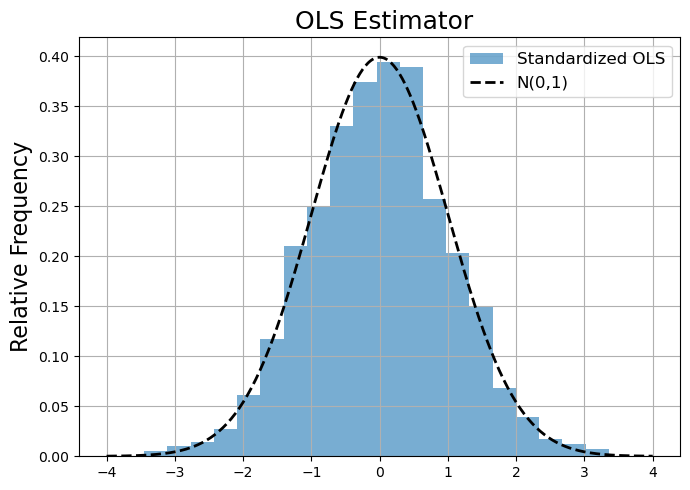}
        \caption*{$\frac{a^{\top}  \left( \betaols - \beta^\star \right)}{ \sqrt{ a^{\top} S_T^{-1} a}}$}
    \end{subfigure}
    \hfill
    \begin{subfigure}[t]{0.42\textwidth}
        \includegraphics[width=\textwidth,trim={1cm 0.5cm 1cm 0.5cm}]{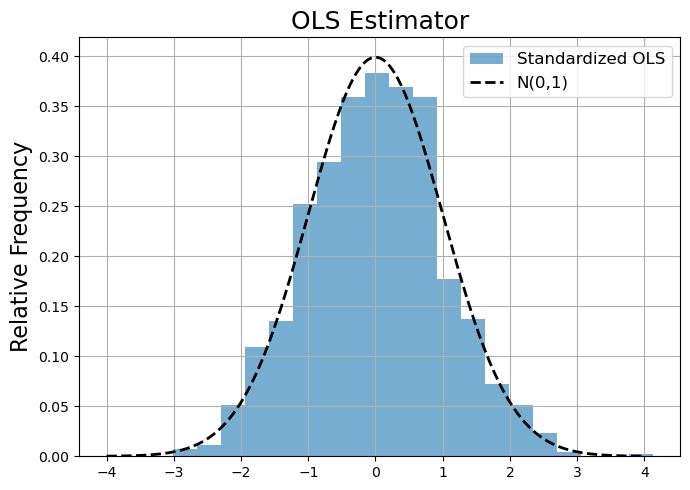}
        \caption*{$\frac{a^{\top}  \left( \betaols - \beta^\star \right)}{ \sqrt{ a^{\top} S_T^{-1} a}}$}
    \end{subfigure}
    \\[4pt]
    \caption{\textbf{Left}: Histogram of the standardized OLS estimator for the Softmax Experts with $T=500$.  \textbf{Right}: Histogram of the standardized OLS estimator for the Softmax Experts with $T=3000$.} 
\end{figure}

\subsubsection*{Neural Experts}    

\noindent We vary the horizon and confidence level: 
\begin{align*}
T\in\{500,3000\},\qquad
A=3,K=5,\qquad
d_x=50,
\end{align*}

 \begin{figure}[H]\label{fig-ucb}
    \centering

    \begin{subfigure}[t]{0.42\textwidth}
        \includegraphics[width=\textwidth,trim={1cm 0.5cm 1cm 0.5cm}]{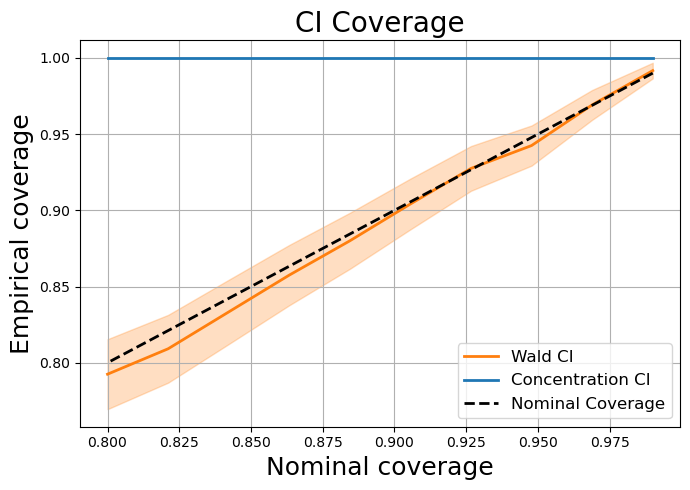}
    \end{subfigure}
    \hfill
    \begin{subfigure}[t]{0.42\textwidth}
        \includegraphics[width=\textwidth,trim={1cm 0.5cm 1cm 0.5cm}]{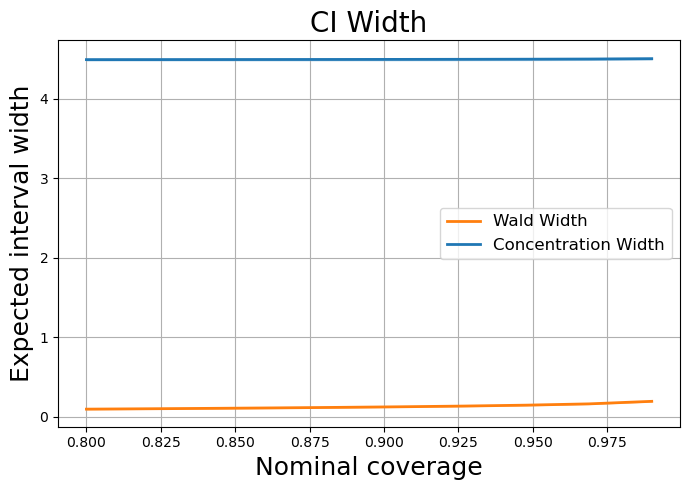}     
    \end{subfigure}
    \\[4pt]
    \caption{\textbf{Left}: Coverages of both $\mathcal{I}^{\mathrm{APS}}$ and $\mathcal{I}^{\mathrm{Wald}}$.  \textbf{Right}: Expected confidence width of  both $\mathcal{I}^{\mathrm{APS}}$ and $\mathcal{I}^{\mathrm{Wald}}$. The average CI widths of $\mathcal{I}^{\mathrm{Wald}}$ and $\mathcal{I}^{\mathrm{APS}}$ across all values of $\alpha$ are $0.12$ and $4.55$ respectively. Simulations are based on \textbf{$T=500$} runs.}  
\end{figure}
    
    \begin{figure}[H]\label{fig-ucb}
    \centering

    \begin{subfigure}[t]{0.42\textwidth}
        \includegraphics[width=\textwidth,trim={1cm 0.5cm 1cm 0.5cm}]{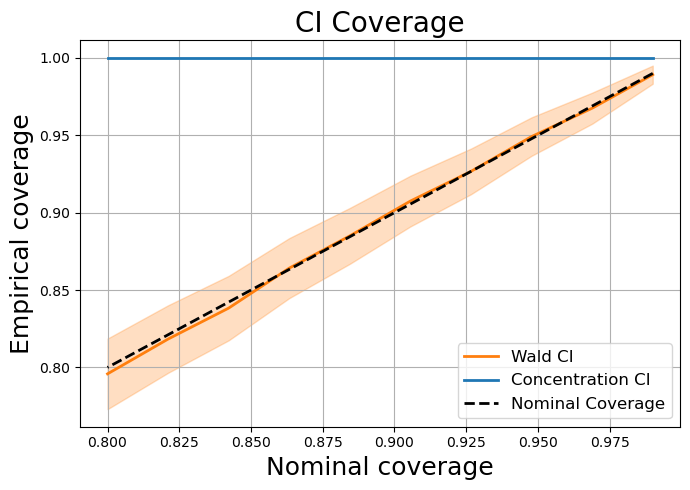}
    \end{subfigure}
    \hfill
    \begin{subfigure}[t]{0.42\textwidth}
        \includegraphics[width=\textwidth,trim={1cm 0.5cm 1cm 0.5cm}]{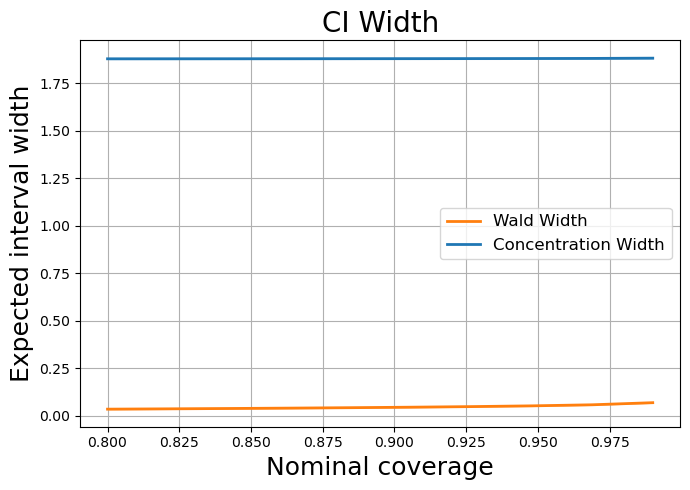}     
    \end{subfigure}
    \\[4pt]
    \caption{\textbf{Left}: Coverages of both $\mathcal{I}^{\mathrm{APS}}$ and $\mathcal{I}^{\mathrm{Wald}}$.  \textbf{Right}: Expected confidence width of  both $\mathcal{I}^{\mathrm{APS}}$ and $\mathcal{I}^{\mathrm{Wald}}$. The average CI widths of $\mathcal{I}^{\mathrm{Wald}}$ and $\mathcal{I}^{\mathrm{APS}}$ across all values of $\alpha$ are $0.06$ and $1.88$ respectively.
    Simulations are based on \textbf{$T=3000$} runs.} 
\end{figure}

   \begin{figure}[H]\label{fig-ucb}
    \centering

    \begin{subfigure}[t]{0.42\textwidth}
        \includegraphics[width=\textwidth,trim={1cm 0.5cm 1cm 0.5cm}]{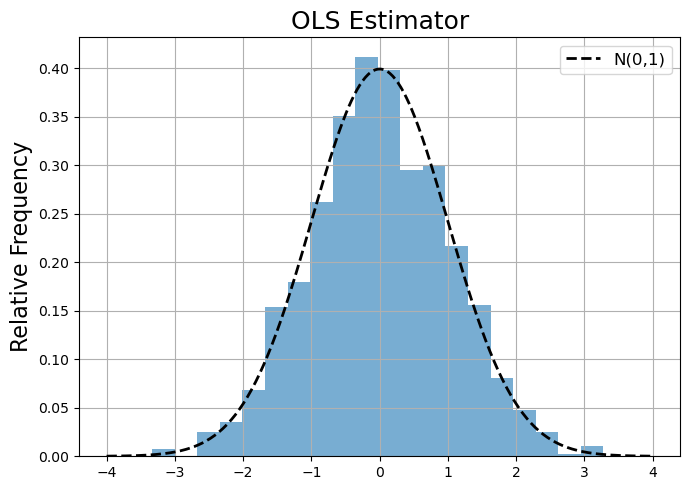}
        \caption*{$\frac{a^{\top}  \left( \betaols - \beta^\star \right)}{ \sqrt{ a^{\top} S_T^{-1} a}}$}
    \end{subfigure}
    \hfill
    \begin{subfigure}[t]{0.42\textwidth}
        \includegraphics[width=\textwidth,trim={1cm 0.5cm 1cm 0.5cm}]{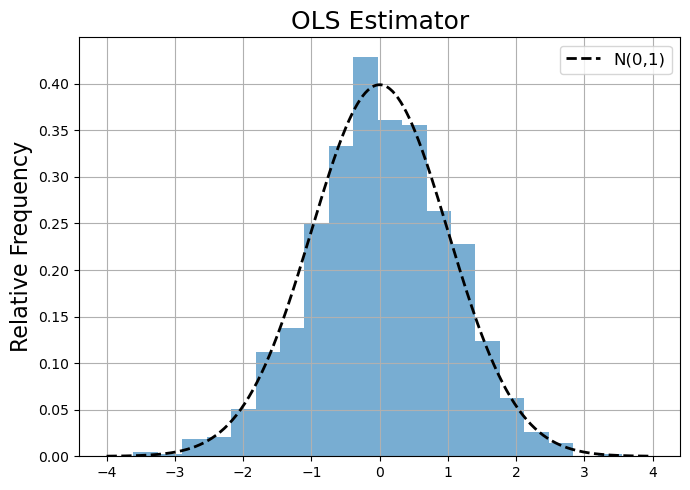}
        \caption*{$\frac{a^{\top}  \left( \betaols - \beta^\star \right)}{ \sqrt{ a^{\top} S_T^{-1} a}}$}
    \end{subfigure}
    \\[4pt]
    \caption{\textbf{Left}: Histogram of the standardized OLS estimator for the $6$-layer Neural-Network Experts with $T=500$.  \textbf{Right}: Histogram of the standardized OLS estimator for the $6$-layer Neural-Network Experts with $T=3000$.} 
\end{figure}

\section{Proof of Theorems}\label{sec:proof_of_theorems}
In this section we prove Theorems~\ref{thm:clt} and~\ref{thm:regret}.
Recall from Assumption~\ref{assn:bonded-loss} that $S_T = \sum^T_{t = 1} \cfeat(x_t, a_t)\cfeat(x_t, a_t)^\top$ where, the random vector $\cfeat(x,a) \in \real^d$ satisfies
\begin{align}
   \sup_{x,a}\|\cfeat(x,a)\|_2  \leq 1. 
\end{align}

\noindent By stability (definition~\eqref{defn:stability}) we mean that the sample covariance matrix satisfies the following property.
\begin{align}
\label{eqn:cov-stability}
\BT^{-1}S_T \inprob \mathbf{I} \qquad \text{where} \;\; \{\BT\} \;\; \text{is a sequence of non-random positive definite matrices.}
\end{align}
Let us define the matrix 
\begin{align*}
 Y_t := \cfeat(x_t, a_t)\cfeat(x_t, a_t)^\top.   
\end{align*}
We denote by $\mathcal{F}_{t}:=\sigma(x_1,a_1,r_1,\ldots,x_t,a_t,\loss_t)$, the $\sigma$-field generated by all observations up to time $t$,
including $x_t$, $a_t$ and $\loss_t$.
Since $x_t \iid \mathcal{P_X}$, the conditional expectation
\begin{align}\label{defn:condi-Yt}
   \mathbb{E}[Y_t \mid \mathcal{F}_{t-1}]
= \mathbb{E}_{x_t}\!\left[
\sum_{a\in \mathcal{A}} Q_t(a\mid x_t)\,\cfeat(x_t,a)\cfeat(x_t,a)^\top
\right]
= \sum^K_{k=1}w_{t,k}\Sigma_k  
\end{align}

\noindent where the matrices
\begin{align}\label{defn:Sigma}
   \Sigma_k := \mathbb{E}_{x}\!\left[\sum_{a\in \mathcal{A}} \pi_k(a\mid x)\,\cfeat(x,a)\cfeat(x,a)^\top\right],
\qquad k=1,\ldots,K 
\end{align}
are population-level second moments under each expert.
In order to simplify notations, throughout we assume $\sigma = 1$.

\subsection{Proof of Theorem~\ref{thm:clt}}
\label{sec:Proof-of-thm:stability}

\noindent We first prove stability holds for the standard OLS estimator, assuming that $S_T$ is invertible. Once the stability condition~\eqref{defn:stability} is verified, the asymptotic normality claim follows directly by invoking Theorem 3 from~\cite{lai1982least} along with an application of Slutsky's theorem. 

\subsection*{Part (a): proof of stability condition for $S_T$~\eqref{eqn:cov-stability}}

\noindent In order to prove the stability of the sample covariance matrix $S_T$, we need to first prove that the average of the \emph{random} weight vectors $\{w_{t}\}_{t \geq 0}$ converge to a non-random vector, which we call $w_T^\star$.  Concretely, 
\begin{align}
\label{eqn:weight-stability}
    \frac{1}{T}\sum_{t = 1}^T w_{t,k}-\wstar_{T,k}  \inprob 0
\end{align}
for all experts $k \in [K]$.  

\subsubsection*{Proof of weight stability~\eqref{eqn:weight-stability}:}

\noindent We begin by recalling that $\gtilde_t = \ghat_t + \lambda \grad R(w_t)$. Let $\tilde{R}_\lambda(w) := \lambda R(w) + \langle \bar g^\star, w\rangle $. If $f$ is any $M$-strongly convex function, then $f+g$ is also  $M$-strongly convex if $g$ is a linear function (\citet{nesterov2013introductory}). This result implies that $\tilde{R}_\lambda$ is $\lambda$ strongly convex. Now, for any arbitrary $y \in \Delta_\epsilon$ we have,
\begin{align} \label{eqn:master-ineqone}
\notag
  \dfrac{1}{T} \sum_{t=1}^T \Exs\langle \tilde{g}_t , w_{t} - y \rangle 
   & = \dfrac{1}{T} \sum_{t=1}^T \Exs\langle \bar g^\star + \lambda \nabla R(w_{t}) , w_{t} - y \rangle \\[8pt]\notag
   & = \dfrac{1}{T} \sum_{t=1}^T \bigg( \Exs\langle \bar g^\star, w_{t} - y \rangle  + \lambda \Exs \langle \nabla R(w_{t}) , w_{t} - y \rangle \bigg) \\[8pt]\notag
   & \overset{(i)}{\geq} \dfrac{1}{T} \sum_{t=1}^T \bigg( \Exs\langle \bar g^\star, w_{t} - y \rangle  + \lambda \Exs( R(w_{t}) - R(y)) \bigg) \\[8pt]\notag
   & \geq \dfrac{1}{T} \sum_{t=1}^T \bigg( \Exs [\langle \bar g^\star, w_{t} \rangle +  \lambda R(w_{t})]  -\Exs [\langle \mu, y \rangle +  \lambda R(y)] \bigg) \\[8pt]\notag
   & \overset{(ii)}{=}  \dfrac{1}{T} \sum_{t=1}^T \Exs \bigg[\tilde{R}(w_{t})  - \tilde{R}(y) \bigg] \\[8pt]\notag
   & \overset{(iii)}{\geq}  \Exs \bigg[\tilde{R}(\bar{w}_{T})  - \tilde{R}(y) \bigg]\\[8pt]
   & \overset{(iv)}{=}  \Exs \bigg[D_{\tilde{R}_\lambda}(\bar{w}_{T},y) + \langle \nabla \tilde{R}_\lambda(y), \bar{w}_{T}-y\rangle) \bigg]
\end{align}

\noindent  Where $D_{\tilde{R}_\lambda}$ is the Bregman divergence defined in terms of  $\tilde{R}_\lambda$. The first equality above follows by the unbiasedness of the gradient estimate $\widehat{g}_t$ (Section~\ref{sec:unbiased-grad}).  Equation $(i)$ above follows from convexity of $R$, equation $(ii)$ follows from definition of $\tilde{R}$ and $(iii)$ holds due to convexity of $\tilde{R}_{\lambda}$. Equality $(iv)$ follows from the definition of Bregman divergence. Note that the above discussion holds for any $y \in \Delta_\epsilon$.  Let us define $\wstar_T$ as follows
\begin{align}\label{defn:wstar}
    \wstar_T = \arg \min_{w \in \Delta_\epsilon} \tilde{R}_\lambda(w)
\end{align}

\noindent If we choose $y = \wstar_T$ then by applying the first order optimality criteria of equation~\eqref{defn:wstar} and Pinsker's inequality we have 
\begin{align}\label{eqn:master-ineqtwo}
  \langle \nabla \tilde{R}_\lambda(\wstar_T), \bar{w}_{T}-\wstar_T \rangle \geq 0
  \quad \text{and,} \ \ D_{\tilde{R}_\lambda}(\bar{w}_{T},\wstar_T) \geq \dfrac{\lambda}{2} \|\bar{w}_{T} - \wstar_T \|_1^2 
\end{align}

\noindent The above inequality is justified as on $\Delta_\varepsilon$, the Bregman divergence induced by the (shifted) negative entropy satisfies
$D_R(p,q)=\mathrm{KL}(p\|q)$ (the additive and linear terms cancel on the simplex).
From equations~\eqref{eqn:master-ineqone} and ~\eqref{eqn:master-ineqtwo}, we arrive at the following lower bound:
\begin{align}\label{eqn:master-ineqthree}
    \dfrac{1}{T} \sum_{t=1}^T \Exs\langle \tilde{g}_t , w_{t} - w^\star_T \rangle 
    \geq \dfrac{\lambda}{2} \Exs \|\bar{w}_{T} - \wstar_T \|_1^2 
\end{align}

\noindent Now,  the following lemma provides an upper bound on $\frac{1}{T} \sum_{t=1}^T \Exs\langle \tilde{g}_t , w_{t} - w^\star_T \rangle $.
\begin{lemma}\label{lemma:wmorm}
Under the setup of Theorem~\ref{thm:clt}, we have
\begin{align*}
    \dfrac{1}{T} \sum_{t=1}^T \Exs\langle \tilde{g}_t , w_{t} - \wstar_T \rangle 
    \leq 
    5{\frac{\sqrt{\numactions \log K}}{\sqrt{T}}} +
     \frac{4\gamma_T^2 \log(K) \log^2\left( KT \right) }{ \numactions^2 \sqrt{T}  }
\end{align*}
     
\end{lemma}

\noindent We  prove this lemma in Section~\ref{sec:regret-thm-proof}
. From equation~\eqref{eqn:master-ineqthree} and Lemma~\ref{lemma:wmorm} it follows from Jensen's inequality that,
\begin{align}\label{eqn:wlone-err}
\notag
    \Exs[\|\bar{w}_{T} - \wstar_T \|_1] 
    &\leq \sqrt{\frac{2}{\lambda_T}} \sqrt{\dfrac{1}{T} \sum_{t=1}^T \Exs\langle \tilde{g}_t , w_{t} - \wstar_T \rangle } \notag \\[8pt] \notag
    & \leq  \sqrt{10{\frac{\sqrt{\numactions \log K}}{\sqrt{T}\lambda_T}} +
     \frac{8\gamma_T^2 \log(K) }{ \lambda_T \numactions^2 T \sqrt{T}}  \log^2\left( KT \right)}\\[8pt] \notag
    & =  \sqrt{10{\frac{\sqrt{\numactions \log K}}{\gamma_T}} +
     \frac{8\gamma_T \log(K) \log^2\left( KT \right) }{ \numactions^2 T  }  }\\[8pt]
    & = \err
\end{align}

\noindent  Hence, if $\numactions$ and $K$ do not vary with $T$ then for any $\gamma_T \rightarrow \infty$ such that $ T/ \log^2 T \gg \gamma_T$ it follows that $\frac{1}{T}\sum_{t = 1}^T w_{t,k}-\wstar_{T,k}  \inprob 0$ for all experts $k \in [K]$. 

\noindent Now we are ready to prove stability of $S_T$. We first decompose $S_T$ into two parts:
\begin{subequations}
\begin{align}\label{defn:MT-BT}
    &S_T = M_T + B_T \qquad \text{where}  \\
&D_t := Y_t - \mathbb{E}[Y_t\mid \mathcal{F}_{t-1}], \;\;   M_T := \sum_{t=1}^T D_t, \;\; \text{and} \;\;  B_T := \sum_{t=1}^T \mathbb{E}[Y_t\mid \mathcal{F}_{t-1}].  
\end{align}
\end{subequations} 
\noindent Now, we define
\begin{align}\label{defn:Sigma-star}
    \frac{1}{T} \BT := \sum^K_{k = 1} \wstar_{T,k} \Sigma_k
\end{align}
where 
\begin{align*}
    \Sigma_k :=   \mathbb{E}_{x \sim \mathcal{P_X}} \left( \sum_{a \in \mathcal{A}} \pi_k(a \mid x) \cdot c(x, a) c(x, a)^\top \right)
\end{align*}
\noindent We show, using Lemma~\ref{lemma:MB-inprob} stated below, that
\begin{align}
\label{eqn:ST-decomp-II}
    \dfrac{S_T}{T} 
    = \dfrac{M_T}{T} + \dfrac{B_T}{T} 
     = \underbrace{\dfrac{M_T}{T}}_{o_{\Prob}(1)} + \underbrace{\left[\dfrac{B_T}{T}-\frac{1}{T} \BT \right]}_{o_{\Prob}(1)} + \frac{1}{T} \BT  
\end{align}

\begin{lemma}\label{lemma:MB-inprob}
    Let $\Sigma_k,M_T$ and $B_T$ be as defined in equations~\eqref{defn:MT-BT}. Suppose that there exists non-random weight vector $\wstar_T = (\wstar_{T,1}, \ldots , \wstar_{T,K})^\top$ such that $\frac{1}{T}\sum_{t = 1}^T w_{t,k}-\wstar_{T,k}  \inprob 0$ for all experts $k \in [K]$. Then we have,
    \begin{align}
        \Exs \left[\bigg| \bigg| \dfrac{M_T}{T} \bigg| \bigg|_{op} \right] \lesssim \dfrac{1}{\sqrt{T}} \quad \text{and}, \quad  \Exs \left[\bigg| \bigg| \dfrac{B_T}{T}-\sum^K_{k = 1} \wstar_{T,k} \Sigma_k \bigg| \bigg|_{op} \right] \leq \Exs[\|\bar{w}_T - w^\star_T \| ].
    \end{align}  
\end{lemma}

\noindent We prove this Lemma in Appendix~\ref{app:Aux-lemmas}. Let $\lambda^\star$ be the minimum eigenvalue among $\lambda_{\min}(\Sigma_k)$ for all $k\in [K]$. Note that $\lambda_{\min} \left(\frac{1}{T} \BT \right) \geq \lambda^\star > 0$ by Assumption~\ref{assn:context}. Hence, the minimum eigenvalue of $\Sigma_T^\star/T$ --- which is a convex combination of $\Sigma_k's$ --- is also lower bounded by $\lambda^\star$. This fact combined with the decomposition~\eqref{eqn:ST-decomp-II} yields 
\begin{align*}
    \BT^{-1}S_T \xrightarrow{\Prob} \Id.
\end{align*}.

\subsection*{Part (b): Proof of the quantitative CLT}

\noindent  The standard approach to prove such Berry Esseen bound is to uncover and utilize a martingale difference structure in our statistic of interest~(\citet{hall2014martingale}).  For notational simplicity, we assume that the sample covariance matrix $S_T$ is invertible. If $S_T$ is not invertible, a modified ridge estimator using $R_T: = S_T + \lambda_{rid} I$ exhibits equivalent asymptotic behavior. We highlight the details in Appendix~\ref{appnd:ridge}.

\noindent Fix a vector $a \in \real^d \setminus \{ 0\} $ and define 
\begin{align}\label{defn:stao}
   \stao : = \frac{a^{\top}  \left( \betaols - \beta^\star \right)}{ \sqrt{ a^{\top} S_T^{-1} a}} = \frac{a^{\top} S_T^{-1}}{ \sqrt{ a^{\top} S_T^{-1} a}} \sum^T_{t=1} \z_t  \varepsilon_t.
\end{align}
where $\z_t : = \cfeat(x_t, a_t)$ .We note that since $S_T$ is a random matrix measurable with respect to $\fil_T$, and $\stao$ is not a sum of martingale difference sequence. 

\noindent Define $\PT := \Exs[S_T]$ and consider the alternate statistic
\begin{align}\label{defn:stato}
    \stato := \frac{a^{\top} (\PT)^{-1}}{ \sqrt{ a^{\top} (\PT)^{-1} a}} \sum^T_{t=1} \z_t  \varepsilon_t
\end{align}

\noindent Observe that assumption~\ref{assn:context} ensures that $\PT$ is invertible. Let $b = (\PT)^{-1/2} \ a$. Algebraic manipulation yields that,
\begin{align*}
    \stato
   = \frac{b^{\top} (\PT)^{-1/2}}{ \|b \|_2} \sum^T_{t=1} \z_t  \varepsilon_t 
    = \frac{1}{ \|b \|_2} \sum^T_{t=1} h_t  \varepsilon_t
\end{align*}
where $h_t := b^{\top} (\PT)^{-1/2} \z_t$. We observe that
\begin{align*}
    \Exs[h_t\varepsilon_t \mid \fil_{t-1}]
    &=  \Exs[ \Exs[ h_t\varepsilon_t \mid \fil_{t-1},x_t,a_t] \mid \fil_{t-1}]\\[8pt]
    & = \Exs[h_t \Exs[ \varepsilon_t \mid \fil_{t-1},x_t,a_t] \mid \fil_{t-1}]\\[8pt]
    & = 0
\end{align*}

\noindent Therefore, $(h_t\varepsilon_t,\fil_t)_{t\geq 1}$ is a martingale difference sequence and one may analyze $\stato$ using standard martingale CLT (\citet{mourrat2013rate}).  Letting  $\MTP := k_T \mtp$ where $k_T = \sqrt{ a^{\top} \SigPinv a / a^{\top} S_T^{-1} a}$ we may rewrite our statistic of interest $\stao$ in terms of $\stato$:
\begin{align}
    \stao = \stato \ + \frac{a^{\top} \left(\MTP-\Id  \right)}{ \| b\|_2} \sum^T_{t=1} \SigPinv\z_t  \varepsilon_t 
    = \stato + \ratto
\end{align}
where $\ratto :=  \frac{a^{\top} \left(\MTP-\Id  \right)}{ \| b\|_2} \sum^T_{t=1} \SigPinv\z_t  \varepsilon_t $. 
In Lemma~\ref{lemma:stato-genbdd} we argue that the behavior of $\stao$ and $\stato$ are equivalent for large $T$.  Formally, we have 
\begin{lemma}\label{lemma:stato-genbdd}
   Let $\natto :=\left[\frac{1}{T} \PT \right]^{-1} \times \left[ \frac{M_T}{T} + \frac{B_T}{T}-\frac{1}{T} \PT\right]$ with $\PT = \Exs [S_T]$, and $Z$ be a copy of the standard normal random variable. Suppose that 
\begin{align}\label{eqn:assn-stato}
    \quad \dist \left(\stato, Z \right) \leq \beta_T
\end{align}
where $\beta_T$ is some positive real sequence. Then,
\begin{align}\label{eqn:stao-general-bdd}
    \dist \left(\stao, Z \right) \lesssim \beta_T + 2 \cdot \err^{1/2} \ + \  d \cdot \expower\left\{ - \frac{(\laml)^2}{32 +  8\laml/3} \cdot  T\right\}
\end{align}
where $\err$ as defined in equation~\eqref{defn:err}.
\end{lemma}

\noindent It now remains to analyze the term $\stato$, and we do so by using a Berry-Esseen bounds for martingale CLTs~\citep{mourrat2013rate}. Lemma~\ref{lemma:rate-stao} stated below characterizes $\beta_T$.
\begin{lemma}\label{lemma:rate-stao}
    Let $a\in \real^d$ be any arbitrary real vector, $\PT = \Exs[\sum^T_{t=1}\z_t \z^{\top}_t]$. Then
    \begin{align}\label{eqn:rate-stao}
    \dist \left(\stato,Z \right)
        \leq  \ C \left[ \Psi(\gamma_T)^{1/3} \ + \ \frac{1}{T^{1/3}}  \right]
    \end{align}
where $\err$ as is defined in equation~\eqref{defn:err} and $C$ is a constant independent of $T$.
\end{lemma}

\noindent Hence by combining equations~\eqref{eqn:stao-general-bdd} and~\eqref{eqn:rate-stao} together we obtain:
\begin{align*}
  \dist \left(\stao, Z \right)
  \lesssim \Psi(\gamma_T)^{1/3} \ + \ \frac{1}{T^{1/3}}
  \ + \ 2 \cdot \err^{1/2} \ + \  d \cdot \expower\left\{ - \frac{(\laml)^2}{32 +  8\laml/3} \cdot  T\right\}
\end{align*}

\noindent As $\err^{1/2}$ is strictly dominated by $\err^{1/3}$, we may ignore this term. This completes the proof of Theorem~\ref{thm:clt}. We prove Lemma~\ref{lemma:rate-stao} in Appendix~\ref{append:imp}, and prove our key Lemma~\ref{lemma:stato-genbdd} next.

\subsubsection*{Proof of Lemma~\ref{lemma:stato-genbdd}:}

\noindent Recall the following decomposition.
\begin{align} 
    \stao = \stato \ + \frac{a^{\top} \left(\MTP-\Id  \right)}{ \| b\|_2} \sum^T_{t=1} \SigPinv\z_t  \varepsilon_t 
    = \stato + \ratto
\end{align}
where $\ratto :=  \frac{a^{\top} \left(\MTP-\Id  \right)}{ \| b\|_2} \sum^T_{t=1} \SigPinv\z_t  \varepsilon_t $ and $\MTP := k_T \mtp$ such that
\begin{align*}
  k_T = \sqrt{ \frac{a^{\top} \SigPinv a }{ a^{\top} S_T^{-1} a}}  
\end{align*}

\noindent Now, fix a positive sequence $\delo$  whose explicit choice will be made later in the proof. Define event $\mathcal{E}_1(T)$ such that
\begin{align}
     \Eone := \left\{ |\ratto| \leq \delta_{1,T} \right\}
\end{align}

\noindent On event $\Eone$, we have
\begin{align}\label{eqn:ratto-sand}
    \stato - \delta_{1,T} \leq \stao \leq \stato + \delta_{1,T}
\end{align}

\noindent Furthermore, suppose that for every sequence $\delo$, there exists another sequence $\deltoto$ such that
\begin{align}\label{eqn:Eone-comp}
    \Prob(\Eone^c)
     \leq \delta_{2,T}
\end{align}

\noindent Now we are ready to prove our claim~\eqref{eqn:stao-general-bdd}, which we prove in two steps. We use  equation~\eqref{eqn:ratto-sand} to derive lower and upper bounds on $\Prob(\stao \leq x)-\Phi(x)$, which are free of $x$. 
\begin{align*}
    \Prob(\stao \leq x) - \Phi(x)
    &\geq  \Prob(\stao \leq x, \Eone) - \Phi(x)\\[8pt]
    &\geq  \Prob(\stato + \delta_{1,T} \leq x, \Eone) - \Phi(x)\\[8pt]
    &\geq  \Prob(\stato \leq x -\delta_{1,T}) - \Phi(x) - \Prob(\Eone^c )
\end{align*}

\noindent From the last inequality we have:
\begin{align*}
    &\Prob(\stao \leq x) - \Phi(x)\\[8pt]
    & \geq \bigg[ \underbrace{\Prob(\stato \leq x-\delta_{1,T}) - \Phi(x-\delta_{1,T})}_{\mathcal{I}_1} \bigg]
    + \bigg[ \underbrace{\Phi(x-\delta_{1,T}) - \Phi(x)}_{\mathcal{I}_2} \bigg]
    - \underbrace{\Prob(\Eone^c)}_{\mathcal{I}_3}
\end{align*}

\noindent Now, from our assumption~\eqref{eqn:assn-stato} we have $\mathcal{I}_1 \geq -\beta_T$. By applying first order Taylor expansion along with the fact that $\sup_x e^{-x^2}<1$ leads us to $\mathcal{I}_2 \gtrsim -\delta_{1,T}$. Finally, equation~\eqref{eqn:Eone-comp} implies that $\mathcal{I}_3 \leq \delta_{2,T}$. Hence, for all $x\in \real$,
\begin{align}
  \Prob(\stao \leq x) - \Phi(x) \gtrsim - \bigg[\beta_T +\delta_{2,T} + \delta_{1,T} \bigg]  
\end{align}

\noindent Now, for the upper bound, observe that
\begin{align*}
    \Prob(\stao \leq x) - \Phi(x)
    &=\Prob(\stao \leq x, \Eone) + \Prob(\stao \leq x, \Eone^c) - \Phi(x)\\[8pt]
    &\leq  \Prob(\stato - \delta_{1,T} \leq x, \Eone) + \Prob(\Eone^c ) - \Phi(x)\\[8pt]
    &\leq  \Prob(\stato \leq x+ \delta_{1,T}) - \Phi(x) + \Prob(\Eone^c )
\end{align*}

\noindent The last inequality leads to,
\begin{align*}
    &\Prob(\stao \leq x) - \Phi(x)\\[8pt]
    & \leq \bigg[ \underbrace{\Prob(\stato \leq x+\delta_{1,T}) - \Phi(x+\delta_{1,T})}_{\mathcal{I}_4} \bigg]
    + \bigg[ \underbrace{\Phi(x+\delta_{1,T}) - \Phi(x)}_{\mathcal{I}_5} \bigg]
    + \underbrace{\Prob(\Eone^c )}_{\mathcal{I}_3}
\end{align*}

\noindent Analogous calculations yield :
\begin{align}
  \Prob(\stao \leq x) - \Phi(x) \lesssim  \beta_T +\delta_{1,T} + \delta_{2,T}
\end{align}
\noindent Therefore,
\begin{align*}
    \sup_{x\in \real} |\Prob(\stao \leq x) - \Phi(x) |
    \lesssim  \beta_T +\delta_{1,T} + \delta_{2,T}
\end{align*}

\noindent \noindent Define $\lambda^\star_L$ and $\lambda^\star_U$ such that 
\begin{align}\label{defn:laml-lamu}
    \min_{k \in [K]}\lambda_{\min}(\Sigma_k) = \lambda^\star_L \quad \text{and} \quad 
    \max_{k \in [K]}\lambda_{\max}(\Sigma_k) = \lambda^\star_U.
\end{align}

\noindent Now, the lemma stated below, characterizes $\deltoto$ in terms of $\delo$, $ \lambda^\star_L$ and $\lambda^\star_U$.
\begin{lemma}\label{lemma:ratto}
     Under the setup of Theorem~\ref{thm:clt} we have:
    \begin{align}\label{eqn:ratto}
       \Prob (|\ratto| > \delo) \lesssim \frac{\err}{\delo} \ + \  d \cdot \expower\left\{ - \frac{(\laml)^2}{32 +  8\laml/3} \cdot  T\right\}
    \end{align}
\end{lemma}

\noindent We prove this lemma in Appendix~\ref{append:imp}.
We choose $\delo$ such that the above upper bound gets minimised.   By applying the AM-GM inequality, we obtain the desired bound by choosing $\delta_T = \sqrt{\err}$ and $\beta_T = \err^{1/3}$:
\begin{align*}
    \sup_{x\in \real} |\Prob(\stao \leq x) - \Phi(x) |
    \lesssim \beta_T + 2 \ \err^{1/2} \ + \  d \cdot \expower\left\{ - \frac{(\laml)^2}{32 +  8\laml/3} \cdot  T\right\}
\end{align*}

\noindent Note that $\err^{1/2}$ is strictly dominated by $\err^{1/3}$, and hence we ignore this term in the upper bound. Therefore, our proof is complete.

\subsection{Proof of Theorem~\ref{thm:regret}}
\label{sec:regret-thm-proof}

\noindent 
We first restrict the comparator $w$ to the truncated simplex. Without loss of generality we assume arm $1$ is \emph{among} the optimal arms. Now define 

\begin{align*}
    w^\star = (1, 0, \ldots, 0) 
    \qquad \text{and} \qquad 
    w_{\epsilon} = \big(1-(K-1)\varepsilon,\;\varepsilon,\ldots,\varepsilon\big)
  \in \Delta_\varepsilon. 
\end{align*}
To prove this theorem, we shall apply the master equation stated below.
\begin{lemma}
\label{thm:avg-conv}
 For any sequence of contexts $\{x_t\}_{t=1}^T$ and stochastic losses $\{\loss_t\}_{t=1}^T$, 
 the outputs $\{w_t\}_{t \geq 1}$ produced by Algorithm~\ref{alg:penexp4-finite} satisfy, for any $y \in \Delta_\epsilon$  
 \begin{align}
 \label{eqn:master-eqn}
     \frac{1}{T}\sum_{t = 1}^T \Exs \langle \widetilde{g}_t,w_t-y\rangle \le  \frac{D_\phi(y,w_1)}{\eta T} + 
     \frac{4\eta}{T} \sum_{t = 1}^T \left\lbrace 
     \Exs \left|\left|\widehat{g}_t\right|\right|_{w_t,\ast}^2 + \lambda^2\Exs \left|\left|\nabla R(w_t)\right|\right|_{w_t,\ast}^2 \right\rbrace.
 \end{align}

 \end{lemma}

\noindent By construction, $\wstar \in \arg\min_{w} \langle \gstar, w \rangle$, and $\weps \in  \deleps$. Consequently, we can apply Lemma~\ref{thm:avg-conv} with the comparator $w = \weps$. An application of Holder's inequality and using $\|\gstar\|_\infty \leq 2$ now yields  
\begin{align}
    \regT &:=  \sum_{t=1}^T \Exs\langle \bar{g}^\star, w_t - \wstar \rangle \nonumber \\  
    &\stackrel{(i)}{\leq}  \sum_{t=1}^T \Exs\langle \gstar, w_t - \weps \rangle
    + 4T K \epsilon \nonumber   \\
    &\stackrel{(ii)}{\leq}  \sum_{t=1}^T \Exs\langle \tilde{g}_t , w_t - \weps \rangle  
    + 4T K \epsilon + 2\lambda T \log(1/\epsilon) \label{eqn:bound-1}
\end{align}
Inequality $(i)$ utilizes the bound $|\langle \gstar, \wstar - \weps \rangle | \leq \|\gstar\|_\infty \|\wstar - \weps\|_1 \leq 4K \epsilon$; inequality $(ii)$ utilizes the relation 
\begin{align*}
\Exs \langle \gstar, w_t - \weps \rangle
&= \Exs \langle \hat{g}_t, w_t - \weps \rangle \\[8pt]
&= \Exs \langle \tilde{g}_t, w_t - \weps \rangle - \lambda \Exs  \langle \grad R(w_t), w_t - \weps \rangle \\[8pt]
& \leq \Exs \langle \tilde{g}_t, w_t - \weps \rangle + \lambda | \Exs  \langle \grad R(w_t), w_t - \weps \rangle | \\[8pt]
& \leq \Exs \langle \tilde{g}_t, w_t - \weps \rangle + \lambda \| \grad R(w_t)\|_\infty \Exs[ \|w_t - \weps \|_1] 
\end{align*}
and the bound $\|\grad R(w)\|_\infty \leq \log(1/\epsilon)$ for all $w \in \deleps$ along with Holder's inequality. It now remains to bound the inner product term $\sum_{t=1}^T \Exs\langle \tilde{g}_t , w_t - \weps \rangle$ using Lemma~\ref{thm:avg-conv}. This result is provided by Lemma~\ref{lemma:wmorm} in the proof of the stability property of $S_T$. For the sake of completeness, we state the lemma below.
\begin{lemma*}
    Under the setup of Theorem~\ref{thm:clt}, we have
\begin{align*}
    \dfrac{1}{T} \sum_{t=1}^T \Exs\langle \tilde{g}_t , w_{t} - \wstar_T \rangle 
    \leq 
    5{\frac{\sqrt{\numactions \log K}}{\sqrt{T}}} +
     \frac{4\gamma_T^2 \log(K) \log^2\left( KT \right) }{ \numactions^2 \sqrt{T}  }
\end{align*}
\end{lemma*}

\noindent The proof of Lemma~\ref{lemma:wmorm} is as follows.
Invoking Lemma~\ref{thm:avg-conv} with $y = \weps$ and setting  $\eta = \sqrt{\frac{\log K}{ \numactions T}}$, $\lambda = \frac{\gamma_T}{\sqrt{T}}$, and $w_1 = (1/K, \ldots, 1/K)$ we have 
\begin{align}
    \sum_{t=1}^T \Exs\langle \tilde{g}_t , w_t - \weps \rangle  
    &\leq \frac{D_\phi(\weps, w_1)}{\eta } + 
    4\eta \sum_{t = 1}^T \left\lbrace 
     \Exs \left|\left|\widehat{g}_t\right|\right|_{w_t,\ast}^2 + \lambda^2\Exs \left|\left|\nabla R(w_t)\right|\right|_{w_t,\ast}^2 \right\rbrace \nonumber  \\ 
     &\stackrel{(i)}{\leq} \sqrt{T\numactions \log K} + \frac{4\sqrt{\log K}}{\sqrt{\numactions T}} 
     \sum_{t = 1}^T  
     \Exs \left|\left|\widehat{g}_t\right|\right|_{w_t,\ast}^2 
     + \frac{4\gamma^2_T \log(K) }{ \numactions^2 T^{3/2}}  \sum_{t = 1}^T \Exs \left|\left|\nabla R(w_t)\right|\right|_{w_t,\ast}^2 \nonumber \\
     &\stackrel{(ii)}{\leq} 
     \sqrt{T\numactions \log K} +
     \frac{4\gamma^2_T \log(K) }{ \numactions^2 \sqrt{T}}  \log^2\left( \frac{1}{\epsilon} \right)
      + 
      \frac{4\sqrt{\log K}}{\sqrt{T \numactions}}  \sum_{t = 1}^T \Exs \left|\left|\widehat{g}_t\right|\right|_{w_t,\ast}^2 
     \label{eqn:bound-2}
\end{align}
Inequality $(i)$ above utilizes  $D(w, w_1) \leq \log K$ for all $w \in \Delta_K$. This claim holds because
\begin{align*}
    D(w,w_1) = \sum^K_{k = 1} w_k \log \left( \frac{w_k}
    {w_{1,k}}\right)
    \ \leq \ \log K \sum^K_{k = 1} w_k = \log K
\end{align*}

\noindent Inequality $(ii)$ uses that for any $w \in \deleps$
\begin{align*}
    \|\grad R(w)\|^2_{w, *} := \sum_{k=1}^K w_k \log^2\!\left(\frac{w_k}{\varepsilon}\right)
  \;\le\;
  \sum_{k=1}^K w_k \log^2\!\left(\frac{1}{\varepsilon}\right)
  = \log^2\!\left(\frac{1}{\varepsilon}\right). 
\end{align*}
It now remains to bound the delicate term $\Exs \left|\left|\widehat{g}_t\right|\right|_{w_t,\ast}^2$. 
\begin{lemma}\label{lemma:norm-ghat}
   Under the setup of Theorem~\ref{thm:clt} we have
   \begin{align}
\label{eqn:local-norm-term}
    \Exs \left|\left|\widehat{g}_t\right|\right|_{w_t,\ast}^2 \leq \numactions.
\end{align}
\end{lemma}
\noindent The proof of Lemma~\ref{lemma:norm-ghat} utilizes the property of the local norm $\|\cdot \|_{w_t, \star}$ and the specific form of the gradient estimate $\ghat_t$ from~\eqref{eq:ips-estimator}. The complete proof is provided in Appendix~\ref{app:Aux-lemmas}. Now,  
by substituting equation~\eqref{eqn:local-norm-term} and $\varepsilon = 1/KT$ in equation~\eqref{eqn:bound-2} we have
\begin{align*}
   \sum_{t=1}^T \Exs\langle \tilde{g}_t , w_t - \weps \rangle
   \leq 5\sqrt{T\numactions \log K} +
     \frac{4\gamma^2_T \log(K) }{ \numactions^2 \sqrt{T}}  \log^2\left( KT \right)
\end{align*}

\noindent Combining bounds~\eqref{eqn:bound-1},~\eqref{eqn:bound-2} and~\eqref{eqn:local-norm-term} we have 
\begin{align*}
    \regT &\leq 5\sqrt{T\numactions \log K} + \frac{4\gamma^2_T \log(K) }{ \numactions^2 \sqrt{T}}  \log^2\left( 1/\epsilon \right)
      +   4T K \epsilon + 2\lambda T \log(1/\epsilon) 
\end{align*}
Finally, substituting $\epsilon = \frac{1}{KT}, \ \lambda = \frac{\gamma_T}{\sqrt{T}}$ and using $T \geq 4$ yields 

\begin{align*}
    \regT \leq 8\sqrt{T\numactions \log K} + \gamma_T \log(KT) \sqrt{T}
    + \frac{4\gamma^2_T \log^3(KT) }{ \numactions^2 \sqrt{T}} 
\end{align*}

\section{Conclusion}
\noindent
We studied statistical inference in linear contextual bandits under adaptive data collection, with a
focus on the validity of classical Wald-type confidence intervals. By introducing a regularized variant
of the EXP4 algorithm, we demonstrated that it is possible to simultaneously enforce the Lai--Wei
stability condition and achieve regret guarantees that are minimax optimal up to logarithmic
factors. As a consequence, ordinary least-squares estimators satisfy a central limit theorem, and
Wald confidence intervals for linear functionals are asymptotically valid without incurring the
$\sqrt{d\log T}$ price of adaptivity. Beyond asymptotic validity, our analysis provides explicit rates of convergence to normality,
offering a quantitative characterization of how stability, regret, and inferential accuracy interact
in finite samples. As a consequence of our results we obtain valid Wald-type inference for conditional average treatment effects in linear models, even when treatments are assigned adaptively via contextual bandit algorithms. Together, these results establish that stability and statistical efficiency are not
fundamentally incompatible in contextual bandit problems, and that careful algorithmic design can
recover classical inferential guarantees even under adaptive sampling.

Several directions remain open for future work. It would be of interest to extend the present
analysis to settings with growing feature dimension or fully adaptive
contexts. Understanding whether analogous stability--regret tradeoffs can be achieved remains an important and challenging question.

\section*{Acknowledgments}
This work was partially supported by the National Science Foundation Grant DMS-2311304 to Koulik Khamaru.

\bibliographystyle{plainnat}
\bibliography{ref}

\newpage

\appendix

\section{Proof of important results}\label{append:imp}

\noindent In this section we provide the proof of Lemma~\ref{thm:avg-conv}.

\subsection*{Proof of Lemma~\ref{thm:avg-conv}}

\noindent 
We use the shorthand $\gtilde_t = \ghat_t + \lambda \grad R(w_t)$ as the gradient estimate of the regularized loss function at $w_t$. We have 
\begin{align}
\label{eqn:wtplus-eqn}
    \grad \phi(w^{+}_{t + 1}) = \grad \phi(w_t) - \eta \gtilde_t 
\end{align}
We begin by analyzing the quantity $\langle\eta \widetilde{g}_t,w_t-y\rangle$.
    \begin{align*}
        \langle\eta \widetilde{g}_t,w_t-y\rangle&=\langle\nabla\phi(w_t)-\nabla\phi(w^{+}_{t +1}),w_t-y\rangle\\
        &\stackrel{(i)}{=} D_\phi(y,w_t)+D_\phi(w_t,w^{+}_{t +1})-D_\phi(y,w^{+}_{t +1})\\
        &\stackrel{(ii)}{\leq} D_\phi(y,w_t)+D_\phi(w_t,w^{+}_{t +1})-D_\phi(y,w_{t+1})-D_\phi(w_{t+1},w^{+}_{t +1})\\
        &\le D_\phi(y,w_t)-D_\phi(y,w_{t+1}) + D_\phi(w_t,w^{+}_{t +1})
    \end{align*}
    Here, equality $(i)$ follows from the Bregman 3-point Lemma~\ref{lem:three-point-alt}, and inequality $(ii)$ utilizes Lemma~\ref{lem:pythagoras} which ensures  $D_\phi(y,w^{+}_{t +1})\ge D_\phi(y,w_{t+1})+D_\phi(w_{t+1},w^{+}_{t +1})$. Next, we bound $D_\phi(w_t,w^{+}_{t +1})$, for which we require the following lemma (Lemma $7.3$, \cite{bauschke2001essential}):
    \begin{lemma}\label{lemma:bregman-fenchel}
       For our choice of $\phi$ the following result holds for any $x,y \in \Delta_\epsilon$
       \begin{align*}
            D_\phi(x,y) = D_{\phi^\star}( \nabla \phi(y),\nabla \phi(x))
       \end{align*}
       where $\phi^\star$ is the Fenchel dual of $\phi$.
    \end{lemma}
    \noindent From Lemma~\ref{lemma:bregman-fenchel} and the definition of Bregman divergence we have,
    \begin{align}\label{eqn:breg-firstorder}
    \notag
        D_\phi(w_t,w^{+}_{t +1}) 
        &= D_{\phi^\star}( \nabla \phi(w^{+}_{t+1}),\nabla \phi(w_t))\\[8pt]
        & = \phi^\star(\nabla \phi(w^{+}_{t+1})) - \phi^\star(\nabla \phi(w_t))
        - \langle \nabla \phi^\star(\nabla \phi(w_t)) ,\nabla \phi(w^{+}_{t+1}) - \nabla \phi(w_t) \rangle
    \end{align}
    
\noindent Furthermore, as $\loss_t \geq 0$ (Assumption~\ref{assn:nonneg-loss}) the coordinate-wise positivity of $\gtilde_t$ in~\eqref{eqn:wtplus-eqn} yields $\nabla\phi(w^{+}_{t +1})\le\nabla\phi(w_t)$. Hence, by applying the second order Taylor expansion in equation~\eqref{eqn:breg-firstorder}, we have that for an intermediate point $ z_t = \delta \ \grad \phi(w^\star_{t+1})+ (1-\delta) \grad \phi(w_{t})$,
    \begin{align*}
        D_\phi(w_t,w^{+}_{t +1})=\frac{1}{2}\left(\nabla\phi(w^{+}_{t +1})-\nabla\phi(w_t)\right)^T \left[\nabla^2{\phi^\ast}(z_t)\right]\left(\nabla\phi(w^{+}_{t +1})-\nabla\phi(w_t)\right)
    \end{align*}
    where $\delta \in (0,1)$ and $\phi^\ast(y):=\sup_{x\in R^d_{>0}}\left\{\langle y,x\rangle-\phi(x)\right\}=\sum_{i=1}^d \exp\left\{y_i\right\}$ is the dual map corresponding to $\phi$.  Note that $\nabla^2\phi^\ast(y)=\text{Diag}\left(\exp\{y_1\},\ldots,\exp\{y_d\}\right)$. By combining this identity with the fact that $z_t \leq \nabla \phi(w_t)$ coordinate-wise we get,
    \begin{align*}
        \nabla^2{\phi^\ast}(z_t)\preceq\nabla^2{\phi^\ast}(\nabla\phi(w_t))=\nabla^2{\phi^*}\left(\log w_t \right)=\text{Diag}(w_1,\ldots,w_d)
    \end{align*}
    As $\|z_t \|_\infty \leq 1$ we have
    \begin{align*}
        D_\phi(w_t,w^{+}_{t +1})&\le \frac{1}{2}\left|\left|\nabla\phi(w_t)-\nabla\phi(w^{+}_{t +1})\right|\right|_{w_t,\ast}^2\\
        &=\frac{1}{2}\left|\left|\eta \widehat{g}_t+\eta\lambda\nabla R(w_t)\right|\right|_{w_t,\ast}^2\\
        &\le \eta^2\left|\left|\widehat{g}_t\right|\right|_{w_t,\ast}^2+\eta^2\lambda^2\left|\left|\nabla R(w_t)\right|\right|_{w_t,\ast}^2 
    \end{align*}

\noindent   
Overall, we arrive at
\begin{align}
\label{eqn:master-eqn}
    \langle\eta \widetilde{g}_t,w_t-y\rangle\le\left[D_\phi(y,w_t)-D_\phi(y,w_{t+1})\right]+\eta^2\left|\left|\widehat{g}_t\right|\right|_{w_t,\ast}^2+\eta^2\lambda^2\left|\left|\nabla R(w_t)\right|\right|_{w_t,\ast}^2
\end{align}
On taking full-expectation, the left hand side is bounded below by $\eta\E[f(w_t)-f(y)]$. Taking the average over $t=1$ to $T$ and  using the fact that $f$ is convex yields 
\begin{align*}
\E\left[ f(\bar{w}_T) - f(\wopt) \right] \le
\frac{\Exs D_\phi(\wopt,w_1)}{T} +  \frac{4}{T} \sum_{t = 1}^T \left\lbrace \eta^2\Exs \left|\left|\widehat{g}_t\right|\right|_{w_t,\ast}^2+\eta^2\lambda^2 \Exs \left|\left|\nabla R(w_t)\right|\right|_{w_t,\ast}^2 \right\rbrace 
\end{align*}
This completes the proof. 

\subsection*{Proof of Lemma~\ref{lemma:rate-stao}}

\noindent Define $\PT := \Exs[S_T]$ and recall that
\begin{align*}
    \stato
    = \frac{1}{ \|b \|_2} \sum^T_{t=1} h_t  \varepsilon_t
\end{align*}
where $h_t = b^{\top} (\PT)^{-1/2} \z_t$, and $b = (\PT)^{-1/2}a$. Note that as $\PT$ is deterministic and $\z_t$ are uniformly bounded, $ h_t\varepsilon_t$ is a square-integrable MDS. To prove a bound on $\dist(\stato,Z)$, we shall apply the following quantitative central limit theorem for martingales (\citet{mourrat2013rate}).

\begin{lemma}\label{lemma:mart-clt}
Let $(Z_t,\mathcal{F}_t)$ be a square-integrable martingale difference sequence. Let $s^2_T:= \sum^T_{t=1} \Exs[Z^2_t]$ and $V^2_T := \dfrac{1}{s^2_T} \sum \Exs[Z^2_t|\mathcal{F}_{t-1}]$. Then for any pair $(p,q) \in [1,\infty)$, there exists some constant $C = C(p,q)>0$ such that,
\begin{align}
    \dist \left(\dfrac{1}{s_T}\sum^T_{t=1}Z_t,Z \right)
    \leq C \left[ \|V^2_T-1 \|^{\frac{p}{2p+1}}_p + \left(\dfrac{1}{s^{2q}_T}\sum^T_{t=1} \| Z_t\|^{2q}_{2q} \right)^{\frac{1}{2q+1}}  \right]
\end{align}
where $Z \sim N(0,1)$. 
\end{lemma}

\noindent Let $Z_t = h_t\varepsilon_t$. Then,
\begin{align*}
    Z^2_t 
    =  [b^{\top} \left(\PT \right)^{-1/2} \z_t]^2 \varepsilon^2_t
    =  \left[b^{\top} \left(\PT \right)^{-1/2} \right] \z_t \z_t^{\top} \left[ \left(\PT \right)^{-1/2} b\right]\varepsilon^2_t
\end{align*}

\noindent Let $\lvec := \left(\PT \right)^{-1/2}b = \left(\PT \right)^{-1}a$. Then we can rewrite,
\begin{align}
    Z^2_t  =  \lvec^{\top} \z_t \z^{\top}_t \lvec \times \varepsilon^2_t
\end{align}
\noindent To apply Lemma~\ref{lemma:mart-clt}, we need to characterize $s_T$. Note that,
\begin{align}\label{eqn:sT-equival}
\notag
    s^2_T 
     = \sum^T_{t=1} \Exs[Z^2_t]
    & = \sum^T_{t=1} \Exs \left[\lvec^{\top} \z_t \z^{\top}_t \lvec \times \varepsilon^2_t \right]\\[8pt] \notag
    & =   \lvec^{\top}  \sum^T_{t=1} \Exs \left[ \Exs \left[ \z_t \z^{\top}_t \times \varepsilon^2_t \mid \mathcal{F}_{t-1},x_t,a_t\right] \right] \lvec \\[8pt] \notag
     & =   \sum^T_{t=1} \Exs \left[\lvec^{\top} \z_t \z^{\top}_t \lvec \ \right]\\[8pt]\notag
     & =   \lvec^{\top} \Exs \left[ \sum^T_{t=1} \z_t \z^{\top}_t  \ \right]\lvec \\[8pt]
     & = \lvec^{\top} \ \PT \ \lvec = \| b\|_2^2
\end{align}

\noindent The last equality holds because $\Exs \left[ \sum^T_{t=1} \z_t \z^{\top}_t  \ \right] = \PT $ by definition and $\lvec =  \left(\PT \right)^{-1/2}b $ is non-random. Therefore, we can apply Lemma~\ref{lemma:mart-clt}  to obtain an upper bound on $\dist(\stato,Z)$. Let us define
\begin{align}\label{defn:TT}
    T_1 := \|V^2_T-1 \|_p \quad \text{and} \quad T_2 = \dfrac{1}{s^{2q}_T}\sum^T_{t=1} \| Z_t\|^{2q}_{2q}
\end{align}
\noindent We shall bound each of $T_1$ and $T_2$ separately.
\subsection*{Upper bound of $T_1$}

Let us recall that $V^2_T := \frac{1}{s^2_T} \sum \Exs[Z^2_t|\mathcal{F}_{t-1}]$. We first calculate $V_T^2-1$ followed up by constructing an upper bound to $\Exs[|V_T^2-1|^p]^{1/p}$. For any $k\in [K]$ let $p_{t,k} := \Exs[w_{t,k}]$ and $\bar{p}_{T,k} := \frac{1}{T} \sum^T_{t=1} p_{t,k}$. Then,
\begin{align}\label{eqn:PT-val}
\notag
    \frac{1}{T} \PT =  
    \frac{1}{T} \sum^T_{t=1} \Exs \left[ \z_t \z^{\top}_t  \right]
    & = \frac{1}{T} \sum^T_{t=1} \Exs \left[ \Exs \left[   \z_t \z^{\top}_t \mid \mathcal{F}_{t-1} \right] \right]\\[8pt]\notag
    & = \frac{1}{T} \sum^T_{t=1} \Exs \left[ \sum^K_{k=1}w_{t,k} \times(\Sigma_k) \right]\\[8pt]
    \notag
    & = \frac{1}{T} \sum^T_{t=1} \sum^K_{k=1}p_{t,k} \times (\Sigma_k) \\[8pt]
    & = \sum^K_{k=1} \bar{p}_{T,k} \times (\Sigma_k)
\end{align}

\noindent The first equality follows from equation~\eqref{defn:condi-Yt}. By combining equations~\eqref{eqn:sT-equival} and~\eqref{eqn:PT-val}, we have the following
\begin{align}\label{eqn:sT-alt}
   s^2_T = T\sum^K_{k=1} \bar{p}_{T,k} \times \left[  \lvec^{\top} \Sigma_k \lvec \right] 
\end{align}

\noindent Following analogous calculations for the conditional variances of $\{Z_t \}$ we also have,
\begin{align*}
    \sum^T_{t=1} \Exs[Z^2_t \mid \mathcal{F}_{t-1}]
    & = \sum^T_{t=1} \Exs \left[\lvec^{\top} \z_t \z^{\top}_t \lvec \times \varepsilon^2_t \mid \fil_{t-1}\right]\\[8pt]
    & =   \lvec^{\top}  \sum^T_{t=1} \Exs \left[ \Exs \left[ \z_t \z^{\top}_t \times \varepsilon^2_t \mid \mathcal{F}_{t-1},x_t,a_t\right] \mid \fil_{t-1} \right] \lvec \\[8pt]
     & =  \sum^T_{t=1} \Exs \left[\lvec^{\top} \z_t \z^{\top}_t \lvec \ \right]\\[8pt]
     & =   \lvec^{\top} \Exs \left[ \sum^T_{t=1} \z_t \z^{\top}_t  \ \mid \fil_{t-1} \right]\lvec
\end{align*}

\noindent The expression in the last equality can be further simplified as :
\begin{align}\label{eqn:vT-alt}
\notag
   \sum^T_{t=1}  \lvec^{\top} \Exs \left[\z_t \z^{\top}_t \mid \mathcal{F}_{t-1} \right] \lvec 
     & = T \lvec^{\top} \left[\frac{1}{T} \sum^T_{t=1} \sum^K_{k=1} w_{t,k} \times (\Sigma_k) \right]\lvec \\[8pt]
    & = T\sum^K_{k=1} \bar{w}_{T,k} \times \left[   \lvec^{\top} \Sigma_k \lvec \right] 
\end{align}

\noindent Therefore, from equations~\eqref{eqn:sT-alt} and~\eqref{eqn:vT-alt} we have:
\begin{align}\label{eqn:VT-alt}
    \|V^2_T-1 \|_p
    & = \bigg| \bigg| \dfrac{\sum^K_{k=1} (\bar{w}_{T,k}-\bar{p}_{T,k}) \times \left[  \lvec^{\top} \Sigma_k \lvec \right] }{\sum^K_{k=1} \bar{p}_{T,k} \times \left[ \lvec^{\top} \Sigma_k \lvec \right] } \bigg| \bigg|_p
\end{align}
\vspace{1pt}

\noindent The high level idea of the proof is to show that $ \| \lvec^{\top} \Sigma_k \lvec \|_p$ is uniformly bounded away from $0$ as well as from above, for each $k \in [K]$. Hence, the rate of convergence of $ \|V^2_T-1 \|_p$ is controlled by $\| \bar{w}_{T}-\bar{p}_{T}\|_p$. Now, recall that $\PT = T \left[\sum^K_{k=1} \bar{p}_{T,k} \times (\Sigma_k) \right]$. Recall that $\lvec = (\PT)^{-1/2} \ b$ and  $b = (\PT)^{-1/2} \ a$. This implies that
\begin{align*}
   \lvec^{\top} \Sigma_k \lvec 
    &= a^{\top} 
    \left(\PT \right)^{-1} \Sigma_k \left(\PT \right)^{-1} a \\[8pt]
    &=\frac{1}{T^2} \times   a^{\top}\left(\frac{1}{T}\PT \right)^{-1} \Sigma_k \left(\frac{1}{T}\PT \right)^{-1} a \\[8pt]
    & = \frac{1}{T^2} \times   a^{\top} \underbrace{\left(\sum^K_{k=1} \bar{p}_{T,k} \times (\Sigma_k) \right)^{-1} \Sigma_k \ \left(\sum^K_{k=1} \bar{p}_{T,k} \times (\Sigma_k) \right)^{-1}}_{\mathcal{H}_{T,k}} a
\end{align*}

\noindent Therefore, we have
\begin{align}\label{eqn:lvec-alt}
    T^2 \ (\lvec^{\top} \Sigma_k \lvec) = a^\top \ \mathcal{H}_{T,k} \ a  
\end{align}

\noindent We recall that for any symmetric matrix $M$ and non-zero real vector $a$, $a^\top M a/\| a\|_2$ lies between $\lambda_{\min}(A)$ and $\lambda_{\max}(A)$ (\cite{rao2000linear}). Combining this observation with equation~\eqref{eqn:lvec-alt} and the fact that all finite dimensional norms are equivalent, we have
\begin{align}\label{eqn:quad Tlvec-bdd}
   \|a\|_2  \ \lambda_{\min}(\mathcal{H}_{T,k}) \lesssim \| T^2 \ (\lvec^{\top} \Sigma_k \lvec)\|_p \lesssim \|a\|_2  \ \lambda_{\max}(\mathcal{H}_{T,k})
\end{align}

\noindent Define $\lambda^\star_L$ and $\lambda^\star_U$ such that 
\begin{align}\label{defn:laml-lamu}
    \min_{k \in [K]}\lambda_{\min}(\Sigma_k) = \lambda^\star_L \quad \text{and} \quad 
    \max_{k \in [K]}\lambda_{\max}(\Sigma_k) = \lambda^\star_U.
\end{align}

\noindent As $\Sigma_k$ are symmetric matrices for all $k\in[K]$,
\begin{align}\label{eqn:eigen-bdd-Sig}
\lambda^\star_L
\leq 
   \lambda_{\min} \left(\sum^K_{k=1} \bar{p}_{T,k} \times (\Sigma_k) \right) 
    \leq\lambda_{\max} \left(\sum^K_{k=1} \bar{p}_{T,k} \times (\Sigma_k) \right) 
    \leq \lambda^\star_U
\end{align}

\noindent Furthermore, for symmetric matrices $A$ and $B$ we have $\lambda_{\max}(AB) = \| AB \|_{op} \leq \|A\|_{op} \|B\|_{op}$ (\cite{rao2000linear}). Therefore, from equation~\eqref{eqn:eigen-bdd-Sig} and definition of $\mathcal{H}_{T,k}$ it follows that,
\begin{align*}
    &\bigg| \bigg| \left(\sum^K_{k=1} \bar{p}_{T,k} \times (\Sigma_k) \right)^{-1} \Sigma_k \ \left(\sum^K_{k=1} \bar{p}_{T,k} \times (\Sigma_k) \right)^{-1} \bigg| \bigg|_{op}\\[8pt]
    &\leq 
     \bigg| \bigg| \left(\sum^K_{k=1} \bar{p}_{T,k} \times (\Sigma_k) \right)^{-1} \bigg| \bigg|_{op} \ \|\Sigma_k \|_{op}\ \bigg| \bigg|\left(\sum^K_{k=1} \bar{p}_{T,k} \times (\Sigma_k) \right)^{-1} \bigg| \bigg|_{op}\\[8pt]
     & \leq \frac{\lambda^\star_U}{(\lambda^\star_L)^2}
\end{align*}

\noindent Furthermore,
\begin{align*}
    &\bigg| \bigg| \left(\sum^K_{k=1} \bar{p}_{T,k} \times (\Sigma_k) \right) \Sigma_k^{-1} \ \left(\sum^K_{k=1} \bar{p}_{T,k} \times (\Sigma_k) \right) \bigg| \bigg|_{op}\\[8pt]
    &\leq 
     \bigg| \bigg| \left(\sum^K_{k=1} \bar{p}_{T,k} \times (\Sigma_k) \right) \bigg| \bigg|_{op} \ \|\Sigma_k^{-1} \|_{op}\ \bigg| \bigg|\left(\sum^K_{k=1} \bar{p}_{T,k} \times (\Sigma_k) \right)\bigg| \bigg|_{op}\\[8pt]
     & \leq \frac{(\lambda^\star_U)^2}{\lambda^\star_L}
\end{align*}

\noindent Note that since for any symmetric invertible matrix $A$, $\lambda_{\min}(A) =1/ \lambda_{\max}(A^{-1})$ (\citet{rao2000linear}) we have
\begin{align}\label{eqn:eigenH-bdd}
  \frac{\lambda^\star_L}{(\lambda^\star_U)^2} \leq   \lambda_{\min}(\mathcal{H}_{T,k}) \leq \lambda_{\max}(\mathcal{H}_{T,k}) \leq \frac{\lambda^\star_U}{(\lambda^\star_L)^2}
\end{align}

\noindent We note that we may divide both numerator and denominator of RHS of equation~\eqref{eqn:VT-alt} by $\|a\|_2$. As a consequence, we assume without loss of generality that $\|a\|_2 = 1$. By combining equations~\eqref{eqn:quad Tlvec-bdd} and~\eqref{eqn:eigenH-bdd} we infer that
\begin{align}\label{eqn:Hquad-bdd}
    \frac{\lambda^\star_L}{(\lambda^\star_U)^2} \leq \| a^\top \ \mathcal{H}_{T,k} \ a \|_p
    \leq \frac{\lambda^\star_U}{(\lambda^\star_L)^2}
\end{align}

\noindent Now, by multiplying $T^2$ in both numerator and denominator of equation~\eqref{eqn:VT-alt} and applying equation~\eqref{eqn:lvec-alt} in equation~\eqref{eqn:vT-alt}, we obtain the following chain of inequalities.
\vspace{5pt}
\begin{align}\label{eqn:VT-bddone}
\notag
    \bigg| \bigg| \dfrac{\sum^K_{k=1} (\bar{w}_{T,k}-\bar{p}_{T,k}) \times \left[  \lvec^{\top} \Sigma_k \lvec \right] }{\sum^K_{k=1} \bar{p}_{T,k} \times \left[ \lvec^{\top} \Sigma_k \lvec \right] } \bigg| \bigg|_p 
    &= \bigg| \bigg| \dfrac{\sum^K_{k=1} (\bar{w}_{T,k}-\bar{p}_{T,k}) \times \left[  a^\top \ \mathcal{H}_{T,k} \ a  \right] }{\sum^K_{k=1} \bar{p}_{T,k} \times \left[ a^\top \ \mathcal{H}_{T,k} \ a  \right] } \bigg| \bigg|_p \\[8pt]\notag
    & \overset{(i)}{\leq} \frac{(\lamu)^2}{\laml} \bigg| \bigg|\sum^K_{k=1} (\bar{w}_{T,k}-\bar{p}_{T,k}) \times \left[ a^\top \ \mathcal{H}_{T,k} \ a  \right] \bigg| \bigg|_{2p}\\[8pt]
    \notag
    & \overset{(ii)}{\leq} \frac{(\lamu)^2}{\laml} \sum^K_{k=1}  \| \bar{w}_{T,k}-\bar{p}_{T,k} \|_{4p} \times \bigg| \bigg| a^\top \ \mathcal{H}_{T,k} \ a   \bigg| \bigg|_{4p} \\[8pt]
    &\overset{(iii)}{\leq} \left( \frac{\lamu}{\laml}\right)^3 \sum^K_{k=1}  \| \bar{w}_{T,k}-\bar{p}_{T,k} \|_{4p} \notag\\[8pt]
    &\overset{(iv)}{\lesssim}  \left( \frac{\lamu}{\laml}\right)^3  \sum^K_{k=1}  \Exs[| \bar{w}_{T,k}-\bar{p}_{T,k} |]  \notag\\[8pt]
    &=  \left( \frac{\lamu}{\laml}\right)^3 \Exs[ \| \bar{w}_{T}-\bar{p}_{T} \|_{1}]
\end{align}

\noindent Inequalities $(i)$ and $(iii)$ in the above chain follows from equation~\eqref{eqn:Hquad-bdd}. Inequality $(ii)$ follows from Holder's inequality. Finally, inequality $(iv)$ is true as $l_p$ and $l_1$ norm are equivalent. Now, by convexity of the $l_p$ norm it follows that $\|\bar{w}_{T}-\bar{p}_{T} \|_1 \leq 2 \| \bar{w}_{T}-w^\star_{T} \|_1$.  Therefore, by suppressing the contribution of the condition number $\lamu/\laml $ it follows from equation~\eqref{eqn:wlone-err} that,
\begin{align}
     \|V^2_T-1 \|_p  \lesssim \Exs[ \|\bar{w}_T - w^\star_T \|_1] \lesssim \Psi(\gamma_T)
\end{align}

\subsection*{Upper bound of $T_2$}

\noindent From equation~\eqref{defn:TT} we recall that $T_2 = \dfrac{1}{s^{2q}_T}\sum^T_{t=1} \| Z_t\|^{2q}_{2q}$, where $Z_t = a^\top (\PT)^{-1}z_t\varepsilon_t$. Furthermore, note that $\lvec = (\PT)^{-1}\ a$ in equation~\eqref{eqn:sT-alt} which states that
\begin{align}\label{eqn:sT-repeat}
    \frac{1}{T}s^2_T = \sum^K_{k=1} \bar{p}_{T,k} \times \left[  \lvec^{\top} \Sigma_k \lvec \right] 
\end{align}

\noindent Therefore, it follows that 
\begin{align*}
    T \ Z_t = a^\top \left( \frac{1}{T}\PT \right)^{-1}z_t \varepsilon_t
    \quad \text{and,} \quad
    T \ \lvec = \left( \frac{1}{T}\PT \right)^{-1} a
\end{align*}

\noindent As we can divide both numerator and denominator of $T_2$ with $1/\|a\|^2_2$, we can assume without loss of generality that $\|a\|_2 = 1$. Hence, by applying the Cauchy-Schwarz inequality we have
\begin{align}
    | T \ Z_t|
    \leq \left| \left| \left( \frac{1}{T}\PT \right)^{-1} \right| \right|_{op} \|z_t\|_2 \|\varepsilon_t\|_2
    \leq \frac{1}{\laml}
\end{align}

\noindent Hence, $\|TZ_t\|^{2q}_{2q} \leq \frac{1}{(\laml)^{2q}}$ for each $t \in [T]$. Note that for any matrix $M$ and vector $v$ we have $v^\top M a \geq \lambda_{\min}(M) \|v\|^2_2$. By substituting $v$ with $a$ and $M$ with $(\PT)^{-2}$, we have $\|Tu\|^2_2 \geq \frac{1}{(\lamu)^2}$. Now, by dividing both numerator and denominator of $T_2$ by $T^{2q}$ we have
\begin{align*}
    T_2 = \dfrac{1}{T^{2q} \cdot s^{2q}_T}\sum^T_{t=1} \| TZ_t\|^{2q}_{2q}
\end{align*}

\noindent Now,
\begin{align}\label{eqn:Tto-numer}
    \frac{1}{T^{q}} \sum^T_{t=1} \| TZ_t\|^{2q}_{2q} \ \leq \ \frac{1}{T^{q-1}} \frac{1}{(\laml)^{2q}}
\end{align}

\noindent The denominator of $T_2$ when multiplied by $T^q$ is equivalent to the equation below by applying  equation~\eqref{eqn:sT-repeat},
\begin{align}\label{eqn:Tto-denom}
    T^q \cdot s^{2q}_T
    = T^{2q} \frac{1}{T^q}\cdot s^{2q}_T
    = \left\{ \sum^K_{k=1} \bar{p}_{T,k} \times \left[  (T\lvec)^{\top} \  \Sigma_k \ (T\lvec) \right]
    \right\}^q
\end{align}

\noindent Therefore, from equations~\eqref{eqn:Tto-numer} and~\eqref{eqn:Tto-denom} we have
\begin{align*}
    \frac{1}{s_T^{2q}}\sum^T_{t=1} \| Z_t\|^{2q}_{2q}
    \ \leq \ \left[ \sum^K_{k=1} \bar{p}_{T,k} \times \left[  (T\lvec)^{\top} \Sigma_k (T\lvec) \right] \right]^{-q} \ \frac{1}{(\laml)^{2q}} \ \frac{1}{T^{q-1}}
\end{align*}

\noindent As $\lvec^{\top} \Sigma_k \lvec$ are uniformly bounded below by $\laml \|u\|^2_2$ for each 
$k \in [K]$, it follows that
\begin{align*}
    \left[ \sum^K_{k=1} \bar{p}_{T,k} \times \left[  \lvec^{\top} \Sigma_k \lvec \right] \right]^{-q} \ \leq \ \left(\frac{1}{\laml \|Tu\|^2_2} \right)^q
    \ \leq \ \left(\frac{(\lamu)^2}{\laml} \right)^q
\end{align*}

\noindent The above  observation leads us to the following conclusion.
\begin{align}\label{eqn:Tto-upp-bdd}
    \left(\dfrac{1}{s^{2q}_T}\sum^T_{t=1} \| Z_t\|^{2q}_{2q} \right)^{\frac{1}{2q+1}} \lesssim  \ \frac{1}{T^{\frac{q-1}{2q+1}} } 
\end{align}

\noindent Therefore, for sufficiently large $q>1$ (free of $T$), the term $T_2$ exhibits polynomial decay. To ensure that our algorithm is near minimax optimal, we choose $\gamma_T = \sqrt{\log T}$. Under this choice, the error term $\err$ is of order $1/\sqrt{\log T}$. Consequently, the overall rate is dominated by $\err$. By fixing $q = 4$ it follows that the upper bound of inequality~\eqref{eqn:Tto-upp-bdd} simplifies to $1/T^{1/3}$.

\subsection*{Proof of Lemma~\ref{lemma:ratto}}

\newcommand{\stath}{\mathcal{S}_3(T)}

\noindent Recall that,
\begin{align*}
    \ratto :=  \frac{a^{\top} \left(\MTP-\Id  \right)}{ \| b\|_2} \sum^T_{t=1} \SigPinv\z_t  \varepsilon_t
\end{align*}
where $\MTP := k_T \mtp$ such that $ k_T = \sqrt{ a^{\top} \SigPinv a / a^{\top} S_T^{-1} a}  
$

\noindent Define the vector $v_T(a) :=  (\frac{1}{T}\PT)^{-1/2} \left(\MTP-\Id  \right)\ a $. Then we can rewrite $\ratto$ as follows
\begin{align*}
    \ratto = \frac{v_T^\top (a)}{  \sqrt{a^\top (\frac{1}{T}\PT)^{-1}  a}} \frac{1}{\sqrt{T}}\sum^T_{t=1} \left(\frac{1}{T}\PT \right)^{-1/2} \z_t  \varepsilon_t \equiv \frac{v_T^\top (a)}{  \sqrt{a^\top (\frac{1}{T}\PT)^{-1}  a}} \cdot \stath
\end{align*}

\noindent where 
\begin{align*}
    \stath := \frac{1}{\sqrt{T}}\sum^T_{t=1} \left(\frac{1}{T}\PT \right)^{-1/2} \z_t  \varepsilon_t
\end{align*}

\noindent We define a new event $\Etwo$ as follows
\begin{align}\label{eqn:Etwo}
    \Etwo := \left\{ \lambda_{\min} \left ( \frac{1}{T}S_T \right) \geq \frac{\laml}{2} \right\}
\end{align}

\noindent On event $\Etwo$, $\ratto$ can be approximated in terms of $\stath$ and $\natto$, where $\natto = \left[\frac{1}{T} \PT \right]^{-1} \times \left[ \frac{M_T}{T} + \frac{B_T}{T}-\frac{1}{T} \PT\right]$. We formalize this notion in Lemma~\ref{lemma:RT-Etwo-approx} stated below.
\begin{lemma}\label{lemma:RT-Etwo-approx}
    On event $\Etwo$, for any vector $a \in \real^d \setminus \{0\}$, there exists a constant $C(d, \laml,\ \lamu)$ depending only on $d, \laml,\ \lamu$  for which we have,
    \begin{align}
       |\ratto | \leq C(d, \laml,\ \lamu) \times \|\natto\|_{op} \times \| \stath \|_2 
    \end{align}
\end{lemma}

\noindent By applying Lemma~\ref{lemma:RT-Etwo-approx} we obtain the following string of inequalities
\begin{align}\label{eqn:Etwo-ctrA}
\notag
    \Prob(\Eone^c) 
    &= \Prob(|\ratto|>\delo) \\[8pt] \notag
    &= \Prob(|\ratto|>\delo,\ \Etwo) \ + \ \Prob(|\ratto|>\delo,\ \Etwo^c)
    \\[8pt] \notag
    & \overset{(i)}{\leq} \Prob \left( C(d, \laml,\ \lamu) \times \|\natto\|_{op} \times \| \stath \|_2 \ >\delo,\ \Etwo \right) \ + \ \Prob \left(|\ratto|>\delo,\ \Etwo^c \right)\\[8pt] \notag
    &  \leq \Prob \left( C(d, \laml,\ \lamu) \times \|\natto\|_{op} \times \| \stath \|_2  \ >\delo \right) \ + \ \Prob( \Etwo^c)\\[8pt] \notag
    & \overset{(ii)}{\leq}  \ \frac{C(d, \laml,\ \lamu)}{\delo} \ \Exs \bigg[  \ \| \natto \|_{op} \  \|\stath \|_2  \bigg] \ + \ \Prob( \Etwo^c) \\[8pt]
     &\overset{(iii)}{\leq}  \frac{C(d, \laml,\ \lamu)}{\delo} \ \sqrt{\Exs \bigg[ \ \| \natto \|^2_{op}  \bigg] \Exs \left[ \|\stath \|^2_2 \right]} \ + \ \Prob( \Etwo^c )
\end{align}

\noindent Inequality $(i)$ holds by applying Lemma~\ref{lemma:RT-Etwo-approx}, whereas $(ii)$ and $(iii)$ follow by applying Markov and the CS inequality, respectively.

\subsubsection*{Bound on $\Exs \left[ \|\stath \|^2_2 \right]$:}
Note that the term $\stath$ is a standardized sum of a martingale difference sequence, and we claim that $\Exs [ \|\stath \|^2_2 ]$ is uniformly bounded above. The justification is the following: 
\begin{align}\label{eqn:stath-ctr}
\notag
   \Exs[\|\stath \|^2_2] 
   & = \Exs \left[ \frac{1}{T} \sum^T_{t=1} \varepsilon_t^2 \z_t^{\top} \left( \frac{1}{T} \PT\right)^{-1} \z_t \right]\\[8pt] \notag
   & \leq \Exs \left[ \frac{1}{T} \sum^T_{t=1} \varepsilon_t^2  \bigg| \bigg|\left( \frac{1}{T} \PT\right)^{-1} \bigg| \bigg|_{op} \|\z_t\|_2^2 \right]\\[8pt]
   &\leq \frac{1}{\laml}
\end{align}

\noindent The first equality holds because the cross product terms vanish as $\stath$ is a sum of martingale difference sequence, and the last inequality above utilizes the fact that $|\varepsilon_t| \leq 1, \|z_t\|_2 \leq 1$ (see Assumptions~\ref{assn:noise} and~\ref{assn:context}) Therefore, by combining equation~\eqref{eqn:Etwo-ctrA}  with equation~\eqref{eqn:stath-ctr} we obtain
\begin{align}
   \Prob(\Eone^c) \leq  \ \frac{C(d, \laml,\ \lamu)}{\delo} \sqrt{\Exs \bigg[ \| \natto \|^2_{op} \bigg]} \ + \ \Prob( \Etwo^c)   
\end{align}

\subsubsection*{Bound on $\Exs \bigg[ \ \| \natto \|^2_{op}  \bigg]$:}
\noindent Recall that $\natto = \left[\frac{1}{T} \PT \right]^{-1} \times \left[ \frac{M_T}{T} + \frac{B_T}{T}-\frac{1}{T} \PT\right]$. Then as $\|\left[\frac{1}{T} \PT \right]^{-1} \|^2_{op} \ \leq \ \frac{1}{(\laml)^2} $ it follows that
\begin{align}
\notag
 \Exs \bigg[ \| \natto \|^2_{op} \bigg] 
 &\leq \frac{1}{(\laml)^2} \cdot \Exs \left[  \left| \left| \frac{M_T}{T} + \frac{B_T}{T}-\frac{1}{T} \PT \right| \right|^2_{op} \right] \\[8pt] \notag
  &\overset{(i)}{\leq} \frac{1}{(\laml)^2} \cdot \Exs \left[   \left| \left| \frac{M_T}{T} \right| \right|^{2}_{op} + \left| \left|\frac{B_T}{T}-\frac{1}{T} \PT \right| \right|^{2}_{op} \ + \ 2 \times \left| \left| \frac{M_T}{T} \right| \right|_{op} \times \left| \left|\frac{B_T}{T}-\frac{1}{T} \PT \right| \right|_{op}    \right] \\[8pt] \notag
  &\overset{(ii)}{\lesssim} \ \frac{1}{(\laml)^2} \cdot  \left[   \frac{1}{T} + (\lamu)^2 \  \err^2 \ + \ 2 \lamu \times \frac{\err}{\sqrt{T}}    \right] \\[8pt]
   &\overset{(iii)}{\lesssim} \ \left(\frac{\lamu}{\laml} \right)^2 \cdot   \err^2  
\end{align}

\noindent Inequality $(i)$ follows from triangle inequality. We note that $\left| \left|\frac{B_T}{T}-\frac{1}{T} \PT \right| \right|_{op} \leq \lamu \|\bar{w}_T - w^\star_T\|_{1}$ (see proof of Lemma~\ref{lemma:MB-inprob} in Appendix~\ref{app:Aux-lemmas} for details). Hence,  Inequality $(ii)$ holds by applying equation~\eqref{eqn:master-ineqthree} along with Lemma~\ref{lemma:wmorm}, while $(iii)$ is a consequence of ignoring lower order terms. Recall that
\begin{align*}
   \err :=  \sqrt{10{\frac{\sqrt{\numactions \log K}}{\gamma_T}} +
     \frac{8\gamma_T \log(K) \log^2\left( KT \right) }{ \numactions^2 T  }  }
\end{align*} Therefore, $\err$ converges to $0$ at a rate much slower than $1/\sqrt{T}$. For example the proposed algorithm is near minimax optimal by choosing $\gamma_T = \sqrt{\log T}$, in which case $\err$ exhibits logarithmic decay.  Due to this reason we can ignore the lower order terms in inequality $(iii)$. Hence,
\begin{align}\label{eqn:RT-ctrB}
    \Prob(\Eone^c) \lesssim  \ \frac{\err}{\delo} \ + \ \Prob( \Etwo^c)   
\end{align}

\noindent By combining equation~\eqref{eqn:RT-ctrB} with Lemma~\ref{lemma:lmin-ST} stated below, we obtain our result.

\begin{lemma}\label{lemma:lmin-ST}
Suppose that the martingale difference sequence $D_t$ are uniformly bounded above by $L$. Then we have the following:
\begin{align}
        \Prob(\Etwo^c) \leq  d \cdot \expower\left\{ - \frac{(\laml)^2}{32 +  8\laml/3} \cdot  T\right\}
    \end{align}
\end{lemma}

\noindent This completes our proof.

\newpage

\section{Statistical Inference with Ridge Estimator}\label{appnd:ridge}

\newcommand{\rao}{\mathcal{V}_1(T)}
\newcommand{\brid}{\hat{\beta}_{rid}}
\newcommand{\rbase}{\sqrt{a^\top R_T^{-1} a}}
\newcommand{\rbaseto}{\sqrt{a^\top [R_T^{-1} \BT ] \left( \frac{1}{T} \BT \right)^{-1}  a}}

\newcommand{\rmatto}{\mathcal{M}_2(T)}
\newcommand{\rratto}{\mathcal{V}_2(T)}

We consider a linear contextual bandit problem with a finite action set $\mathcal{A}$. 
At each round $t = 1, \ldots, T$, the learner observes a context vector $x_t \in \mathcal{X}$, drawn iid from a distribution $\mathcal{P_X}$, selects an action $a_t \in \mathcal{A}$, and receives a random loss
\begin{equation*}
\loss_t \;=\; \langle \beta^\star,\, c(x_t, a_t) \rangle + \varepsilon_t,
\end{equation*}
where $c(x_t, a_t) \in \mathbb{R}^d$ is a known feature representation of the pair $(x_t,a_t)$, 
$\beta^\star \in \mathbb{R}^d$ is an unknown parameter vector. We use $\mathcal{F}_t := \sigma(x_1, a_1, \loss_1, \ldots, x_{t}, a_t, \loss_t)$ to denote the $\sigma$-field generated by observation up to time $t$. We assume that the noise sequence $\{\varepsilon_t\}_{t=1}^T$ satisfies
\begin{equation*}
\mathbb{E}[\varepsilon_t \mid \mathcal{F}_{t-1}, x_t, a_t] = 0.
\end{equation*}

\noindent In this section we show that if we consider a ridge estimator, with penalty term $\rpen $, then Theorem~\ref{thm:clt} holds. Let $R_T := S_T + \rpen \ \Id$ and consider the ridge estimator 
$\brid:= R_T^{-1}S_T \ \beta^\star + \ R_T^{-1} \sum^T_{t=1} \z_t \varepsilon_t$, where $\z_t = c(x_t, a_t)$. Our statistic of interest is the following
\begin{align}\label{defn:rao}
   \rao := \dfrac{a^\top(\brid - \beta^\star)}{\rbase} 
\end{align}

\noindent Theorem~\ref{thm:ridge-clt} below states that for the ridge estimator $\rao$ defined above, both stability and  CLT hold with same rate of convergence as in Theorem~\ref{thm:clt}. We prove Theorem~\ref{thm:ridge-clt} in Section~\ref{sec:ridge-proof} and provide additional simulation details in Section~\ref{sec:ride-sim}.  

\begin{subequations}
\begin{theorem}
\label{thm:ridge-clt}
Suppose that Assumptions~\ref{assn:noise},~\ref{assn:bonded-loss},~\ref{assn:context} and~\ref{assn:Q-bdd} hold. Then the regularized-EXP4 algorithm (\ref{alg:penexp4-finite}) with step size $\eta = \sqrt{\frac{\log K}{ \numactions T}}$, and tuning parameters  $\epsilon = \frac{1}{KT}$, $\lambda = \frac{\gamma_T}{\sqrt{T}}$ is stable, and for any $a \in \real^d$ satisfies 
\begin{align}
    \distt\!\left(
        \rao,\;
        Z
    \right)
    \;\lesssim\;
     \ \Psi(\gamma_T)^{1/3} 
\end{align}
where $Z \sim \mathcal{N}(0,1)$,  $C$ is a constant independent of $T$, and 
\begin{align}
        \err := \sqrt{10{\frac{\sqrt{\numactions \log K}}{\gamma_T}} +
        \frac{8\gamma_T \log(K) \log^2\left( KT \right) }{ \numactions^2 T  }  } 
\end{align}
\end{theorem}
\end{subequations}

\subsection{ Proof of Theorem~\ref{thm:ridge-clt}}\label{sec:ridge-proof}

We split the proof into two parts. We first prove that stability property (Definition~\ref{defn:stability}) is satisfied which leads to asymptotic normality of our $\rao$, followed by the proof for rate of convergence.

\subsubsection*{Part (a) : Proof of stability property of the ridge estimator}
\label{sec:ridgre-stability}

\noindent In this section we show that if we consider a ridge estimator, with penalty term $\rpen \ll \sqrt{T} $, then Theorem~\ref{thm:clt} holds. Let $R_T := S_T + \rpen \ \Id$ and consider the ridge estimator 
$\brid:= R_T^{-1}S_T \ \beta^\star + \ R_T^{-1} \sum^T_{t=1} \z_t \varepsilon_t$, where $\z_t = c(x_t, a_t)$. Our statistic of interest is the following
\begin{align}\label{defn:rao}
   \rao := \dfrac{a^\top(\brid - \beta^\star)}{\rbase} 
\end{align}

\noindent We are interested to show that for the design matrix $R_T = S_T + \rpen \Id$, there exists a sequence of deterministic, positive definite matrices $\{\tilde{\Sigma}_T \}$ such that
\begin{align}\label{eqn:ridge-stab}
    \tilde{\Sigma}^{-1}_T R_T \xrightarrow{\Prob} \Id
 \end{align}

\noindent Now, if we choose our penalty term $\lambda_{rid} \ll T$ then for the ridge estimator we have the decomposition of $R_T$ :
\begin{align}
    \dfrac{R_T}{T} 
     = \underbrace{\dfrac{M_T}{T}}_{o_{\Prob}(1)} + \underbrace{\left[\dfrac{B_T}{T}-\frac{1}{T} \BT \right]}_{o_{\Prob}(1)} + \underbrace{\frac{\rpen}{T}}_{o_\Prob(1)} \Id+ \frac{1}{T} \BT  
\end{align}

\noindent Therefore, applying Lemma~\ref{lemma:MB-inprob} along with the fact that $\lambda_{\min} \left(\frac{1}{T} \BT \right) > 0$ shows that property~\eqref{eqn:ridge-stab} is satisfied by choosing $\tilde{\Sigma}_T = \BT$, where $\BT$ is as defined in equation~\eqref{defn:Sigma-star}.

\noindent By substituting the value of $\brid$ in equation~\eqref{defn:rao}, we obtain the following decomposition
\begin{align}\label{eqn:rao-decom}
   \rao = \underbrace{\dfrac{a^\top (R_T^{-1}S_T - \Id)}{\rbase}  \beta^\star}_{\mathcal{T}_1(T)}  \ \  + \  \ \underbrace{\dfrac{a^\top R_T^{-1}}{\rbase}  \sum^T_{t=1} \z_t \varepsilon_t}_{\mathcal{T}_2(T)}
\end{align}

\noindent The term $\mathcal{T}_1(T)$ in equation~\eqref{eqn:rao-decom} is the bias induced by the regularization. As $R_T := S_T + \rpen \ \Id$ we note that $R_T^{-1}S_T - \Id = R_T^{-1}(S_T - R_T)$, which is equal to $-\rpen \ R_T^{-1}$. Therefore, $\mathcal{T}_1(T)$ can be rewritten as
\begin{align}\label{eqn:taoo-simp}
  \mathcal{T}_1(T) 
  = -\rpen \dfrac{a^\top R_T^{-1}}{\rbase}  \beta^\star
\end{align}

\noindent Let $k_T := \sqrt{a^\top  \left( \frac{1}{T} \BT \right)^{-1}  a} / \sqrt{a^\top  R_T^{-1} a}$ and $\mathcal{M}_T = k_T \times  R_T^{-1}\BT$. Then, 
\begin{align}\label{eqn:taoo-decom}
   \mathcal{T}_1(T)  
   =\frac{-\rpen}{\sqrt{T}} \dfrac{a^\top  \left( \frac{1}{T} \BT \right)^{-1}}{\sqrt{a^\top  \left( \frac{1}{T} \BT \right)^{-1}  a}}  \beta^\star
   + \frac{-\rpen}{\sqrt{T}} \dfrac{a^\top [\mathcal{M}_T - \Id] \left( \frac{1}{T} \BT  \right)^{-1} }{\sqrt{a^\top  \left( \frac{1}{T} \BT \right)^{-1}  a}}  \beta^\star
\end{align}

\noindent Therefore, as $\lambda_{\min}(\frac{1}{T} \BT)$ is uniformly bounded away from zero and $\lambda_{\max}(\frac{1}{T} \BT)$ is uniformly bounded above, it follows from stability (equation~\eqref{eqn:ridge-stab}) that by choosing $\rpen \ll \sqrt{T}$, $\mathcal{M}_T \xrightarrow{\Prob} \Id$ and consequently, the bias $\mathcal{T}_1(T) $ converges to $0$ in probability, as $T \rightarrow \infty$. 

\noindent Now, for $\mathcal{T}_2(T)$ we have a similar decomposition:
\begin{align}\label{eqn:tautwo-decom}
  \mathcal{T}_2(T) =  \underbrace{\dfrac{a^\top \left( \frac{1}{T} \BT \right)^{-1}}{\sqrt{a^\top  \left( \frac{1}{T} \BT \right)^{-1}a}  } \frac{1}{\sqrt{T}} 
  \sum^T_{t=1} \z_t \varepsilon_t }_{I_1(T)}
  \ \ + \ \
  \underbrace{\dfrac{a^\top [\mathcal{M}_T - \Id]\left( \frac{1}{T} \BT \right)^{-1}}{\sqrt{a^\top  \left( \frac{1}{T} \BT \right)^{-1}a}  } \frac{1}{\sqrt{T}} 
  \sum^T_{t=1} \z_t \varepsilon_t}_{I_2(T)}
\end{align}

\noindent From the central limit theorem for sum of martingale difference sequence we have (\citet{dvoretzky1972asymptotic}), $I_1(T) \xrightarrow{d} \mathcal{N}(0,1)$ . This result, combined with the fact that $\mathcal{M}_T \xrightarrow{\Prob} \Id$ implies $I_2(T)$ converges to $0$ in probability. This completes the proof. 

\subsubsection*{Part (b) : Proof of quantitative CLT of the ridge estimator}
\label{sec:ridgre-berry}

The proof of the quantitative CLT for the ridge estimator $\rao$ is similar to that for the OLS estimator $\stao$. In this section we highlight the main steps of the proof. Recall that $\PT := \Exs[S_T]$. From equation~\eqref{eqn:PT-val} and proof of Lemma~\ref{lemma:MB-inprob} it follows that $\frac{B_T}{T}-\frac{1}{T} \PT \xrightarrow{\Prob} 0 $. Choose $\rpen \ll T$ and consider the following decomposition of $R_T$:
\begin{align}
    \dfrac{R_T}{T} 
     = \underbrace{\dfrac{M_T}{T}}_{o_{\Prob}(1)} + \underbrace{\left[\dfrac{B_T}{T}-\frac{1}{T} \PT \right]}_{o_{\Prob}(1)} + \underbrace{\frac{\rpen}{T}}_{o_\Prob(1)} \Id+ \frac{1}{T} \PT  
\end{align}

\noindent As $\lambda_{\min} (\PT) > \laml$, we have
\begin{align}\label{eqn:RT-stabtwo}
    (\PT)^{-1}R_T \xrightarrow{\Prob} \Id
\end{align}

\noindent Let $k_{2,T} := \sqrt{a^\top  \left( \frac{1}{T} \PT \right)^{-1}  a} / \sqrt{a^\top  R_T^{-1} a}$ and define $\mathcal{M}_{2,T} = k_{2,T} \times  R_T^{-1}\PT$. By replacing $\BT$ with $\PT$ in equations~\eqref{eqn:rao-decom},~\eqref{eqn:taoo-simp} and~\eqref{eqn:tautwo-decom} we obtain,
\begin{align}\label{eqn:vto-decom}
  \rao =  \underbrace{\frac{a^\top \left( \frac{1}{T} \PT \right)^{-1}}{\sqrt{a^\top  \left( \frac{1}{T} \PT \right)^{-1}a}  } \frac{1}{\sqrt{T}} 
  \sum^T_{t=1} \z_t \varepsilon_t }_{\stato}     
  +  \underbrace{(-\rpen) \frac{a^\top R_T^{-1}}{\rbase}  \beta^\star}_{ J_1(T)} 
    + 
  \underbrace{\frac{a^\top [\mathcal{M}_{2,T} - \Id]\left( \frac{1}{T} \PT \right)^{-1}}{\sqrt{a^\top  \left( \frac{1}{T} \PT \right)^{-1}a}  } \frac{1}{\sqrt{T}} 
  \sum^T_{t=1} \z_t \varepsilon_t}_{J_2(T)} 
\end{align}

\noindent We note that since $\ratto = \stato \ + \  J_1(T) + J_2(T)$, and from Lemma~\ref{lemma:rate-stao} we have $\dist \left(\stato,Z \right)
        \leq  \ C \left[ \Psi(\gamma_T)^{1/3} \ + \ \frac{1}{T^{1/3}}  \right]$. Therefore if we substitute $\stao$ with $\rao$ and $\ratto$ with $J_1(T) + J_2(T)$ in the proof of Theorem~\ref{thm:clt} in Section~\ref{sec:proof_of_theorems}, all the arguments continue to remain valid, provided the following claim holds:
\begin{subequations}
    \begin{align}
    &\Prob \left(|J_1(T)| > \frac{\delo}{2} \right) \lesssim \frac{\rpen}{\delo \ \sqrt{T}} \label{eqn:jo-approx}\\[8pt]
    & \Prob \left(|J_2(T)| > \frac{\delo}{2} \right) \lesssim \frac{\err}{\delo} \ + \  d \cdot \expower\left\{ - \frac{(\laml)^2}{32 +  8\laml/3} \cdot  T\right\} \label{eqn:jto-approx}
\end{align}
\end{subequations}

\noindent Therefore, assuming these claims are valid from equations~\eqref{eqn:jo-approx} and~\eqref{eqn:jto-approx} we have
\begin{align}\label{eqn:Jo-Jto-bdd}
\notag
    \Prob(| J_1(T) \ + \ J_2(T)| >\delo)
    &\ \leq \  \Prob(| J_1(T)| \ + \ |J_2(T)| >\delo) \\[8pt] \notag
    &\ \leq  \ \Prob \left(|J_1(T)| > \frac{\delo}{2} \right) 
    +  \Prob \left(|J_2(T)| > \frac{\delo}{2} \right)\\[8pt] 
    & \  \lesssim \  \frac{\err}{\delo} \ + \  d \cdot \expower\left\{ - \frac{(\laml)^2}{16 +  4\laml/3} \cdot  T\right\}
    + \frac{\rpen}{\delo \ \sqrt{T}} 
\end{align}

\noindent Therefore, by replacing equation~\eqref{eqn:ratto} of Lemma~\ref{lemma:ratto} in Section~\ref{sec:proof_of_theorems} with equation~\eqref{eqn:Jo-Jto-bdd} we obtain, 
\begin{align*}
    \sup_{x\in \real} |\Prob(\rao \leq x) - \Phi(x) |
    \lesssim \  \beta_T +\delta_{1,T} + \frac{\err}{\delo} \ + \expower\left\{ - \frac{(\laml)^2}{32 +  8\laml/3} \cdot  T\right\}
    + \frac{\rpen}{\delo \ \sqrt{T}} 
\end{align*}

\noindent We choose $\delo$ such that the above upper bound gets minimised.   By applying the AM-GM inequality, we obtain the desired bound by choosing $\delta_T = \sqrt{\alpha_T}$ and recalling that $\beta_T = \err^{1/3}$,
\begin{align*}
    \sup_{x\in \real} |\Prob(\rao \leq x) - \Phi(x) |
    \lesssim \err^{1/3} + 2 \ (\err)^{1/2}
    + \expower\left\{ - \frac{(\laml)^2}{32 +  8\laml/3} \cdot  T\right\}
    + \frac{\rpen}{ \ \sqrt{\err T}} 
\end{align*}

\noindent Therefore, once we justify claims~\eqref{eqn:jo-approx} and~\eqref{eqn:jto-approx} our proof is complete.

\subsubsection*{Proof of claim~\eqref{eqn:jo-approx}}

\noindent Recall from equation~\eqref{eqn:taoo-decom} that
\begin{align*}
   J_1(T)  
   =\frac{-\rpen}{\sqrt{T}} \dfrac{a^\top  \left( \frac{1}{T} \PT \right)^{-1} \ \beta^\star}{\sqrt{a^\top  \left( \frac{1}{T} \PT \right)^{-1}  a}}  
   + \frac{-\rpen}{\sqrt{T}} \dfrac{a^\top [\mathcal{M}_T - \Id] \left( \frac{1}{T} \PT  \right)^{-1} \ \beta^\star }{\sqrt{a^\top  \left( \frac{1}{T} \PT \right)^{-1}  a}}  
\end{align*}

\noindent Note that for any symmetric pd matrix $M$, $u^\top M u \geq \lambda_{\min}(M) \cdot \| u\|^2_2$ and $|u_1^{\top}M u_2| \leq \| u_1\|_2 \  \lambda_{\max}(M) \ \| u_2\|_2$, where $u, u_1, u_2 \ \in \ \real^d$ (\citet{rao2000linear}). Consequently, we observe that
\begin{align*}
    \left|\frac{a^\top  \left( \frac{1}{T} \PT \right)^{-1} \ \beta^\star}{\sqrt{a^\top  \left( \frac{1}{T} \PT \right)^{-1}  a}}  \right| 
    \ \leq \ \frac{\sqrt{\lamu}}{\laml} \| \beta^\star\|_2
    \quad \text{and,} \quad 
    \left|\frac{a^\top [\mathcal{M}_T - \Id] \left( \frac{1}{T} \PT \right)^{-1} \ \beta^\star}{\sqrt{a^\top  \left( \frac{1}{T} \PT \right)^{-1}  a}}  \right| 
    \ \leq \ \frac{\sqrt{\lamu}}{\laml} \| \beta^\star\|_2 \|\mathcal{M}_T - \Id \|_{op}
\end{align*}

\noindent From the above equation, it follows from the definition of $J_1(T)$ that $|J_1(T)| \ \lesssim \ \frac{\rpen}{\sqrt{T}}.$ Consequently, by applying Markov's inequality
\begin{align}\label{eqn:Jo-bdd}
    \Prob \left(|J_1(T)| > \frac{\delo}{2} \right) \lesssim \frac{\rpen}{\delo \ \sqrt{T}} 
\end{align}

\subsubsection*{Proof of claim~\eqref{eqn:jto-approx}}

\noindent From equation~\eqref{eqn:vto-decom} we observe that,
\begin{align*}
   J_2(T) = \frac{a^\top [\mathcal{M}_{2,T} - \Id]\left( \frac{1}{T} \PT \right)^{-1}}{\sqrt{a^\top  \left( \frac{1}{T} \PT \right)^{-1}a}  } \frac{1}{\sqrt{T}} 
  \sum^T_{t=1} \z_t \varepsilon_t
\end{align*}
where  $\mathcal{M}_{2,T} = k_{2,T} \times  R_T^{-1}\PT$ and $k_{2,T} := \sqrt{a^\top  \left( \frac{1}{T} \PT \right)^{-1}  a} / \sqrt{a^\top  R_T^{-1} a}$. As $R_T = S_T + \rpen/\sqrt{T}$, the lower bound on $\lambda_{\min} \left( \frac{1}{T} R_T \right)$ as follows
\begin{align*}
   \lambda_{\min} \left( \frac{1}{T} R_T\right) \geq  \lambda_{\min} \left( \frac{1}{T} S_T\right) + 
\frac{\rpen}{T} \ > \ \lambda_{\min} \left( \frac{1}{T} S_T\right)
\end{align*}

\noindent Hence, on event $\Etwo := \left\{ \lambda_{\min} \left( \frac{1}{T} S_T\right) \ > \ \laml \right\}$, we have $\left\{ \lambda_{\min} \left( \frac{1}{T} R_T\right) \ > \ \laml \right\}$. Therefore, by replacing $S_T$ with $R_T$ in the proof of Lemmas~\ref{lemma:ratto} and~\ref{lemma:RT-Etwo-approx} and applying Lemma~\ref{lemma:lmin-ST} we have:
\begin{align}\label{eqn:Jto-bdd}
       \Prob \left(|J_2(T)| > \frac{\delo}{2} \right) \lesssim \frac{\err}{\delo} \ + \  d \cdot \expower\left\{ - \frac{(\laml)^2}{32 +  8\laml/3} \cdot  T\right\}
    \end{align}

\subsection{Numerical Experiments}\label{sec:ride-sim}

In this section we evaluate the empirical behavior of the proposed \emph{regularized–EXP4} algorithm in a stochastic contextual bandit environment for the modified ridge estimator.

\subsubsection*{Model and Data Generation}

\noindent The loss model follows a block–sparse linear structure
\begin{align*}
    \loss_t = \langle x_t , \theta_{a_t} \rangle + \varepsilon_t,
\end{align*}
where $\varepsilon_t\sim\text{Unif}(-0.1,0.1)$ and $x_t\in\mathbb{R}^{d_x}$ is a normalized Gaussian context vector with $\|x_t\|_2 \le 1$. Each arm $a\in\{1,\dots,A\}$ possesses an unknown parameter $\theta_a\in\mathbb{R}^{d_x}$, and the global coefficient vector
\begin{align*}
    \beta^\star = (\theta_1,\dots,\theta_A) \in \mathbb{R}^{A d_x}
\end{align*}
is normalized to satisfy $\|\beta^\star\|_2 \le 1$. The learner never observes $\theta_a$; instead, only bandit feedback $\loss_t$ is revealed. We encode actions through a sparse block feature map
\begin{align*}
    c(x,a) = \big(0,\dots,x,\dots,0\big) \in \mathbb{R}^{A d_x},
\end{align*}

\paragraph{Extension of the feature map and parameter space.}
To ensure that the losses are non-negative,  we extend the original feature representation by introducing an intercept term. For each context–action pair $(x_t, a_t)$, the original feature vector
\begin{align*}
c(x_t, a_t) \in \mathbb{R}^{A d_x}
\end{align*}
is augmented as
\begin{align*}
\tilde{c}(x_t, a_t)
&=
\begin{pmatrix}
c(x_t, a_t) \\
1
\end{pmatrix}
\in
\mathbb{R}^{A d_x + 1}.
\end{align*}
\noindent Correspondingly, the unknown parameter vector is extended to
\begin{align*}
\tilde{\beta}^\star
&=
\begin{pmatrix}
\beta^\star \\
2
\end{pmatrix}
\in
\mathbb{R}^{A d_x + 1},
\end{align*}
Under this augmented representation, the loss model becomes
\begin{align*}
\loss_t
&=
\langle \tilde{c}(x_t, a_t), \tilde{\beta}^\star \rangle + \varepsilon_t,
\end{align*}
which is algebraically equivalent to the original linear model but explicitly accounts for a constant offset in the losses. In particular, as $|\varepsilon_t| \leq 0.1,$ and $|c(x_t,a_t)| \leq 1$ adding shift of constant $2$ to the observed loss ensures that the resultant losses are non-negative..

\subsection*{Algorithmic Configuration}

We consider the same simulation environment with $6$ layer neural network based experts as considered in Section~\ref{sec:exps}.  The experts now form neural policies with a six-layer ReLU architecture. The expert policy is a six-layer neural network given by
\begin{align*}
   x \;\longrightarrow\; h_1 \;\longrightarrow\; h_2 \;\longrightarrow\; h_3
\;\longrightarrow\; h_4 \;\longrightarrow\; h_5 \;\longrightarrow\; h_6
\;\longrightarrow\; \operatorname{softmax}(\text{logits}), 
\end{align*}
where the hidden layers satisfy
\begin{align*}
   h_i = \operatorname{ReLU}\!\left( W_i h_{i-1} + b_i \right),
\qquad i = 1,\ldots,6, 
\end{align*}
with \(h_0 = x\). The entries of the weight matrices $W_i$ are i.i.d. draws from $\mathcal{N}(0,0.04)$ distribution. 
The resulting expert policy is
\begin{align*}
    \pi(a \mid x) = \frac{\exp(\text{logits}_a)}{\sum_{a'=1}^A \exp(\text{logits}_{a'})}.
\end{align*}

\noindent Let $R_T = S_T + \lambda_{rid} \Id$, where $S_T$ is the sample covariance matrix. Here we consider the ridge estimator
\begin{align*}
  \widehat{\beta}_{rid} := R_T^{-1} \sum^T_{t=1}c(x_t,a_t)l_t  
\end{align*}

\noindent We draw a random unit direction $a\in\mathbb{R}^{A d_x}$ and for each confidence level $\alpha \in [0.20, 0.01]$, check whether the true parameter lies inside the interval. Concretely we check if the target parameter $a^\top\beta^\star$ lies in the interval 
\begin{align*}
\mathcal{I}^{\mathrm{APS}}_T(a) :=  \big[a^\top\widehat\beta_{\mathrm{ridge}}-\xi_T\sqrt{a^\top \VT^{-1}a},\;
a^\top\widehat\beta_{\mathrm{ridge}}+\xi_T\sqrt{a^\top \VT^{-1}a}\big],
\end{align*}
where $\xi_T$ is
\begin{align}
    \xi_T :=  \sqrt{
2\!\left(
\frac{1}{2}\log\frac{\det(V_t)}{\det(\lambda I)}
\;+\;
\log\frac{1}{\alpha}
\right)
}
\;+\;
\sqrt{\lambda}\,\|\beta^\star\|_2
\end{align}
We note that the confidence intervals $\mathcal{I}^{\mathrm{APS}}_T(a)$ defined above are sharper than the anytime valid confidence interval defined in equation~\eqref{eq:APS-interval}. Wald coverage is measured analogously using the confidence
 interval $ \mathcal{I}^{\mathrm{Wald}}_T(a)$ similar to~\eqref{eq:wald-interval},
 \begin{align} 
    \mathcal{I}^{\mathrm{Wald}}_T(a)
    \;:=\;
    \Bigl[
        a^\top \widehat{\beta}_{rid}
        \;\pm\;
        z_{1-\alpha/2}\,
        \widehat\sigma\, \sqrt{a^\top R_T^{-1} a}
    \Bigr],
\end{align} with $\widehat{\sigma}$ as the sample standard deviation estimate~\cite[Lemma 3]{lai1982least}. For each $T \in \{500, 3000 \}$, we report empirical coverage of $\mathcal{I}^{\mathrm{APS}}(a)$ vs $\mathcal{I}^{\mathrm{Wald}}(a)$, and their average width.

\noindent In our experiments we set $A  = 3$ and $K = 5$ and $d_x = 50$. At the end of horizon $T$, we compute the ridge estimator
\begin{align*}
    \widehat\beta_{rid} = \big(S_T + \lambda_{rid}I\big)^{-1} b_T, \qquad
    S_T = \sum_{t=1}^T c(x_t,a_t) c(x_t,a_t)^\top,\quad
    b_T = \sum_{t=1}^T c(x_t,a_t) \loss_t.
\end{align*}

\noindent Hyperparameters are selected as
\begin{align*}
\varepsilon = \frac{1}{KT},\qquad
\lambda_{\rm pen}= \frac{\sqrt{\log T}}{\sqrt{T}},\qquad
\eta = \sqrt{\frac{\log K}{ \numactions T}},
\qquad \lambda_{rid} = \frac{1}{T}
\end{align*}
unless stated otherwise.  For each $T$, we report:

\begin{itemize}
    \item Empirical coverage  $\mathcal{I}^{\mathrm{APS}}$ vs $\mathcal{I}^{\mathrm{Wald}}$,
    \item Average width $\mathcal{I}^{\mathrm{APS}}$ vs $\mathcal{I}^{\mathrm{Wald}}$,
    \item Dependence on nominal confidence level $\alpha$.
\end{itemize}

\subsection*{Simulation Plots}

 \begin{figure}[H]\label{fig-ucb}
    \centering

    \begin{subfigure}[t]{0.42\textwidth}
        \includegraphics[width=\textwidth,trim={1cm 0.5cm 1cm 0.5cm}]{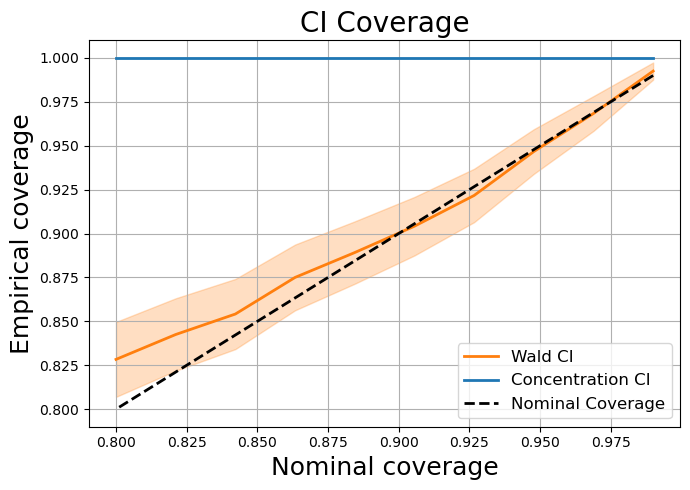}
    \end{subfigure}
    \hfill
    \begin{subfigure}[t]{0.42\textwidth}
        \includegraphics[width=\textwidth,trim={1cm 0.5cm 1cm 0.5cm}]{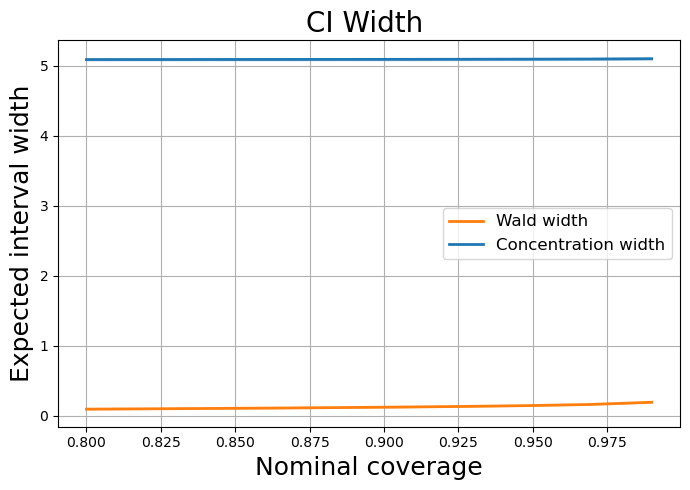}
    \end{subfigure}
    \\[4pt]
    \caption{\textbf{Left}: Coverage of both $\mathcal{I}^{\mathrm{APS}}$ and $\mathcal{I}^{\mathrm{Wald}}$.  \textbf{Right}: Expected confidence width of  both $\mathcal{I}^{\mathrm{APS}}$ and $\mathcal{I}^{\mathrm{Wald}}$. The average CI widths of $\mathcal{I}^{\mathrm{Wald}}$ and $\mathcal{I}^{\mathrm{APS}}$ across all values of $\alpha$ are $0.12$ and $5.09$ respectively.
     Simulations are based on \textbf{$T=3000$} runs.} 
\end{figure}

 \begin{figure}[H]\label{fig-ucb}
    \centering

    \begin{subfigure}[t]{0.42\textwidth}
        \includegraphics[width=\textwidth,trim={1cm 0.5cm 1cm 0.5cm}]{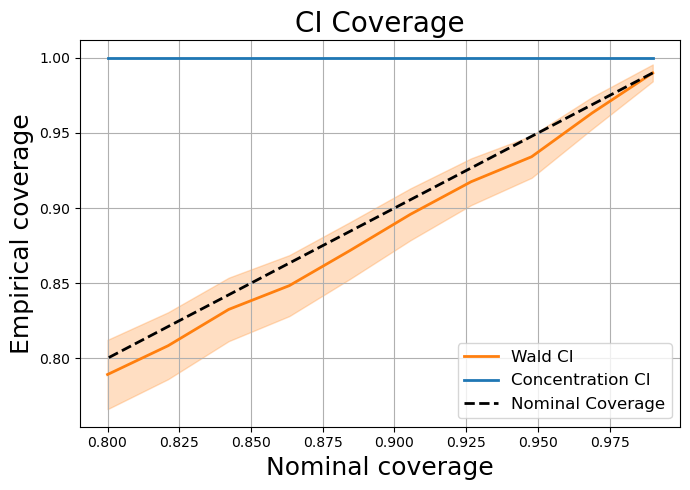}
    \end{subfigure}
    \hfill
    \begin{subfigure}[t]{0.42\textwidth}
        \includegraphics[width=\textwidth,trim={1cm 0.5cm 1cm 0.5cm}]{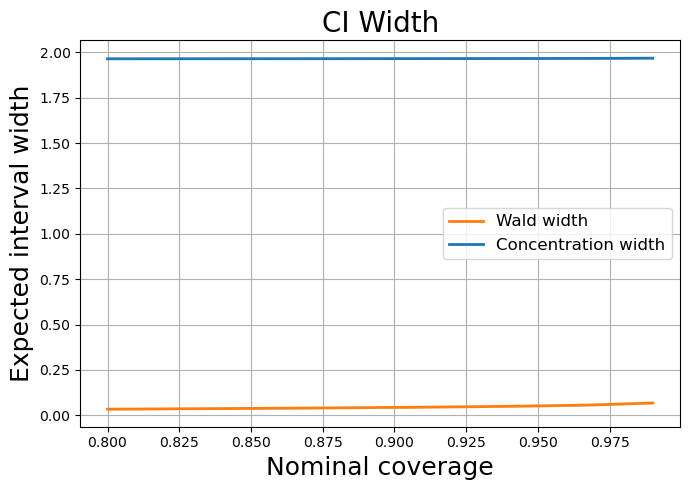}
    \end{subfigure}
    \\[4pt]
    \caption{\textbf{Left}: Coverage of both $\mathcal{I}^{\mathrm{APS}}$ and $\mathcal{I}^{\mathrm{Wald}}$.  \textbf{Right}: Expected confidence width of  both $\mathcal{I}^{\mathrm{APS}}$ and $\mathcal{I}^{\mathrm{Wald}}$. The average CI widths of $\mathcal{I}^{\mathrm{Wald}}$ and $\mathcal{I}^{\mathrm{APS}}$ across all values of $\alpha$ are $0.04$ and $1.96$ respectively.
     Simulations are based on \textbf{$T=3000$} runs.} 
\end{figure}

 \begin{figure}[H]\label{fig-ucb}
    \centering

    \begin{subfigure}[t]{0.42\textwidth}
        \includegraphics[width=\textwidth,trim={1cm 0.5cm 1cm 0.5cm}]{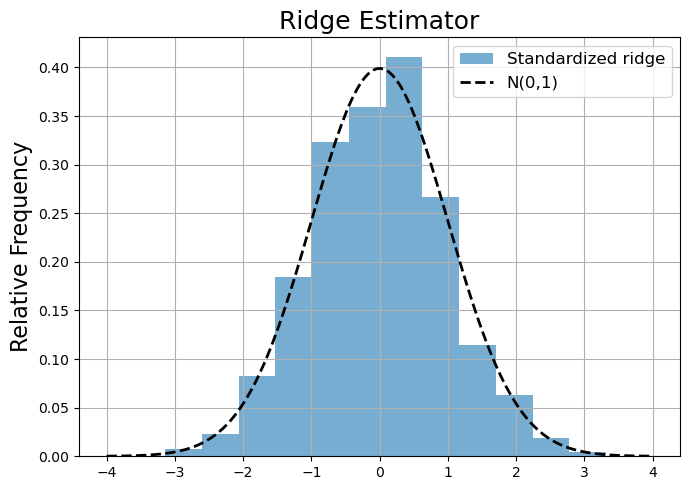}
        \caption*{$\frac{a^{\top}  \left( \hat{\beta}_{rid} - \beta^\star \right)}{ \sqrt{ a^{\top} S_T^{-1} a}}$}
    \end{subfigure}
    \hfill
    \begin{subfigure}[t]{0.42\textwidth}
        \includegraphics[width=\textwidth,trim={1cm 0.5cm 1cm 0.5cm}]{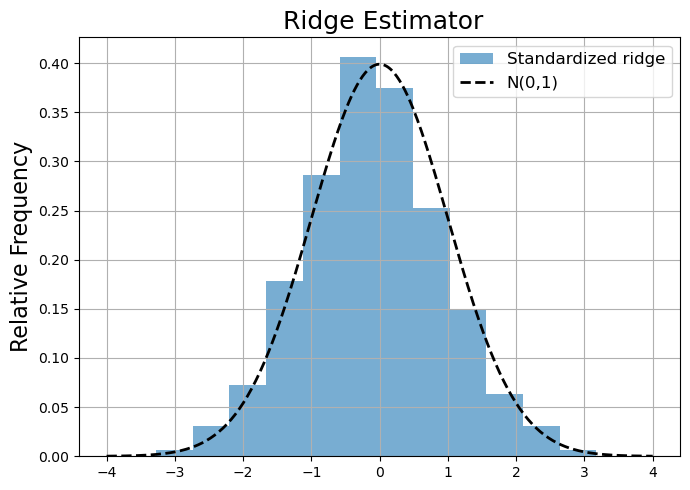}
        \caption*{$\frac{a^{\top}  \left( \hat{\beta}_{rid} - \beta^\star \right)}{ \sqrt{ a^{\top} S_T^{-1} a}}$}
    \end{subfigure}
    \\[4pt]
    \caption{\textbf{Left}: Histogram of the standardized Ridge estimator for $T=500$.  \textbf{Right}: Histogram of the standardized Ridge estimator for $T=3000$.} 
\end{figure}

\newpage

\section{Auxiliary Lemmas}
\label{app:Aux-lemmas}

In this section we present some well-known results on Bregman divergence for the sake of completeness, followed by the proofs of Lemma~\ref{lemma:MB-inprob} and~\ref{lemma:norm-ghat}.  

\subsection*{Some Results on Bregman Divergence}

We begin with two standard identities for Bregman divergences that will be used repeatedly.
Let $\phi:\mathcal{D}\to\mathbb{R}$ be a differentiable, strictly convex function on an open convex set $\mathcal{D}\subset\mathbb{R}^K$.
The Bregman divergence generated by $\phi$ is
\[
D_\phi(x,y) \;=\; \phi(x)-\phi(y)-\langle \nabla\phi(y),\,x-y\rangle,
\qquad x,y\in\mathcal{D}.
\]

\begin{lemma}[Bregman three-point identity]
\label{lem:three-point-alt}
For any differentiable and strictly convex function $\phi:\mathcal{D}\to\mathbb{R}$ 
and any $x, x^+, y \in \mathcal{D}$,
\[
\big\langle \nabla\phi(x)-\nabla\phi(x^+),\, x - y \big\rangle
\;=\;
D_\phi(y,x)\;-\;D_\phi(y,x^+)\;+\;D_\phi(x,x^+),
\]
where $D_\phi(u,v)=\phi(u)-\phi(v)-\langle\nabla\phi(v),\,u-v\rangle$ 
is the Bregman divergence generated by $\phi$.
\end{lemma}

\begin{proof}
By the definition of the Bregman divergence,
\begin{align*}
&D_\phi(y,x)-D_\phi(y,x^+)+D_\phi(x,x^+) \\
&= \big(\phi(y)-\phi(x)-\langle\nabla\phi(x),y-x\rangle\big)
   -\big(\phi(y)-\phi(x^+)-\langle\nabla\phi(x^+),y-x^+\rangle\big) \\
&\quad\;\;+\big(\phi(x)-\phi(x^+)-\langle\nabla\phi(x^+),x-x^+\rangle\big).
\end{align*}
The $\phi(\cdot)$ terms cancel, giving
\[
-\langle\nabla\phi(x),y-x\rangle
+\langle\nabla\phi(x^+),y-x^+\rangle
-\langle\nabla\phi(x^+),x-x^+\rangle.
\]
Since $y-x^+=(y-x)+(x-x^+)$, the last two inner products combine to
$\langle\nabla\phi(x^+),y-x\rangle$.
Thus the entire expression simplifies to
\[
\langle\nabla\phi(x^+)-\nabla\phi(x),\,y-x\rangle
\;=\;
\langle\nabla\phi(x)-\nabla\phi(x^+),\,x-y\rangle,
\]
which proves the desired identity.
\end{proof}

\begin{definition}[Bregman projection]
Let $X\subset\mathcal{D}$ be closed and convex.
For $y\in\mathcal{D}$, the (right) Bregman projection of $y$ onto $X$ is
\[
\Pi_X^\phi(y) \;\in\; \arg\min_{x\in X} D_\phi(x,y).
\]
\end{definition}

\begin{lemma}[Bregman Pythagorean theorem]
\label{lem:pythagoras}
Let $X\subset\mathcal{D}$ be nonempty, closed, and convex, and let $x^+=\Pi_X^\phi(y)$.
Then for all $x\in X$,
\[
D_\phi(x,y) \;\ge\; D_\phi(x,x^+) \;+\; D_\phi(x^+,y).
\]
\end{lemma}

\begin{proof}
By first-order optimality for the convex problem $\min_{x\in X} D_\phi(x,y)$,
\[
\big\langle \nabla_x D_\phi(x,y)\big|_{x=x^+},\, x-x^+ \big\rangle \;\ge\; 0
\quad\text{for all } x\in X.
\]
Since $\nabla_x D_\phi(x,y)=\nabla\phi(x)-\nabla\phi(y)$, we have
\[
\big\langle \nabla\phi(x^+)-\nabla\phi(y),\, x-x^+ \big\rangle \;\ge\; 0.
\]
Apply Lemma~\ref{lem:three-point-alt} with $(x,y,z)=(x^+,y,x)$ to obtain
\[
\big\langle \nabla\phi(x^+)-\nabla\phi(y),\, x-x^+ \big\rangle
= D_\phi(x,x^+) - D_\phi(x,y) + D_\phi(x^+,y).
\]
Rearranging yields the stated inequality.
\end{proof}

\subsection*{Proof of Lemma~\ref{lemma:MB-inprob}}
\noindent 
By construction, $(M_T,\mathcal{F}_T)$ is a mean-zero
\emph{matrix-valued martingale}:
\begin{align*}
  \mathbb{E}[D_t\mid \mathcal{F}_{t-1}] = 0,
\qquad
\mathbb{E}[M_T\mid \mathcal{F}_{t-1}] = M_{t-1}.  
\end{align*}

\noindent Because $\|\cfeat(x,a)\|$ is uniformly bounded,
\begin{align*}
    \|Y_t\|_F^2 \le d^2\,\sup_{x,a}\|\cfeat(x,a)\|^4 < \infty.
\end{align*}
Hence
$\mathbb{E}[\|D_t\|_F^2 \mid \mathcal{F}_{t-1}] \le C$ for some constant $C<\infty$,
and therefore
\begin{align*}
  \mathbb{E}\|M_T\|_F^2 = \sum_{t=1}^T \mathbb{E}\|D_t\|_F^2 \le CT.  
\end{align*}

\noindent By orthogonality of martingale differences,
\begin{equation*}
  \mathbb{E}\|M_T\|_F^2
  = \sum_{t=1}^T \mathbb{E}\|D_t\|_F^2
  = \sum_{t=1}^T \mathbb{E}\big[\mathbb{E}[\|D_t\|_F^2 \mid \mathcal{F}_{t-1}]\big]
  \le CT.
\end{equation*}
\noindent Thus
\begin{equation}\label{eqn:MT-bdd}
  \mathbb{E}\left\|\frac{M_T}{T}\right\|_F^2
  = \frac{1}{T^2} \mathbb{E}\|M_T\|_F^2
  \le \frac{C}{T} \;\to\; 0,
\end{equation}
which shows $M_T/T \to 0$ in $L^2$ and hence in probability. Note that in the above equation we have considered the Frobenius norm. However, as all finite dimensional norms are equivalent, $M_T/T \xrightarrow{\Prob} 0$ with respect to the operator norm as well.
From our assumption it follows that there exists non-random weight vector $w^\star_T$ such that
\begin{align}
    \frac{1}{T}\sum_{t = 1}^T w_{t,k}-\wstar_{T,k}  \inprob 0,
\end{align}
for all experts $i \in [K]$. Now we observe that,
\begin{align*}
    \frac{B_T}{T} = \frac{1}{T}\sum_{t=1}^T \mathbb{E}[Y_t\mid \mathcal{F}_{t-1}] = \sum^K_{k=1} \left(\frac{1}{T}\sum_{t=1}^Tw_{t,k} \right) \Sigma_k
\end{align*}

\noindent As all the elements of $\Sigma_k$ are \emph{uniformly} bounded, we have the following string of inequalities,
\begin{align}\label{eqn:BT-bdd}
\notag
    \left\|\frac{B_T}{T} -\frac{1}{T} \BT  \right\|_{op}
    &=  \left\| \sum^K_{k=1} \left(\frac{1}{T}\sum_{t=1}^Tw_{t,k} - \wstar_{T,k}\right) \Sigma_k \right\|_{op} \\[8pt] \notag
    & \leq \sum^K_{k=1}  \bigg| \frac{1}{T}\sum_{t=1}^Tw_{t,k} - \wstar_{T,k}  \bigg| \left\| \Sigma_k  \right\|_{op} \\[8pt]
    & \leq \lamu  \sum^K_{k=1} \bigg|\frac{1}{T}\sum_{t=1}^Tw_{t,k} - \wstar_{T,k} \bigg| 
     = \lamu \  \| \bar{w}_{T} - w^\star_T \|_1
\end{align}

\noindent Therefore, our proof is complete.

\subsection*{Proof of Lemma~\ref{lemma:norm-ghat}}
Recall that 
\begin{align*}
  \widehat g_{t,k}
  = \,\loss_t\,\frac{\pi_k(a_t\mid x_t)}{Q_t(a_t\mid x_t)},
  \qquad
  Q_t(a\mid x_t) = \sum_{j=1}^K w_{t,j}\,\pi_j(a\mid x_t),
\end{align*}
and assume $|\loss_t|\le 1$ almost surely.
By definition of the local dual norm,
\begin{align*}
  \|\widehat g_t\|_{w_t,*}^2
  = \sum_{k=1}^K w_{t,k}\widehat g_{t,k}^2
  = \loss_t^2 \sum_{k=1}^K
      w_{t,k}\,\frac{\pi_k^2(a_t\mid x_t)}{Q_t^2(a_t\mid x_t)}.
\end{align*}
Conditioning on $x_t$ and summing over $a_t\sim Q_t(\cdot\mid x_t)$ gives
\begin{align*}
  \Exs\big[\|\widehat g_t\|_{w_t,*}^2 \mid x_t\big]
  &= \sum_{a\in\mathcal{A}}
      Q_t(a\mid x_t)\,
      \Big\{
        \loss_t^2 \sum_{k=1}^K
        w_{t,k}\,\frac{\pi_k^2(a\mid x_t)}{Q_t^2(a\mid x_t)}
      \Big\} \\
  &= \sum_{a\in\mathcal{A}}
      \loss_t^2
      \sum_{k=1}^K
      w_{t,k}\,\frac{\pi_k^2(a\mid x_t)}{Q_t(a\mid x_t)} \\
  &\le \sum_{a\in\mathcal{A}}
      \loss_t^2
      \sum_{k=1}^K
      w_{t,k}\,\frac{\pi_k(a\mid x_t)\,\max_j \pi_j(a\mid x_t)}{Q_t(a\mid x_t)} \\
  &= \sum_{a\in\mathcal{A}}
      \loss_t^2\,\max_j \pi_j(a\mid x_t)
      \frac{\sum_{k=1}^K w_{t,k}\pi_k(a\mid x_t)}{Q_t(a\mid x_t)} \\
  &= \sum_{a\in\mathcal{A}}
      \loss_t^2\,\max_j \pi_j(a\mid x_t)
      \;\le\; |\mathcal{A}|\,\max_{x,a} \loss_t^2
      \;\le\; |\mathcal{A}|,
\end{align*}
where we used \(Q_t(a\mid x_t)=\sum_k w_{t,k}\pi_k(a\mid x_t)\) in the third line,
and $\max_{x,a} \loss_t^2\le 1$ in the last inequality. This completes the proof.

\subsection*{ Proof of Lemma~\ref{lemma:RT-Etwo-approx}}

\noindent Let us recall that,
\begin{align}\label{eqn:rat-nat}
    \ratto = \frac{v_T^\top (a) \ \stath}{  \sqrt{a^\top (\frac{1}{T}\PT)^{-1}  a}}
    \quad \text{and,}
    \quad
    \natto = \left[\frac{1}{T} \PT \right]^{-1} \times \left[ \frac{M_T}{T} + \frac{B_T}{T}-\frac{1}{T} \PT\right]
\end{align}

\noindent where vector $v_T(a) :=  (\frac{1}{T}\PT)^{-1/2} \left(k_T S_T^{-1} \PT-\Id  \right)\ a $. Note that for any vector $u$, $u^\top (\frac{1}{T}\PT)^{-1} u \ \geq \ \|u\|^2_2 \ \lambda_{\max} \left( (\frac{1}{T}\PT)^{-1} \right)$ and $\lambda_{\min} \left( (\frac{1}{T}\PT)^{-1} \right) = \frac{1}{\lambda_{\max} \left( \frac{1}{T}\PT \right)}$. Therefore,
\begin{align}\label{eqn:modrat-one}
\notag
     | \ratto | 
    = \frac{|v_T^\top (a) \ \stath|}{  \sqrt{a^\top (\frac{1}{T}\PT)^{-1}  a}} 
    \  &\leq \  |v_T^\top (a) \ \stath| \times \frac{\sqrt{\lambda_{\max} \left( \frac{1}{T} \BT \right)}} {\|a\|_2} \\[8pt]\notag
    & \leq |v_T^\top (a) \ \stath| \times \frac{\sqrt{\lamu}} {\|a\|_2}\\[8pt]
    &\leq \| v_T^\top (a)\|_2 \times \| \stath\|_{op} \times \frac{\sqrt{\lamu}} {\|a\|_2}
\end{align}

\noindent As $v_T(a) :=  (\frac{1}{T}\PT)^{-1/2} \left(k_T S_T^{-1} \PT-\Id  \right)\ a $, by repeated application of the CS inequality it follows that,
\begin{align}\label{eqn:vt-uppbdd}
   \|v_T(a)\|_2 \ \leq \ \|a\|_2 \ \left| \left|\left( \frac{1}{T} \PT \right)^{-1/2} \right| \right|_{op}
    \ \| k_T S_T^{-1} \PT - \Id  \|_{op} 
\end{align}


\noindent If $M$ is any symmetric, positive definite matrix with eigen values $\lambda_1,\ldots,\lambda_b$ then the eigen values of $M^{1/2}$ are $\sqrt{\lambda_1},\ldots,\sqrt{\lambda_b}$ (\citet{rao2000linear}). Hence,
\begin{align*}
    \left| \left|\left( \frac{1}{T} \PT \right)^{-1/2} \right| \right|_{op} \leq \frac{1}{\sqrt{\laml}}
\end{align*}

\noindent Consequently, from equation~\eqref{eqn:vt-uppbdd} we have
\begin{align}\label{eqn:rt-bddtwo}
    | \ratto | \leq  \sqrt{\frac{\lamu}{\laml}} \times
    \ \| k_T S_T^{-1} \PT - \Id  \|_{op}  \times \| \stath\|_{op} \
\end{align}

\newcommand{\natoto}{\mathcal{N}_2(T)}

\noindent Now, by an application of the triangle and CS inequality we have
\begin{align}\label{eqn:MTP-bdd}
\notag
    \| k_T S_T^{-1} \PT - \Id  \|_{op} 
    &\leq | k_T - 1| \| S_T^{-1} \PT\|_{op} \ + \ \| S_T^{-1} \PT - \Id \|_{op} \\[8pt]
    & \leq | k_T - 1| \ \frac{\lamu}{\lambda_{\min} \left ( \frac{1}{T}S_T \right)} 
    \ + \ \| S_T^{-1} \PT - \Id \|_{op}
\end{align}

\noindent Now, recall that  $k_T = \sqrt{ a^{\top} \SigPinv a / a^{\top} S_T^{-1} a}$. Therefore,
\begin{align}\label{eqn:kt-decomo}
    \left(\frac{1}{T} \PT \right)^{-1}
    = \ \left(\frac{1}{T} S_T \right)^{-1}
    + \left[\left(\frac{1}{T} \PT \right)^{-1}
    \left(\frac{1}{T} S_T \right) - \Id \right]
    \left(\frac{1}{T} S_T \right)^{-1}
\end{align}

\noindent Let us define $ \natoto:= \left[\left(\frac{1}{T} \PT \right)^{-1}
    \left(\frac{1}{T} S_T \right) - \Id \right]
    \left(\frac{1}{T} S_T \right)^{-1}$. Then due to sub-multiplicity of the operator norm it follows that
\begin{align}\label{eqn:kt-decomto}
    \|\natoto \|_{op}
    \leq \left| \left| \left(\frac{1}{T} S_T \right)^{-1}\right| \right|_{op} \times \ \|(\PT)^{-1} S_T - \Id \|_{op}
\end{align}

\noindent  By applying equations~\eqref{eqn:kt-decomo} and~\eqref{eqn:kt-decomto} we bound $ |k_T -1|$ below:
\begin{align}\label{eqn:kt-MT-pbdd}
\notag
     | k_T - 1 |   
    &=  \left| \sqrt{1 + \dfrac{a^{\top}\natoto a}{a^{\top} \left( \frac{1}{T} S_T \right)^{-1}a}} -1 \right| \\[8pt]\notag
    & =  \left| \dfrac{1}{ \sqrt{1 + \frac{a^{\top} \natoto a}{a^{\top} \left( \frac{1}{T} S_T \right)^{-1}a}} +1} \times   \dfrac{a^{\top} \natoto a}{a^{\top} \left( \frac{1}{T} S_T \right)^{-1}a} \right|  \\[10pt]\notag
    &\leq  \left|   \frac{a^{\top} \natoto a}{a^{\top} \left( \frac{1}{T} S_T \right)^{-1}a} \right|  \\[10pt]\notag
    & \leq \|a\|^2_2 \times \| \natoto\|_{op} \times \frac{\lambda_{\max} \left( \frac{1}{T} S_T \right)}{\| a\|^2_2} \\[8pt] \notag
    & \leq  \ \frac{\lambda_{\max} \left( \frac{1}{T} S_T \right)}{\lambda_{\min} \left( \frac{1}{T} S_T \right)} \times  \| (\PT)^{-1}S_T - \Id)\|_{op}
    \\[8pt]
    & \lesssim \frac{1}{\lambda_{\min} \left( \frac{1}{T} S_T \right)} \| \natto\|_{op}
\end{align}

\noindent The last inequality follows directly from the definition of $\natto$ (see equation~\eqref{eqn:rat-nat}). We note that for any invertible matrix $B$ such that $B = \Id + A$, we have $B^{-1} = \Id - B^{-1}A $. By rearranging the terms, it follows that $\Id -B^{-1} = B^{-1}(B-\Id) $. By substituting $B = \PT^{-1}S_T$ and applying the sub-multiplicative property of the operator norm we obtain 
\begin{align}\label{eqn:st-Inv-decom}
     \| S_T^{-1} \PT - \Id\|_{op} 
     \leq  \| S_T^{-1} \PT \|_{op} \|(\PT)^{-1} S_T - \Id\|_{op} 
     = \| S_T^{-1} \PT \|_{op} \times \| \natto\|_{op}
\end{align}

\noindent By combining equations~\eqref{eqn:MTP-bdd}, ~\eqref{eqn:kt-MT-pbdd} and~\eqref{eqn:st-Inv-decom} we have the following chain of inequalities:
\begin{align*}
    \| k_T S_T^{-1} \BT - \Id  \|_{op} 
    & \lesssim  \ \frac{\lamu}{\lambda_{\min} \left ( \frac{1}{T}S_T \right)} \times \frac{1}{\lambda_{\min} \left( \frac{1}{T} S_T \right)} \| \natto\|_{op}
    \ + \ \| S_T^{-1} \PT \|_{op} \times \| \natto\|_{op} \\[8pt]
    & \leq  \ \left[ \frac{\lamu}{\lambda_{\min} \left ( \frac{1}{T}S_T \right)^2} \ + \ \frac{\lamu}{\lambda_{\min} \left ( \frac{1}{T}S_T \right)} \right] \times \| \natto\|_{op} 
\end{align*}

\noindent Therefore, on event  $\Etwo$ we have
\begin{align}\label{eqn:mtp-upbddtwo}
    \| k_T S_T^{-1} \PT - \Id  \|_{op} 
    \leq \ \left[ \frac{2\lamu}{(\laml)^2} \ + \ \frac{2\lamu}{\laml} \right] \times  \|\natto \|_{op}
\end{align}

\noindent Hence, by combining inequalities~\eqref{eqn:rt-bddtwo} and~\eqref{eqn:mtp-upbddtwo}, it follows that there exists a constant $C(d,\ \laml,\ \lamu)$ such that
\begin{align*}
    |\ratto | \leq C(d,\laml,\ \lamu) \times \|\natto\|_{op} \times \| \stath \|_2
\end{align*}

\subsection*{Proof of Lemma~\ref{lemma:lmin-ST}}

\noindent Note that
\begin{align*}
    \Prob(\Etwo) = \Prob\left( \lambda_{\min} \left ( \frac{1}{T}S_T \right) \geq \frac{\laml}{2} \right)
\end{align*}

\noindent Recall that we have the following decomposition
\begin{align*}
    \lambda_{\min} \left ( \frac{1}{T}S_T \right)
     =  \lambda_{\min} \left ( \frac{1}{T}M_T \ + \ \frac{1}{T}B_T \right)
     \geq \ \lambda_{\min} \left( \frac{1}{T}M_T  \right) \ + \ \lambda_{\min} \left(\frac{1}{T}B_T \right)
\end{align*}

\noindent As $\frac{1}{T} B_T$ is a convex combination of the $\Sigma_k$'s for $k \in [K]$, it follows that
\begin{align*}
    \lambda_{\min} \left(\frac{1}{T}B_T \right) \ \geq \ \laml
\end{align*}

\noindent Therefore, 
\begin{align*}
   \lambda_{\min} \left ( \frac{1}{T}S_T \right) 
   \geq   \ \lambda_{\min} \left( \frac{1}{T}M_T  \right) \ + \ \laml
\end{align*}

\noindent This implies that
\begin{align*}
   \Prob(\Etwo^c) = \Prob \left(\lambda_{\min} \left ( \frac{1}{T}S_T \right) < \frac{\laml}{2} \right)
   \ \leq \ \Prob \left(\lambda_{\min} \left ( \frac{1}{T}M_T \right) < -\frac{\laml}{2} \right)
\end{align*}

\noindent By observing that $M^\star_T :=-M_T$ is also a sum of martingale difference sequences, we have
\begin{align}
    \Prob(\Etwo^c)  \ \leq \  \Prob \left(\lambda_{\max} \left ( \frac{1}{T}M^\star_T \right) > \frac{\laml}{2} \right)
\end{align}

\noindent Now, we apply the following lemma (\citet{tropp2011freedman}):
\begin{theorem}[Matrix Freedman]
\label{thm:matrix-freedman}
Consider a matrix martingale $\{Y_k : k = 0,1,2,\dots\}$ whose values are
self-adjoint matrices of dimension $d$, and let
$\{X_k : k = 1,2,3,\dots\}$ be the difference sequence, where
\begin{align*}
X_k &= Y_k - Y_{k-1}.
\end{align*}

\noindent Assume that the difference sequence is uniformly bounded in the sense that
\begin{align*}
\lambda_{\max}(X_k) &\le R
\quad \text{almost surely, for all } k \ge 1.
\end{align*}

\noindent Define the predictable quadratic variation process
\begin{align*}
W_k &:= \sum_{j=1}^k \mathbb{E}_{j-1}\!\left[X_j^2\right],
\quad \text{for } k = 1,2,3,\dots
\end{align*}

\noindent Then, for all $t \ge 0$ and $\sigma^2 > 0$,
\begin{align*}
\mathbb{P}\!\left(
\exists\, k \ge 0 :
\lambda_{\max}(Y_k) \ge t
\ \text{and}\ 
\|W_k\|_{op} \le \sigma^2
\right)
&\le
d \cdot \exp\!\left(
-\frac{t^2/2}{\sigma^2 + Rt/3}
\right).
\end{align*}
\end{theorem}

\noindent Recall that $M_T = \sum^T_{t=1}D_t$, where $D_t = z_t z_t^\top - \Exs[z_t z_t^\top \mid \mathcal{F}_{t -1}]$. Now, for any vector $v$ with $\|v\|_2 \leq 1$, $\|vv^\top\|_{op} \leq 1$ (as it is a rank $1$ matrix and has only one positive eigenvalue equaling $\|v \|_2$). Hence, Assumption~\eqref{assn:context} ensures  
\begin{align*}
    \|D_t\|_{op} \leq \|z_t\|_2 +  \| \Exs [z_t z_t^\top \mid \mathcal{F}_{t - 1}]\|_{op} \leq 1 + \Exs \|z_t\|_2 \leq 2 
\end{align*}
where the second inequality above follows via Jensen's inequality. Additionally, 
\begin{align*}
 \| W_T\|_{op} \leq \sum_{t = 1}^T \Exs\|D_t^2 \mid \mathcal{F}_{t-1}\|_{op} \leq 
 \sum_{t = 1}^T \Exs[\|D_t^2 \|_{op}\mid \mathcal{F}_{t-1}] \leq 4T
\end{align*}

\noindent Therefore, by applying Lemma~\ref{thm:matrix-freedman}  with $X_t = D_t$ we have
\noindent Therefore, we have proved that
\begin{align*}
     \Prob(\Etwo^c) \leq  d \cdot \expower\left\{ - \frac{(\laml)^2}{32 +  8\laml/3} \cdot  T\right\}
\end{align*}

\noindent Hence, our proof is complete.

\end{document}